\documentclass[a4paper]{article}

\usepackage{amssymb}
\usepackage{amsmath,amsthm}
\usepackage{enumerate}
\usepackage{xcolor}
\usepackage{soul}
\usepackage{tikz}
\usepackage{url}
\usepackage[subnum]{cases}
\usepackage{bm}
\usepackage[numbers,sort&compress]{natbib}
\usepackage[titlenumbered,ruled,lined,linesnumbered,algo2e]{algorithm2e}
\usepackage{subfigure}
\usepackage{multirow}
\usepackage{hyperref}
\newcommand{\nbb}{\mathbb{N}}
\newcommand{\hcal}{\mathcal{H}}
\newcommand{\diag}{\text{diag}}

\newcommand{\fcal}{\mathcal{F}}
\newcommand{\xcal}{\mathcal{X}}
\newcommand{\zcal}{\mathcal{Z}}
\newcommand{\ycal}{\mathcal{Y}}

\newcommand{\ebb}{\mathbb{E}}

\newcommand{\rbb}{\mathbb{R}}

\newcommand{\ncal}{\mathcal{N}}
\newcommand{\bv}{\mathbf{v}}
\newcommand{\bw}{\mathbf{w}}
\newcommand{\bt}{\mathbf{t}}
\newcommand{\bx}{\mathbf{x}}
\newcommand{\bu}{\mathbf{u}}
\newcommand{\bz}{\mathbf{z}}
\newcommand{\be}{\mathbf{e}}
\newcommand{\tr}{\mathrm{tr}}
\newcommand{\fat}{\mathrm{fat}}
\newcommand{\bp}{\tilde{\bm{\phi}}}

\newcommand{\inn}[1]{\langle#1\rangle}
\newtheorem{theorem}{Theorem}
\newtheorem{lemma}[theorem]{Lemma}
\newtheorem{proposition}[theorem]{Proposition}
\newtheorem{corollary}[theorem]{Corollary}
\theoremstyle{definition}
\newtheorem{definition}{Definition}

\newtheorem{example}{Example}
\theoremstyle{definition}

\newtheorem{remark}{Remark}

\title{Data-dependent Generalization Bounds for Multi-class Classification}
\author{Yunwen Lei\thanks{Y. Lei was with Department of Mathematics, City University of Hong Kong, Kowloon, Hong Kong, China. He is now with Department of Computer Science and Engineering, Southern University of Science and Technology, Shenzhen, China (e-mail: leiyw@sustc.edu.cn).},\;
        \"Ur\"un Dogan\thanks{U. Dogan is with Microsoft Research, Cambridge CB1 2FB, UK (e-mail: udogan@microsoft.com).},\; 
        Ding-Xuan Zhou\thanks{D.-X. Zhou is with Department of Mathematics, City University of Hong Kong, Kowloon, Hong Kong, China (e-mail: mazhou@cityu.edu.hk).}\;   
        and~Marius~Kloft\thanks{M. Kloft is with Department of Computer Science,  University of Kaiserslautern, Kaiserslautern, Germany (e-mail: kloft@cs.uni-kl.de).}}
\date{}
\begin{document}
\maketitle
\begin{abstract}

In this paper, we study \emph{data-dependent} generalization error bounds exhibiting a mild dependency on the number of classes,
making them suitable for multi-class learning with a large number of label classes.  The bounds generally hold for empirical
multi-class risk minimization algorithms using an arbitrary norm as regularizer.
 Key to our analysis are new structural results for multi-class Gaussian complexities and empirical $\ell_\infty$-norm
covering numbers, which exploit  the Lipschitz continuity of the loss function with respect to the $\ell_2$-
and $\ell_\infty$-norm, respectively. We establish data-dependent error bounds
in terms of complexities of a linear function class defined on a finite set induced by training examples, for which we show tight
lower and upper bounds.  We apply the  results to several prominent
multi-class learning machines, exhibiting a tighter dependency on the number of classes than the state of the art.
 For instance, for the multi-class SVM by Crammer and Singer (2002), we obtain a data-dependent bound with a logarithmic dependency
which significantly improves the previous square-root dependency.
Experimental results are reported to verify the effectiveness of our theoretical findings.

\medskip
\textbf{Keywords}: Multi-class classification, Generalization error bounds, Covering numbers, Rademacher complexities, Gaussian complexities.
\end{abstract}

\section{Introduction}

Multi-class learning is a classical problem in machine learning \citep{vapnik1998statistical}.
The outputs here stem from a finite set of categories (\emph{classes}),
and the aim is to classify each input into one of several possible target classes \citep{har2002constraint,hsu2002comparison,dogan2016unified}.
Classical applications of multi-class classification include handwritten optical character recognition,
where the system learns to automatically interpret handwritten characters \citep{kato1999handwritten},
part-of-speech tagging, where each word in a text is annotated with part-of-speech tags \citep{voutilainen2003part},
and image categorization, where predefined categories are associated with digital images \citep{deng2009imagenet,binder2012taxonomies},
to name only a few.

Providing a theoretical framework of multi-class learning algorithms is a fundamental task in statistical learning theory \citep{vapnik1998statistical}.
Statistical learning theory aims to ensure formal guarantees safeguarding the performance of learning algorithms,
often in the form of generalization error bounds \citep{mohri2012foundations}. Such bounds may lead to improved understanding of commonly used empirical practices and
spur the development of novel learning algorithms (``Nothing is more practical than a good theory'' \cite{vapnik1998statistical}).

Classic generalization bounds for multi-class learning scale rather unfavorably (like quadratic, linear, or square root at best) in the number of classes
\citep{koltchinskii2002empirical,guermeur2002combining,mohri2012foundations}.
This may be because the standard theory has been constructed without the need of having a large number of label classes in mind
as many classic multi-class learning problems consist only of a small number of classes, indeed.
For instance, the historically first multi-class dataset---\texttt{Iris}---\citep{fisher1936use}---contains solely three classes,
the MNIST dataset \citep{lecun1998mnist} consists of 10 classes,
and most of the datasets in the popular UCI corpus \citep{asuncion2007uci} contain up to several dozen classes.

However, with the advent of the big data era, multi-class learning problems---such as text or image classification \citep{deng2009imagenet,partalas2015lshtc}---can involve tens or hundreds of thousands of classes. 
Recently, there is a subarea of machine learning studying classification problems involving an extremely large number of classes
(such as the ones mentioned above) called \emph{eXtreme Classification} (XC)
\citep{extremeWS2013}.
Several algorithms have recently been proposed to speed up the training or improve the prediction accuracy
in classification problems with many classes \citep{varadarajan2015efficient,bhatia2015sparse,bengio2010label,partalas2015lshtc,beygelzimer2009conditional,sedhai2015hspam14,jain2016extreme,babbar2016learning,prabhu2014fastxml,alber2017distributed,babbar2016tersesvm}.

However, there is still a discrepancy between \emph{algorithms} and \emph{theory} in classification with many classes,
as standard statistical learning theory is void in the large number of classes scenario \citep{extremeWS2015}.
With the present paper we want to contribute toward a \emph{better theoretical understanding} of multi-class classification with many classes.
This theoretical understanding can provide theoretical grounds to the commonly used empirical practices in classification
with many classes and lead to insights that may be used to guide the design of new learning algorithms.

Note that the present paper focuses on \emph{multi-class} learning.
Recently, there has been a growing interest in \emph{multi-label} learning.
The difference in the two scenarios is that each instance is associated with exactly one label class (in the multi-class case) or multiple classes (in the multi-label case), respectively.
While the present analysis is tailored to the multi-class learning scenario, it
may serve as a starting point for subsequent analysis of the multi-label learning scenario.

\subsection{Contributions in a Nutshell\label{sec:nutshell}}

We build the present journal article upon our previous conference paper published at NIPS 2015 \citep{lei2015multi},
where we propose a multi-class support vector machine (MC-SVM) using block $\ell_{2,p}$-norm regularization,
for which we prove data-dependent generalization bounds based on Gaussian complexities (GCs).

While the previous analysis employed the margin-based loss, in the present article, we generalize the  GC-based data-dependent analysis to general
loss functions that are Lipschitz continuous with respect to (w.r.t.) a variant of the $\ell_2$-norm .
Furthermore, we develop a new approach to derive data-dependent bounds
based on empirical covering numbers (CNs) to capture the Lipschitz continuity of loss functions w.r.t.
the $\ell_\infty$-norm with a moderate Lipschitz constant, which is \emph{not} studied in the conference version
of this article. For both two approaches, our data-dependent error bounds can be stated in terms
of complexities of a linear function class defined only on a finite set induced by training examples,
for which we give lower and upper bounds matching up to a constant factor in some cases.
We present examples to show that each of these two approaches has its advantages and may
outperform the other by inducing tighter error bounds for some specific MC-SVMs.

As applications of our theory, we show error bounds for several prominent multi-class learning algorithms: multinomial logistic regression \citep{bishop2006pattern}, top-$k$ MC-SVM \citep{lapin2015top}, $\ell_p$-norm MC-SVMs \citep{lei2015multi},
and several classic MC-SVMs \citep{crammer2002algorithmic,weston1998multi,lee2004multicategory}.
For all of these methods, we show error bounds with an improved dependency on the number of classes over the state of the art.
For instance, the best known bounds for multinomial logistic regression and the MC-SVM by \citet{crammer2002algorithmic}
scale square root in the number of classes. We improve this dependency to be \emph{logarithmic}.
This gives strong theoretical grounds for using these methods in classification with many classes.

We develop a novel algorithm to train $\ell_p$-norm MC-SVMs \citep{lei2015multi} and report experimental results
to verify our theoretical findings and their applicability to model selection.

\section{Related Work and Contributions}\label{sec:comparison}

In this section, we discuss related work and outline the main contributions of this paper.

\subsection{Related Work}

In this subsection, we recapitulate the state of the art in multi-class learning theory.

\subsubsection{Related Work on Data-dependent Bounds}\label{sec:related_data}

Existing error bounds for multi-class learning can be classified into two groups:  \emph{data-dependent} and \emph{data-independent} error bounds.
Both types of bounds are often based on the assumption that the data is realized from independent and identically distributed random variables.
However, this assumption can be relaxed to weakly dependent time series, for which \citet{mohri2009rademacher} and \citet{steinwart2009learning}
show data-dependent and -independent generalization bounds, respectively.

\emph{Data-dependent} generalization error bounds refer to bounds that can be evaluated on training samples
and thus can capture properties of the distribution that has generated the data \citep{mohri2012foundations}.
Often these bounds built on the empirical Rademacher complexity (RC) \citep{koltchinskii2001rademacher,bartlett2002rademacher,mendelson2002rademacher},
which can be used in model selection 
and for the construction of new learning algorithms \citep{cortes2013learning}.

The investigation of data-dependent error bounds for multi-class learning is initiated, to our best knowledge, by \citet{koltchinskii2002empirical},
who give the following structural result on RCs:
given a set $H=\{h=(h_1,\ldots,h_c)\}$ of vector-valued functions and training examples $\bx_1,\ldots,\bx_n$, it holds
\begin{equation}\label{Rademacher-maximum-lem-8-1}
  \ebb_{\bm{\epsilon}}\sup_{h\in H}\sum_{i=1}^{n}\epsilon_i\max\big\{h_1(\bx_i),\ldots,h_c(\bx_i)\big\}\leq \sum_{j=1}^{c}\ebb_{\bm{\epsilon}}\sup_{h\in H}\sum_{i=1}^{n}\epsilon_ih_j(\bx_i).
\end{equation}
Here, $\epsilon_1,\ldots,\epsilon_n$ denote independent Rademacher variables (i.e., taking values $+1$ or $-1$, with equal probability),
and $\ebb_{\bm{\epsilon}}$  denotes the conditional expectation operator removing the randomness coming from the variables $\epsilon_1,\ldots,\epsilon_n$.

In much subsequent theoretical work on multi-class learning, the above result is used as a starting point,
by which the maximum operator involved in multi-class hypothesis classes (Eq. \ref{Rademacher-maximum-lem-8-1}, left-hand side) can be removed \citep{crammer2002algorithmic,mohri2012foundations}.
Applying the result leads to a simple sum of $c$ many RCs (Eq. \eqref{Rademacher-maximum-lem-8-1}, right-hand side),
each of which can be bounded using standard theory \citep{bartlett2002rademacher}.
This way, \citet{koltchinskii2002empirical}, \citet{cortes2013multi}, and \citet{mohri2012foundations}
derive multi-class generalization error bounds that exhibit a quadratic dependency on the number of classes,
which \citet{kuznetsov2014multi} improve to a linear one.

\begin{figure}[!h]
\centering
\begin{tikzpicture}
\draw[color=blue!70, fill=gray!30, very thick] (1,1) -- (-1,1) -- (-1,-1)-- (1,-1) -- (1,1);
\draw[color=red!70, fill=white!15, very thick] (0,0) circle(1);
\draw[->,thick] (-1.5,0)--(1.5,0) node[right]{$w_1$};
\draw[->,thick] (0,-1.5)--(0,1.55) node[above]{$w_2$};
\node [below right] at (1,0) {1};
\node [above right] at (0,1) {1};
\node [below right] at (0,-1) {-1};
\node [below left] at (-1,0) {-1};
\draw[ultra  thick,color=black] (1,0) -- (1,0.1);
\draw[ultra  thick,color=black] (-1,0) -- (-1,0.1);
\draw[ultra  thick,color=black] (0,1) -- (0.1,1);
\draw[ultra  thick,color=black] (0,-1) -- (0.1,-1);
\end{tikzpicture}
\caption{Illustration why Eq. \eqref{Rademacher-maximum-lem-8-1} is loose.
Consider a $1$-dimensional binary classification problem with  hypothesis class $H$ consisting of functions
mapping $x\in\mathbb R$ to $\max(h_1(x),h_2(x))$, where $h_j(x)=w_jx$ for $j=1,2$.
Assume the class is regularized through the constraint $\|(w_1,w_2)\|_2\leq1$,
so the left-hand side of the inequality \eqref{Rademacher-maximum-lem-8-1}
involves a supremum over the $\ell_2$-norm constraint $\|(w_1,w_2)\|_2\leq1$.
In contrast, the right-hand side of \eqref{Rademacher-maximum-lem-8-1} has individual suprema for $w_1$ and $w_2$ (no coupling anymore),
resulting in a supremum over the $\ell_\infty$-norm constraint $\|(w_1,w_2)\|_\infty\leq1$.
Thus applying Eq. \eqref{Rademacher-maximum-lem-8-1} enlarges the size of constraint set by the area that is shaded in the figure,
which grows as $O(\sqrt{c})$.
In the present paper, we show a proof technique to elevate this problem,
thus resulting in an improved bound (tighter by a factor of $\sqrt{c}$).
\label{fig:coupling}}
\end{figure}
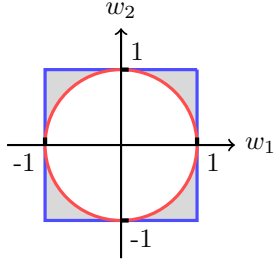

However, the reduction \eqref{Rademacher-maximum-lem-8-1} comes at the expense of at least a linear dependency on the number of classes $c$,
coming from the sum in Eq. \eqref{Rademacher-maximum-lem-8-1} (right-hand side), which consists of $c$ many terms.
We show in this paper that this linear dependency can oftentimes be suboptimal because
\eqref{Rademacher-maximum-lem-8-1} does not take into account the coupling among the classes. 
To understand why, it is illustrative to consider the example of MC-SVM by \citet{crammer2002algorithmic},
which uses an $\ell_2$-norm constraint
\begin{equation}\label{eq:contr1}
 \big\|\big(h_1,\ldots,h_c\big)\big\|_2\leq\Lambda
\end{equation}
to couple the components $h_1,\ldots,h_c$.
The problem with Eq. \eqref{Rademacher-maximum-lem-8-1} is that it decouples the components, resulting in the constraint $\big\|\big(h_1,\ldots,h_c\big)\big\|_\infty\leq\Lambda$,
which---as illustrated in Figure~\ref{fig:coupling}---is a poor approximation of \eqref{eq:contr1}.

In our previous work \citep{lei2015multi}, we give a structural result addressing this shortcoming
and tightly preserving the constraint defining the hypothesis class.
Our result is based on the so-called GC \citep{bartlett2002rademacher},
a notion similar but yet different to the RC.
The difference in the two notions is that RC and GC are the supremum of a Rademacher and Gaussian process, respectively.

The core idea of our analysis is that we exploit a comparison inequality for the suprema of Gaussian processes known as \emph{Slepian's Lemma} \citep{slepian1962one},
by which we can remove, from the GC, the maximum operator that occurs in the definition of the hypothesis class, thus preserving the
above mentioned coupling---we call the supremum of the resulting Gaussian process the \emph{multi-class Gaussian complexity}.

Using our structural result, we obtain in \cite{lei2015multi} a data-dependent error bound for \cite{crammer2002algorithmic} that
exhibits---for the first time---a sublinear (square root) dependency on the number of classes.
When using a block $\ell_{2,p}$-norm constraint (with $p$ close to $1$), rather than an $\ell_2$-norm one, one can reduce this dependency to be \emph{logarithmic}, making the analysis appealing for classification with many classes. 

We note that, addressing the same need, the following structural result \citep{maurer2016vector,cortes2016structured} appear since the publication of our previous work \citep{lei2015multi}:
\begin{equation}\label{structural-rademacher}
  \ebb_{\bm{\epsilon}}\sup_{h\in H}\sum_{i=1}^{n}\epsilon_if_i(h(\bx_i))\leq \sqrt{2}L\ebb_{\bm{\epsilon}}\sup_{h\in H}\sum_{i=1}^{n}\sum_{j=1}^{c}\epsilon_{ij}h_j(\bx_i),
\end{equation}
where $f_1,\ldots,f_n$ are $L$-Lipschitz continuous w.r.t. the $\ell_2$-norm.

For the MC-SVM by \citet{crammer2002algorithmic}, the above result leads to the same favorable square root dependency on the number of classes as our previous result in \cite{lei2015multi}.
We note, however, that the structural result \eqref{structural-rademacher} requires $f_i$ to be Lipschitz continuous w.r.t. the $\ell_2$-norm,
while some multi-class loss functions \citep{weston1998multi,jenssen2012scatter,lapin2015top} are Lipschitz continuous with a moderate Lipschitz constant,
when choosing a more appropriate norm.
In these cases, the analysis given in the present paper improves not only the classical results obtained through \eqref{Rademacher-maximum-lem-8-1},
but also the results obtained through \eqref{structural-rademacher}.

%
%
\subsubsection{Related Work on Data-independent Bounds}

\emph{Data-independent} generalization bounds refer to classical theoretical bounds that hold for
any sample, with a certain probability over the draw of the samples \citep{vapnik1998statistical,steinwart2008support}.
In their seminal contribution \emph{On the Uniform Convergence of Relative Frequencies of Events to Their Probabilities},
\citet{Vapnik71} show historically one of the first bounds of that type---introducing the notion of VC dimension.

Several authors consider data-independent bounds for  multi-class learning.
By controlling entropy numbers of linear operators with Maurey's theorem,
\citet{guermeur2002combining} derives generalization error bounds with a linear dependency on the number of classes.
This is improved to square-root by \citet{zhang2004statistical} using $\ell_\infty$-norm CNs without considering the correlation
among class-wise components.
\citet{pan2008parzen} consider a multi-class Parzen window classifier and derive an error bound with a quadratic dependency on the number of classes.
Several authors show data-independent generalization bounds based on combinatorial dimensions,
including  the graph dimension, the Natarajan dimension $d_{\text{nat}}$, and its scale-sensitive analog $d_{\text{nat},\gamma}$ for margin $\gamma$
\citep{guermeur2010sample,guermeur2007vc,daniely2015multiclass,daniely2012multiclass,natarajan1989learning}.

\citet{guermeur2010sample,guermeur2007vc} shows a generalization bound decaying as $O\big(\log c\sqrt{\frac{d_{\text{nat},\gamma}\log n}{n}}\big)$.
When using an $\ell_\infty$-norm regularizer $d_{\text{nat},\gamma}$ is bounded by $O(c^2\gamma^{-2})$,
and the generalization bound reduces to $O\big(\frac{c\log c}{\gamma}\sqrt{\frac{\log n}{n}}\big)$.
The author does not give a bound for an $\ell_2$-norm regularizer as this is more challenging to deal with,
due to the above mentioned coupling of the hypothesis components.

\citet{daniely2015multiclass} give a bound decaying as $O\big(\sqrt{\frac{d_{\text{nat}}(H)\log c}{n}}\big)$,
which transfers to $O\big(\sqrt{\frac{dc\log c}{n}}\big)$ for multi-class linear classifiers since the associated Natarajan dimension grows as $O(dc)$ \citep{daniely2012multiclass}.

\citet{guermeur2017} recently establish an $\ell_p$-norm Sauer-Shelah lemma for large-margin multi-class classifiers,
based on which error bounds with a square-root dependency on the number of classes are derived.
This setting comprises the MC-SVM by \citet{crammer2002algorithmic}.

What is common in all of the above mentioned data-independent bounds is their super logarithmic dependency (square root at best) on the number of classes.
As notable exception, \citet{kontorovich2014maximum} show a bound exhibiting a logarithmic dependency on the number of classes.
However, their bound holds only for the specific nearest-neighbor-based algorithm that they propose,
so their analysis does not cover the commonly used multi-class learning machines mentioned in the introduction (like multinomial logistic regression and classic MC-SVMs).
Furthermore, their bound is of the order $\min\Big\{O\big(\gamma^{-1}\big(\frac{\log c}{n}\big)^{\frac{1}{1+D}}\big),O\big(\gamma^{-\frac{D}{2}}\big(\frac{\log c}{n}\big)^{\frac{1}{2}}\big)\Big\}$
and thus has an \emph{exponential} dependence on the doubling dimension $D$ of the metric space where the learning takes place.
For instance, for linear learning methods with input space dimension $d$, the doubling dimension $D$ grows linearly in $d$,
so the bound in \citep{kontorovich2014maximum} grows exponentially in $d$.
For kernel-based learning using an infinite doubling dimension (e.g., Gaussian kernels) the bound is void.

\subsection{Contributions of this Paper\label{sec:contribution}}

This paper aims to contribute a solid theoretical foundation for learning with many class labels
by presenting data-dependent generalization error bounds with relaxed dependencies on the number of classes.
 We develop two approaches to establish data-dependent error bounds: one based on multi-class
GCs and one based on empirical $\ell_\infty$-norm CNs. We give specific examples to show that each
of these two approaches has its advantages and may yield error bounds tighter than the other.
We also develop novel algorithms to train $\ell_p$-norm MC-SVMs~\citep{lei2015multi} and report experimental results.
Below we summarize the main results of this paper.

\subsubsection{Tighter Generalization Bounds by Gaussian Complexities}
As an extension of our NIPS 2015 conference paper, our GC-based analysis depends on a novel structural
result on GCs (Lemma \ref{lem:GP-structural-lipschitz} below) that is able to preserve the correlation among class-wise components.
Similar to \citet{maurer2016vector} and \citet{cortes2016structured}, our structural result applies to function classes induced
by operators satisfying a Lipschitz continuity. However, here we measure the Lipschitz continuity with respect to a specially
crafted variant of the $\ell_2$-norm involving a Lipschitz constant pair $(L_1,L_2)$ (cf. Definition \ref{def:lipschitz-variant-2-norm} below),
motivated by the observation that some multi-class loss functions satisfy this Lipschitz continuity with a relatively small $L_1$
in a dominant term and a relatively large $L_2$ in a non-dominant term.
This allows us to improve the error bounds based on the structural result \eqref{structural-rademacher}
for MC-SVMs with a relatively large $L_2$.

Based on this new structural result, we show an error bound for multi-class empirical
risk minimization algorithms using an arbitrary norm as regularizer.
As instantiations of our general bound, we compute specific bounds for $\ell_{2,p}$-norm and Schatten $p$-norm regularizers.
We apply this general GC-based bound to some popular MC-SVMs \citep{weston1998multi,crammer2002algorithmic,lee2004multicategory,bishop2006pattern,jenssen2012scatter}.


Our GC-based analysis yields the first error bound for top-$k$ MC-SVM \citep{lapin2015top}
as a decreasing function in $k$. When setting $k$ proportional to $c$, the bound does not depend at all on the number of classes.
In contrast, error bounds based on the structural result \eqref{structural-rademacher} fail to shed insights on the influence of $k$
on the generalization performance because the involved Lipschitz constant is dominated by a constant.
For the MC-SVM by \citet{weston1998multi},
our analysis yields a bound exhibiting a linear dependency on the number of classes,
while the dependency is $O(c^{3\over2})$ for the error bound based on the structural result \eqref{structural-rademacher}.
For the MC-SVM by \citet{jenssen2012scatter}, our analysis yields a bound with no dependencies on $c$, while the error bound
based on the structural result \eqref{structural-rademacher} enjoys a square root dependency.
This shows the effectiveness of our new structural result in capturing the Lipschitz continuity w.r.t. a variant of the $\ell_2$-norm.

\subsubsection{Tighter Generalization Bounds by Covering Numbers}

While the GC-based analysis uses the Lipschitz continuity measured by the $\ell_2$-norm or a variant thereof,
some multi-class loss functions are Lipschitz continuous w.r.t. the $\ell_\infty$-norm with a moderate Lipschitz constant.
To apply the GC-based error bounds, we need to transform this $\ell_\infty$-norm Lipschitz continuity to the $\ell_2$-norm
Lipschitz continuity, at the cost of a  multiplicative factor of $\sqrt{c}$.
Motivated by this observation, we present another data-dependent analysis based on empirical $\ell_\infty$-norm CNs to fully exploit the Lipschitz
continuity measured by the $\ell_\infty$-norm. We show that this leads to bounds that for some MC-SVMs exhibit a milder dependency on the number of classes.

The core idea is to introduce a linear and scalar-valued function class induced by training examples to extract all components
of hypothesis functions on training examples, which allows us to relate the empirical $\ell_\infty$-norm CNs of loss function classes to
that of this linear function class. Our main result is a data-dependent error bound for general MC-SVMs expressed in terms of the \emph{worst-case}
RC of a linear function class, for which we establish lower and upper bounds matching up to a constant factor.
The analysis in this direction is unrelated to the conference version \citep{lei2015multi} and provides an alternative to GC-based arguments.


As direct applications, we derive other data-dependent generalization error bounds
scaling sublinearly for $\ell_p$-norm MC-SVMs and Schatten-$p$ norm MC-SVMs,
and \emph{logarithmically} for top-$k$ MC-SVM \citep{lapin2015top}, trace-norm regularized MC-SVM \citep{amit2007uncovering},
multinomial logistic regression~\citep{bishop2006pattern} and the MC-SVM by \citet{crammer2002algorithmic}.
Note that the previously best results for the MC-SVM in \cite{crammer2002algorithmic} and  multinomial logistic regression
scale only square root in the number of classes \citep{zhang2004statistical}.

\subsubsection{Novel Algorithms with Empirical Verifications}
We propose a novel algorithm to train $\ell_p$-norm MC-SVMs~\citep{lei2015multi} using the Frank-Wolfe algorithm \citep{frank1956algorithm},
for which we show that the involved linear optimization problem has a closed-form solution,
making the implementation of the Frank-Wolfe algorithm very simple and efficient.
This avoids the introduction of class weights used in our previous optimization algorithm \citep{lei2015multi}, which moreover only
applies to the case $1\leq p\leq2$.
In empirical comparisons, we show on benchmark data that  $\ell_p$-norm MC-SVM can  outperform  $\ell_2$-norm
MC-SVM.
We also perform experiments to show that our error bounds well capture the effect by the number of classes and the parameter
$p$. Furthermore, our error bounds suggest a structural risk able to guide the
selection of model parameters.

\section{Main Results}

\subsection{Problem Setting}

In multi-class classification with $c\label{pg:class-size}$ classes, we are given training examples $S=\{\bz_i=(\bx_i,y_i)\}_{i=1}^n\subset\zcal:=\xcal\times\ycal\label{pg:input-output-space}$,
where $\xcal\subset\rbb^d$ is the input space and $\ycal=\{1,\ldots,c\}$ the output space.
We assume that $\bz_1,\ldots,\bz_n$ are independently drawn from a probability measure $P$ defined on $\zcal$.

Our aim is to learn, from a hypothesis space $H$, a hypothesis $h=(h_1,\ldots,h_c):\xcal\mapsto\rbb^c$
used for prediction via the rule $\bx\to\arg\max_{y\in\ycal}h_y(\bx)$.
We consider prediction functions of the form $h^{\bw}_j(\bx)=\inn{\bw_j,\phi(\bx)}$,
where $\phi\label{pg:feature}$ is a feature map associated to a Mercer kernel $K$ defined over $\xcal\times\xcal$,
$\bw_j$ belongs to the reproducing kernel Hilbert space $H_K~\label{pg:rkhs}$ induced from $K\label{pg:kernel}$
with the inner product $\inn{\cdot,\cdot}$ satisfying $K(\bx,\tilde{\bx})=\langle\bx,\tilde{\bx}\rangle$.

We consider hypothesis spaces of the form
\begin{equation}\label{hypothesis-space}
  H_{\tau}=\Big\{h^{\bw}= \big(\langle \bw_1,\phi(\bx)\rangle,\ldots,\langle \bw_c,\phi(\bx)\rangle\big):\bw=(\bw_1,\ldots,\bw_c)\in H_K^c,\tau(\bw)\leq\Lambda\Big\},
\end{equation}
where $\tau$ is a functional defined on $H_K^c:=\underbrace{H_K\times\cdots\times H_K}_{c \rm ~ times}\label{pg:Cartesian-product}$ and $\Lambda>0$. Here we omit the dependency on $\Lambda$ for brevity.

We consider a general problem setting with $\Psi_y(h_1(\bx),\ldots,h_c(\bx))\label{pg:Psi-y}$
used to measure the prediction quality of the model $h$ at $(\bx,y)$ \citep{zhang2004statistical,tewari2007consistency},
where $\Psi_y:\rbb^c\to\rbb_+$ is a real-valued function taking a $c$-component vector as its argument.
This general loss function $\Psi_y$ is widely used in many existing MC-SVMs,
including the models by \citet{crammer2002algorithmic}, \citet{weston1998multi}, \citet{lee2004multicategory}, \citet{zhang2004statistical}, and \citet{lapin2015top}.

\begin{table}
  \centering
  \caption{Notations used in this paper and page number where it first occurs.\label{tab:notation}}
  \begin{tabular}{|c|c|c|}
    \hline
    notation & meaning & page \\ \hline \hline
    $\xcal,\ycal$ & the input space and output space, respectively & \pageref{pg:input-output-space} \\  \hline
    $S$ & the set of training examples $\{\bz_i=(\bx_i,y_i)\}\in\xcal\times\ycal$ & \pageref{pg:input-output-space} \\ \hline
    $c$ & number of classes & \pageref{pg:class-size} \\ \hline
    $K$ & Mercer kernel & \pageref{pg:kernel} \\ \hline
    $\phi$ & feature map associated to a kernel $K$& \pageref{pg:feature} \\ \hline
    $H_K$ & reproducing kernel Hilbert space induced by a Mercer kernel $K$& \pageref{pg:rkhs} \\ \hline
    $H_K^c$ & $c$-fold Cartesian product of the reproducing kernel Hilbert space $H_K$ & \pageref{pg:Cartesian-product}  \\\hline
    $\bw$ & $(\bw_1,\ldots,\bw_c)\in H_K^c$ & \pageref{hypothesis-space}  \\\hline
    $h^{\bw}$ & prediction function $(\langle \bw_1,\phi(\bx)\rangle,\ldots,\langle \bw_c,\phi(\bx)\rangle)$ & \pageref{hypothesis-space}\\ \hline
    $H_\tau$ & hypothesis space for MC-SVM constrained by a regularizer $\tau$ & \pageref{hypothesis-space} \\ \hline
    $\Psi_y$ & multi-class loss function for class label $y$&  \pageref{pg:Psi-y} \\ \hline
    $\|\cdot\|_p$ & $\ell_p$-norm defined on $\rbb^c$ & \pageref{pg:p-norm} \\ \hline
    $\|\cdot\|_{2,p}$ & $\ell_{2,p}$ norm defined on $H_K^c$  &  \pageref{2-p-norm}\\ \hline
    $\inn{\bw,\bv}$ & inner product on $H_K^c$ as $ \sum_{j=1}^{c}\inn{\bw_j,\bv_j}$ & \pageref{pg:inner-product}\\ \hline
    $\|\cdot\|_*$ & dual norm of $\|\cdot\|$ & \pageref{pg:dual-norm} \\ \hline
    $\nbb_n$ & the set $\{1,\ldots,n\}$ &\pageref{nbb-n} \\ \hline
    $p^*$ & dual exponent of $p$ satisfying $1/p+1/p^*=1$ & \pageref{pg:dual-exponent} \\ \hline
    $\ebb_{\bu}$ & the expectation w.r.t. random $\bu$ & \pageref{pg:ebb-u}\\ \hline
    $B_\Psi$ & the constant $\sup_{(\bx,y)\in\zcal,h\in H_\tau}\Psi_y(h(\bx))$ & \pageref{pg:B}  \\\hline
    $\hat{B}_\Psi$ & the constant $n^{-\frac{1}{2}}\sup_{h\in H_\tau}\big\|\big(\Psi_{y_i}(h(\bx_i))\big)_{i=1}^n\big\|_2$ & \pageref{pg:B}  \\\hline
    $\hat{B}$ & the constant $\max_{i\in\nbb_n}\|\phi(\bx_i)\|_2\sup_{\bw:\tau(\bw)\leq\Lambda}\|\bw\|_{2,\infty}$ & \pageref{pg:B} \\ \hline
    $A_\tau$ & the term defined in \eqref{pg:A-tau} & \pageref{pg:A-tau} \\ \hline
    $I_y$ & indices of examples with class label $y$ & \pageref{pg:I-y} \\ \hline
    $\|\cdot\|_{S_p}$ & Schatten-$p$ norm of a matrix & \pageref{pg:schatten} \\ \hline
    $\mathfrak{R}_S(H)$ & empirical Rademacher complexity of $H$ w.r.t. sample $S$ & \pageref{def:rademacher} \\ \hline
    $\mathfrak{G}_S(H)$ & empirical Gaussian complexity of $H$ w.r.t. sample $S$ & \pageref{def:rademacher} \\ \hline
    $\frak{R}_n(H)$ & worst-case Rademacher complexity of $H$ w.r.t. $n$ examples & \pageref{def:rademacher} \\ \hline
    $\widetilde{H}_{\tau}$ & class of scalar-valued linear functions defined on $H_K^c$ & \pageref{H-tilde-tau} \\ \hline
    $\widetilde{S}$ & an enlarged set of cardinality $nc$ defined in \eqref{tilde-S} & \pageref{tilde-S}\\ \hline
    $\widetilde{S}'$ & a set of cardinality $n$ defined in \eqref{tilde-S-prime} & \pageref{tilde-S-prime}\\ \hline
    $F_{\tau,\Lambda}$ & loss function class for MC-SVM & \pageref{loss-function-class} \\  \hline
    $\rho_h(\bx,y)$ & margin of $h$ at $(\bx,y)$ & \pageref{margin} \\ \hline
    $\ncal_\infty(\epsilon,F,S)$ & empirical covering number of $F$ w.r.t. sample $S$ & \pageref{def:covering-number} \\ \hline
    $\fat_\epsilon(F)$ & fat-shattering dimension of $F$ & \pageref{def:shattering} \\ \hline
    \hline
  \end{tabular}
\end{table}

\subsection{Notations}

We collect some notations used throughout this paper (see also Table~\ref{tab:notation}).
We say that a function $f:\rbb^c\to\rbb$ is $L$-Lipschitz continuous w.r.t. a norm $\|\cdot\|$ in $\rbb^c$ if
$$
  |f(\bt)-f(\bt')|\leq L\|(t_1-t_1',\ldots,t_c-t_c')\|,\quad\forall \bt,\bt'\in\rbb^c.
$$
The $\ell_p$-norm of a vector $\bt=(t_1,\ldots,t_c)$ is defined as $\|\bt\|_p=\big[\sum_{j=1}^{c}|t_j|^p\big]^{1\over p}\label{pg:p-norm}$.
For any $\bv=(\bv_1,\ldots,\bv_c)\in H_K^c$ and $p\geq1$, we define the structure norm
$\|\bv\|_{2,p}=\big[\sum_{j=1}^{c}\|\bv_j\|_2^p\big]^{\frac{1}{p}}\label{2-p-norm}$.
Here, for brevity, we denote by $\|\bv_j\|_2$ the norm of $\bv_j$ in $H_K$.
For any $\bw=(\bw_1,\ldots,\bw_c),\bv=(\bv_1,\ldots,\bv_c)\in H_K^c$, we denote $\inn{\bw,\bv}=\sum_{j=1}^{c}\inn{\bw_j,\bv_j}\label{pg:inner-product}$.
For any $n\in\nbb$, we introduce the notation $\nbb_n:=\{1,\ldots,n\}\label{nbb-n}$.
For any $p\geq1$, we denote by $p^*$ the dual exponent of $p$ satisfying $1/p+1/p^*=1\label{pg:dual-exponent}$.
For any norm $\|\cdot\|$ we use $\|\cdot\|_*\label{pg:dual-norm}$ to mean its dual norm. 
Furthermore, we define $B_\Psi=\sup\limits_{(\bx,y)\in\zcal}\sup\limits_{h^{\bw}\in H_\tau}\Psi_y(h^{\bw}(\bx))$,
$\hat{B}_\Psi=n^{-\frac{1}{2}}\sup\limits_{h^{\bw}\in H_\tau}\big\|\big(\Psi_{y_i}(h^{\bw}(\bx_i))\big)_{i=1}^n\big\|_2$,
and $\hat{B}=\max\limits_{i\in\nbb_n}\|\phi(\bx_i)\|_2\sup\limits_{\bw:\tau(\bw)\leq\Lambda}\|\bw\|_{2,\infty}\label{pg:B}$.
For brevity, for any functional $\tau$ over $H_K^c$, we introduce the following notation to write our bounds compactly
\begin{equation}\label{pg:A-tau}
  A_{\tau}:=\sup_{h^{\bw}\in H_{\tau}}\Big[\ebb_{\bx,y}\Psi_y(h^{\bw}(\bx))-\frac{1}{n}\sum_{i=1}^n\Psi_{y_i}(h^{\bw}(\bx_i))\Big]-3B_\Psi\Big[\frac{\log\frac{2}{\delta}}{2n}\Big]^{\frac{1}{2}},
\end{equation}
where we omit the dependency on $n$ and loss function for brevity.
Note that, for any random $\bu$, the notation $\ebb_{\bu}\label{pg:ebb-u}$ denotes the expectation w.r.t. $\bu$.
For any $y\in\ycal$, we use $I_y=\{i\in\nbb_n:y_i=y\}\label{pg:I-y}$ to mean the indices of examples with label $y$.

If $\phi$ is the identity map, then the hypothesis $h^{\bw}$ can be compactly represented by a matrix $W=(\bw_1,\ldots,\bw_c)\in\rbb^{d\times c}$.
For any $p\geq1$, the Schatten-$p$ norm of a matrix $W\in\mathbb R^{d\times c}$ is defined as the $\ell_p$-norm of the
vector of singular values $\sigma(W):=(\sigma_1(W),\ldots,\sigma_{\min\{c,d\}}(W))^\top$ (singular values assumed to be sorted in a non-increasing order), i.e.,
$\|W\|_{S_p}:=\|\sigma(W)\|_p\label{pg:schatten}$.

\subsection{Data-dependent Bounds by Gaussian Complexities\label{sec:data-dependent-bound}}

We first present data-dependent analysis based on the established methodology of RCs and GCs~\citep{bartlett2002rademacher}.
\begin{definition}[Empirical Rademacher and Gaussian complexities]\label{def:rademacher}
  Let $H$ be a class of real-valued functions defined over a space $\tilde{\zcal}$ and $S'=\{\tilde{\bz}_i\}_{i=1}^n\in\tilde{\zcal}^n$.
  The \emph{empirical Rademacher and Gaussian complexities} of $H$ with respect to $S'$ are respectively defined as
  $$
    \mathfrak{R}_{S'}(H)=\ebb_{\bm{\epsilon}}\big[\sup_{h\in H}\frac{1}{n}\sum_{i=1}^n\epsilon_ih(\tilde{\bz}_i)\big],\quad
    \mathfrak{G}_{S'}(H)=\ebb_{\bm{g}}\big[\sup_{h\in H}\frac{1}{n}\sum_{i=1}^ng_ih(\tilde{\bz}_i)\big],
    $$
  where $\epsilon_1,\ldots,\epsilon_n$ are independent Rademacher variables,
  and $g_1,\ldots,g_n$ are independent $N(0,1)$ random variables.
  We define the \emph{worst-case} Rademacher complexity as
  $\frak{R}_n(H)=\sup_{S'\in\tilde{\zcal}^n}\frak{R}_{S'}(H)$\label{pg:cardinality}.
\end{definition}

Existing data-dependent analyses build on either the structural result \eqref{Rademacher-maximum-lem-8-1} or \eqref{structural-rademacher},
which either ignores the correlation among predictors associated to individual class labels or requires $f_i$ to be Lipschitz continuous w.r.t. the $\ell_2$-norm.
Below we introduce a new structural complexity result based on the following Lipschitz property w.r.t. a variant of the $\ell_2$-norm.
  The motivation of this Lipschitz continuity is that some multi-class loss functions satisfy \eqref{generalized-lipschitz-condition} with a relatively small
$L_1$ and a relatively large $L_2$, the latter of which is not the influential one since it is involved in a single component.

\begin{definition}[Lipschitz continuity w.r.t. a variant of the $\ell_2$-norm]\label{def:lipschitz-variant-2-norm}
  We say a function $f:\rbb^c\to\rbb$ is \emph{Lipschitz continuous w.r.t. a variant of the $\ell_2$-norm} involving a Lipschitz constant pair $(L_1,L_2)$ and index $r\in\{1,\ldots,c\}$ if
  \begin{equation}\label{generalized-lipschitz-condition}
    |f(\bt)-f(\bt')|\leq L_1\|(t_1-t_1',\ldots,t_c-t_c')\|_2+L_2|t_r-t'_r|,\quad\forall \bt,\bt'\in\rbb^c.
  \end{equation}
\end{definition}


We now present our first core result of this paper, the following structural lemma.
Proofs of results in this section are given in Section \ref{sec:proof-dependent}.

\begin{lemma}[Structural Lemma]\label{lem:GP-structural-lipschitz}
  Let $H$ be a class of functions mapping from $\xcal$ to $\rbb^c$. Let $L_1,L_2\geq0$ be two constants and $r:\nbb\to\ycal$. Let $f_1,\ldots,f_n$
  be a sequence of functions from $\rbb^c$ to $\rbb$. Suppose that for any $i\in\nbb_n$, $f_i$ is Lipschitz continuous w.r.t. a variant of the $\ell_2$-norm involving a Lipschitz constant pair $(L_1,L_2)$ and index $r(i)$.
   Let $g_1,\ldots,g_n, g_{11},\ldots,g_{nc}$ be a sequence of independent $N(0,1)$ random variables.
   Then, for any sample $\{\tilde{\bx}_i\}_{i=1}^n\in\xcal^n$ we have
  \begin{equation}\label{gaussian-lipschitz-l2}
    \ebb_{\bm{g}}\sup_{h\in H}\sum_{i=1}^{n}g_if_i(h(\tilde{\bx}_i))\leq\sqrt{2}L_1\ebb_{\bm{g}}\sup_{h\in H}\sum_{i=1}^n\sum_{j=1}^cg_{ij}h_j(\tilde{\bx}_i)+\sqrt{2}L_2\ebb_{\bm{g}}\sup_{h\in H}\sum_{i=1}^{n}g_ih_{r(i)}(\tilde{\bx}_i).
  \end{equation}
\end{lemma}

Lemma \ref{lem:GP-structural-lipschitz} controls the GC of the multi-class loss function class by that of the original hypothesis class,
thereby removing the dependency on the potentially cumbersome operator $f_i$ in the definition of the loss function class
(for instance for \citet{crammer2002algorithmic}, $f_i$ would be the component-wise maximum).
The above lemma is based on a comparison result (Slepian's lemma, Lemma \ref{lem:gaussian-comparison} below) among the suprema of Gaussian processes.

Equipped with Lemma \ref{lem:GP-structural-lipschitz}, we are now able to present our main results based on GCs.
Eq. \eqref{risk-data-dependent-gaussian} is a data-dependent bound in terms of the GC of the following linear scalar-valued function class
\begin{equation}\label{H-tilde-tau}
  \widetilde{H}_{\tau}:=\{\bv\to\langle\bw,\bv\rangle:\bw,\bv\in H_K^c,
  \tau(\bw)\leq \Lambda,\bv\in\widetilde{S}\},
\end{equation}
where $\widetilde{S}$ is defined as follows
\begin{equation}\label{tilde-S}
  \widetilde{S}:=\Big\{\underbrace{\bp_1(\bx_1),\bp_2(\bx_1),\ldots,\bp_c(\bx_1)}_{\text{induced by }\bx_1},
  \underbrace{\bp_1(\bx_2),\bp_2(\bx_2),\ldots,\bp_c(\bx_2)}_{\text{induced by }\bx_2},\ldots,
  \underbrace{\bp_1(\bx_n),\ldots,\bp_c(\bx_n)}_{\text{induced by }\bx_n}\Big\}
\end{equation}
and, for any $\bx\in\xcal$, we use the notation
\begin{equation}\label{bp-j}
  \bp_j(\bx):=\big(\underbrace{0,\ldots,0}_{j-1},\phi(\bx),\underbrace{0,\ldots,0}_{c-j}\big)\in H_K^c,\quad j\in\nbb_c.
\end{equation}
Note that $\widetilde{H}_{\tau}$ is a class of functions defined on a finite set $\widetilde{S}$. We also introduce
\begin{equation}\label{tilde-S-prime}
  \widetilde{S}'=\Big\{\bp_{y_1}(\bx_1),\bp_{y_2}(\bx_2),\ldots,\bp_{y_n}(\bx_n)\Big\}.
\end{equation}
The terms $\widetilde{S},\widetilde{S}'$ and $\bp_j(\bx)$ are motivated by the following identity
\begin{equation}\label{S-tilde-identity-inner-product}
  \inn{\bw,\bp_k(\bx)}=\Big\langle(\bw_1,\ldots,\bw_c),\big(\underbrace{0,\ldots,0}_{k-1},\phi(\bx),\underbrace{0,\ldots,0}_{c-k}\big)\Big\rangle=\inn{\bw_k,\phi(\bx)},\quad\forall k\in\nbb_c.
\end{equation}
Hence, the right-hand side of \eqref{gaussian-lipschitz-l2} can be rewritten as Gaussian complexities of $\widetilde{H}_\tau$ when $H=H_\tau$.

\begin{theorem}[Data-dependent bounds for general regularizer and Lipschitz continuous loss w.r.t. Def. \ref{def:lipschitz-variant-2-norm}]\label{thm:risk-data-dependent}
  Consider the hypothesis space $H_\tau$ in \eqref{hypothesis-space} with $\tau(\bw)=\|\bw\|$, where $\|\cdot\|$ is a norm defined on $H_K^c$.
  Suppose there exist $L_1,L_2\in\rbb_+$ such that $\Psi_y$ is Lipschitz continuous w.r.t. a variant of the $\ell_2$-norm involving a Lipschitz constant pair $(L_1,L_2)$ and index $y$ for all $y\in\ycal$.
  Then, for any $0<\delta<1$, with probability at least $1-\delta$, we have
  \begin{equation}\label{risk-data-dependent-gaussian}
    A_\tau\leq 2\sqrt{\pi}\Big[L_1c\mathfrak{G}_{\widetilde{S}}(\widetilde{H}_\tau)+L_2\mathfrak{G}_{\widetilde{S}'}(\widetilde{H}_\tau)\Big]
  \end{equation}
  and
  \begin{equation}\label{risk-data-dependent}
   A_\tau\leq \frac{2\Lambda \sqrt{\pi}}{n}\bigg[L_1\ebb_{\bm{g}}\Big\|\big(\sum_{i=1}^{n}g_{ij}\phi(\bx_i)\big)_{j=1}^c\Big\|_*+
  L_2\ebb_{\bm{g}}\big\|\big(\sum_{i\in I_j}g_i\phi(\bx_i)\big)_{j=1}^c\big\|_*\bigg],
  \end{equation}
  where $g_1,\ldots,g_n, g_{11},\ldots,g_{nc}$ are independent $N(0,1)$ random variables.
\end{theorem}

\begin{remark}[Motivation of Lipschitz continuity w.r.t. Def. \ref{def:lipschitz-variant-2-norm}\label{rem:lipschitz-variant}]
  The dominant term on the right-hand side of \eqref{risk-data-dependent-gaussian} is $L_1c\mathfrak{G}_{\widetilde{S}}(\widetilde{H}_\tau)$
  if $L_2=O(\sqrt{c}L_1)$. This explains the motivation in introducing the new structural result \eqref{gaussian-lipschitz-l2}
  to exploit the Lipschitz continuity w.r.t. a variant of the $\ell_2$-norm involving a large $L_2$. For comparison, if we apply
  the previous structural result \eqref{structural-rademacher} for loss functions satisfying \eqref{generalized-lipschitz-condition}, then
  the associated $\ell_2$-Lipschitz constant is  $L_1+L_2$,  resulting in the following bound
  $$
    A_\tau\leq 2\sqrt{\pi}(L_1+L_2)c\mathfrak{R}_{\widetilde{S}}(\widetilde{H}_\tau),
  $$
  which is  worse than \eqref{risk-data-dependent-gaussian} when $L_1=O(L_2)$ since the dominant term becomes $L_2c\mathfrak{R}_{\widetilde{S}}(\widetilde{H}_\tau)$.
  Many popular loss functions satisfy \eqref{generalized-lipschitz-condition} with $L_1=O(L_2)$ \citep{weston1998multi,jenssen2012scatter,lapin2015top}.
  For example, the loss function used in the top-$k$ SVM \citep{lapin2015top} satisfies \eqref{generalized-lipschitz-condition} with $(L_1,L_2)=(\frac{1}{\sqrt{k}},1)$,
  which, as we will show, allows us to derive data-dependent bounds with no dependencies on the number of classes by setting $k$ proportional to $c$.
  In comparison, the $(k^{-\frac{1}{2}}+1)$-Lipschitz continuity w.r.t. $\ell_2$-norm does not capture the special structure of the top-$k$ loss function
  since $k^{-\frac{1}{2}}$ is dominated by the constant $1$.
  As further examples, the loss function in \citet{weston1998multi} satisfies \eqref{generalized-lipschitz-condition} with
  $(L_1,L_2)=(\sqrt{c},c)$, while the loss function
  in \citet{jenssen2012scatter} satisfies \eqref{jenssen} with $(L_1,L_2)=(0,1)$.
\end{remark}

%

We now consider two applications of Theorem \ref{thm:risk-data-dependent} by considering $\tau(\bw)=\|\bw\|_{2,p}$ defined on $H_K^c$~\citep{lei2015multi} and $\tau(W)=\|W\|_{S_p}$ defined on $\rbb^{d\times c}$~\citep{amit2007uncovering}, respectively. 

\begin{corollary}[Data-dependent bound for $\ell_p$-norm regularizer and Lipschitz continuous loss w.r.t. Def. \ref{def:lipschitz-variant-2-norm}]\label{cor:risk-dependent-lp}
  Con\-sider the hypothesis space $H_{p,\Lambda}:=H_{\tau,\Lambda}$ in \eqref{hypothesis-space} with $\tau(\bw)=\|\bw\|_{2,p},p\geq1$.
  If there exist $L_1,L_2\in\rbb_+$ such that $\Psi_y$ is Lipschitz continuous w.r.t. a variant of the $\ell_2$-norm involving a Lipschitz constant pair $(L_1,L_2)$ and index $y$ for all $y\in\ycal$,
  then for any $0<\delta<1$, the following inequality holds with probability at least $1-\delta$ (we use the abbreviation $A_p=A_\tau$ with $\tau(\bw)=\|\bw\|_{2,p}$)
  \begin{equation}\label{risk-dependent-lp}
    A_p\leq \frac{2\Lambda \sqrt{\pi}}{n}\Big[\sum_{i=1}^{n}K(\bx_i,\bx_i)\Big]^{\frac{1}{2}}\inf_{q\geq p}\Big[L_1(q^*)^{\frac{1}{2}}c^{\frac{1}{q^*}}
    +L_2(q^*)^{\frac{1}{2}}\max(c^{\frac{1}{q^*}-\frac{1}{2}},1)\Big].
  \end{equation}
\end{corollary}

\begin{corollary}[Data-dependent bound for Schatten-$p$ norm regularizer and Lipschitz continuous loss w.r.t. Def. \ref{def:lipschitz-variant-2-norm}]\label{cor:risk-dependent-shatten}
	Let $\phi$ be the identity map and represent $\bw$ by a matrix $W\in\rbb^{d\times c}$.
  Consider the hypothesis space $H_{S_p,\Lambda}:=H_{\tau,\Lambda}$ in \eqref{hypothesis-space} with $\tau(W)=\|W\|_{S_p},p\geq1$.
  If there exist $L_1,L_2\in\rbb_+$ such that $\Psi_y$ is Lipschitz continuous w.r.t. a variant of the $\ell_2$-norm involving a Lipschitz constant pair $(L_1,L_2)$ and index $y$ for all $y\in\ycal$, then for any $0<\delta<1$ with probability at least $1-\delta$, we have (we use the abbreviation $A_{S_p}=A_\tau$ with $\tau(W)=\|W\|_{S_p}$)
  \begin{equation}\label{risk-dependent-shatten}
    A_{S_p}\leq\begin{cases}
                \frac{2^{\frac{3}{4}}\pi\Lambda}{n\sqrt{e}}\inf\limits_{p\leq q\leq 2}(q^*)^{1\over 2}\bigg\{(L_1c^{\frac{1}{q^*}}+L_2)\Big[\sum\limits_{i=1}^{n}\|\bx_i\|_2^2\Big]^{\frac{1}{2}}+
  L_1c^{\frac{1}{2}}\Big\|\sum\limits_{i=1}^{n}\bx_i\bx_i^\top\Big\|_{S_{\frac{q^*}{2}}}^{\frac{1}{2}}\bigg\}, & \mbox{if } p\leq2, \\
                \frac{2^{\frac{5}{4}}\pi\Lambda\big(L_1c^{\frac{1}{2}}+L_2\big)\min\{c,d\}^{\frac{1}{2}-\frac{1}{p}}}{n\sqrt{e}}\Big[\sum_{i=1}^{n}\|\bx_i\|_2^2\Big]^{\frac{1}{2}}, & \mbox{otherwise}.
              \end{cases}
  \end{equation}
\end{corollary}

In comparison to Corollary \ref{cor:risk-dependent-lp}, the error bound of Corollary \ref{cor:risk-dependent-shatten} involves an additional term\linebreak $O\big(c^{1\over2}n^{-1}\big\|\sum_{i=1}^{n}\bx_i\bx_i^\top\big\|^{1\over2}_{S_{\frac{q^*}{2}}}\big)$ for the case $p\leq2$ due to the need of applying non-commutative  Khintchine-Kahane inequality \eqref{khitchine-kahane-matrix} for Schatten norms.
As we will show in Section \ref{sec:applications}, we can derive from Corollaries \ref{cor:risk-dependent-lp}
and \ref{cor:risk-dependent-shatten} error bounds with sublinear dependencies on the number of classes for $\ell_p$-norm and Schatten-$p$ norm MC-SVMs.
Furthermore, the dependency is logarithmic for $\ell_{p}$-norm MC-SVM \citep{lei2015multi} when $p$ approaches $1$.

\subsection{Data-dependent Bounds by Covering Numbers\label{sec:data-independent-bounds}}

The data-dependent generalization bounds given in subsection \ref{sec:data-dependent-bound} assume the loss function
$\Psi_y$ to be Lipschitz continuous w.r.t. a variant of the $\ell_2$-norm. However, some typical loss functions used in
the multi-class setting are Lipschitz continuous w.r.t. the much milder $\ell_\infty$-norm with a comparable Lipschitz
constant \citep{zhang2004statistical}. This mismatch between the norms w.r.t. which the Lipschitz continuity is measured
requires an additional step of controlling the $\ell_\infty$-norm of vector-valued predictors by the $\ell_2$-norm in the
application of Theorem \ref{thm:risk-data-dependent}, at the cost of a possible multiplicative factor of $\sqrt{c}$.
This subsection aims to avoid this loss in the class-size dependency by presenting data-dependent analysis based on empirical
$\ell_\infty$-norm CNs to directly use the Lipschitz continuity measured by the $\ell_\infty$-norm.

The key step in this approach lies in estimating the empirical CNs of the loss function class
\begin{equation}\label{loss-function-class}
  F_{\tau,\Lambda}:=\{(\bx,y)\to\Psi_y(h^{\bw}(\bx)):h^{\bw}\in H_\tau\}.
\end{equation}
A difficulty towards this purpose consists in the non-linearity of $F_{\tau,\Lambda}$
and the fact that $h^{\bw}\in H_{\tau}$ takes vector-valued outputs,
while standard analyses are limited to scalar-valued and essentially linear (kernel) function classes \citep{zhang2002covering,zhou2002covering,zhou2003capacity}.
The way we bypass this obstacle is to consider a related linear scalar-valued function class $\widetilde{H}_{\tau}$ defined in \eqref{H-tilde-tau}.
A key motivation in introducing $\widetilde{H}_{\tau}$ is that the CNs of $F_{\tau,\Lambda}$ w.r.t. $\bx_1,\ldots,\bx_n$
(CNs are defined in subsection \ref{sec:proof-independent}) can be related to that of the function class
$\{\bv\to\inn{\bw,\bv}:\tau(\bw)\leq\Lambda\}$,
w.r.t. the set $\widetilde{S}$ defined in \eqref{tilde-S}.
The latter can be conveniently tackled since it is a linear and scalar-valued function class, to which standard arguments apply.
In more details, to approximate the projection of $F_{\tau,\Lambda}$ onto the examples $S$ with $(\epsilon,\ell_\infty)$-covers (cf. Definition \ref{def:covering-number} below),
the $\ell_\infty$-Lipschitz continuity of the loss function requires us to approximate the set $\big\{\big(\langle\bw_j,\phi(\bx_i)\rangle_{i\in\nbb_n,j\in\nbb_c}\big):\tau(\bw)\leq\Lambda\big\}$,
which, according to \eqref{S-tilde-identity-inner-product}, is exactly the projection of $\widetilde{H}_\tau$ onto $\widetilde{S}$:
$\big\{\big(\langle\bw,\bp_j(\bx_i)\rangle_{i\in\nbb_n,j\in\nbb_c}\big):\tau(\bw)\leq\Lambda\big\}$. This motivates the definition of $\widetilde{H}_\tau$
in \eqref{H-tilde-tau} and $\widetilde{S}$ in \eqref{tilde-S}.

Theorem \ref{thm:rademacher-independent-bound} reduces the estimation of $\mathfrak{R}_S(F_{\tau,\Lambda})$ to bounding $\frak{R}_{nc}(\widetilde{H}_{\tau})$,
based on which the data-dependent error bounds are given in Theorem \ref{thm:risk-data-independent}.
Note that $\mathfrak{R}_{nc}(\widetilde{H}_\tau)$ is data-dependent since $\widetilde{H}_\tau$ is a class of functions defined on a finite set
induced by training examples. The proofs of complexity bounds in Proposition \ref{prop:rademacher-independent-lp} and Proposition \ref{prop:rademacher-independent-schatten}
are given in subsection \ref{sec:proof-independent-rademacher-lp} and Appendix \ref{sec:proof-independent-rademacher-schatten}, respectively.
The proofs of error bounds in this subsection are given in subsection \ref{sec:proof-independent}.

\begin{theorem}[Worst-case RC bound]\label{thm:rademacher-independent-bound}
  Suppose that $\Psi_y$ is $L$-Lipscthiz continuous w.r.t. the $\ell_\infty$-norm for any $y\in\ycal$ and assume that $\hat{B}_\Psi\leq 2e\hat{B}ncL$. Then the RC of $F_{\tau,\Lambda}$ can be bounded by
  $$
  \mathfrak{R}_S(F_{\tau,\Lambda}) \leq 16L\sqrt{c\log 2}\frak{R}_{nc}(\widetilde{H}_{\tau})\Big(1+\log_2^{3\over2}\frac{\hat{B}n\sqrt{c}}{\mathfrak{R}_{nc}(\widetilde{H}_\tau)}\Big).
  $$
\end{theorem}

\begin{theorem}[Data-dependent bounds for general regularizer and Lipschitz continuous loss function w.r.t. $\|\cdot\|_\infty$]\label{thm:risk-data-independent}
  Under the condition of Theorem \ref{thm:rademacher-independent-bound}, for any $0<\delta<1$, with probability at least $1-\delta$, we have
  $$
    A_{\tau}\leq27L\sqrt{c}\frak{R}_{nc}(\widetilde{H}_{\tau})\Big(1+\log_2^{3\over2}\frac{\hat{B}n\sqrt{c}}{\mathfrak{R}_{nc}(\widetilde{H}_\tau)}\Big).
  $$
\end{theorem}

The application of Theorem \ref{thm:risk-data-independent} requires to control the \emph{worst-case} RC of the linear function class $\widetilde{H}_{\tau}$ from both below and above, to which the following two propositions give respective tight estimates for $\tau(\bw)=\|\bw\|_{2,p}$ defined on $H_K^c$~\citep{lei2015multi} and $\tau(W)=\|W\|_{S_p}$ defined on $\rbb^{d\times c}$~\citep{amit2007uncovering}.

\begin{proposition}[Lower and upper bound on worst-case RC for $\ell_p$-norm regularizer]\label{prop:rademacher-independent-lp}
  For $\tau(\bw)=\|\bw\|_{2,p},p\geq1$ in \eqref{H-tilde-tau}, the function class $\widetilde{H}_\tau$ becomes
  $$
    \widetilde{H}_p:=\big\{\bv\to\inn{\bw,\bv}:\bw,\bv\in H_K^c,\|\bw\|_{2,p}\leq\Lambda,\bv\in\widetilde{S}\big\}.
  $$
  The RC of $\widetilde{H}_p$ can be upper and lower bounded by
  \begin{equation}\label{rademacher-independent-lp}
    \Lambda \max_{i\in\nbb_n}\|\phi(\bx_i)\|_2(2n)^{-\frac{1}{2}}c^{-\frac{1}{\max(2,p)}}\leq\frak{R}_{nc}(\widetilde{H}_p)\leq \Lambda \max_{i\in\nbb_n}\|\phi(\bx_i)\|_2n^{-\frac{1}{2}}c^{-\frac{1}{\max(2,p)}}.
  \end{equation}
\end{proposition}

\begin{remark}[Phase Transition for $p$-norm regularized space]\label{rem:phase}
  We see an interesting phase transition at $p=2$. The \emph{worst-case} RC of $\widetilde{H}_p$ decays as $O((nc)^{-\frac{1}{2}})$
  for the case $p\leq2$, and decays as $O(n^{-\frac{1}{2}}c^{-\frac{1}{p}})$ for the case $p>2$. Indeed, the definition of $\widetilde{S}$ by
  \eqref{tilde-S} implies $\|\bv\|_{2,\infty}=\|\bv\|_{2,p}$ for all $\bv\in\widetilde{S}$ and $p\geq1$ (sparsity of elements in $\widetilde{S}$),
  from which we derive the following identity
  \begin{equation}\label{identity-tilde-S}
    \max_{\bv^i\in\widetilde{S}:i\in\nbb_{nc}}\sum_{j=1}^{c}\sum_{i=1}^{nc}\|\bv_j^i\|_2^2
  =\max_{\bv^i\in\widetilde{S}:i\in\nbb_{nc}}\sum_{i=1}^{nc}\|\bv^i\|_{2,\infty}^2=nc\max_{i\in\nbb_n}\|\phi(\bx_i)\|_2^2,
  \end{equation}
  where $\bv_j^i$ is the $j$-th component of $\bv^i\in\widetilde{S}$.
  That is, we have an automatic constraint on $\big\|\big(\sum_{i=1}^{nc}\|\bv_j^i\|_2^2\big)_{j=1}^c\big\|_1$ for all $\bv^i\in\widetilde{S},i\in\nbb_{nc}$.
  Furthermore, according to \eqref{data-independent>2-upper}, we know $nc\mathfrak{R}_{nc}(\widetilde{H}_p)$ can be controlled
  in terms of $\max_{\bv^i\in\widetilde{S}:i\in\nbb_n}\big\|\big(\sum_{i=1}^{nc}\|\bv_j^i\|_2^2\big)_{j=1}^c\big\|_{\frac{p^*}{2}}$, for which an appropriate $p$
  to fully use the identity \eqref{identity-tilde-S} is $p=2$. This explains the phase transition phenomenon.
\end{remark}

\begin{proposition}[Lower and upper bound on worst-case RC for Schatten-$p$ norm regularizer]\label{prop:rademacher-independent-schatten}
Let $\phi$ be the identity map and represent $\bw$ by a matrix $W\in\rbb^{d\times c}$. For $\tau(W)=\|W\|_{S_p},p\geq1$ in \eqref{H-tilde-tau}, the function class $\widetilde{H}_\tau$ becomes
\begin{equation}\label{widetilde-H-schatten}
    \widetilde{H}_{S_p}:=\big\{V\to\inn{W,V}:W\in\rbb^{d\times c},\|W\|_{S_p}\leq\Lambda, V\in\widetilde{S}\subset\rbb^{d\times c}\big\}.
\end{equation}
The RC of $\widetilde{H}_{S_p}$ can be upper and lower bounded by
\begin{equation}\label{rademacher-independent-schatten}
\begin{cases}
  \Lambda \max\limits_{i\in\nbb_n}\|\bx_i\|_2(2nc)^{-\frac{1}{2}}\leq \frak{R}_{nc}(\widetilde{H}_{S_p})\leq \Lambda \max\limits_{i\in\nbb_n}\|\bx_i\|_2(nc)^{-\frac{1}{2}}, & \mbox{if } p\leq 2, \\
  \Lambda \max\limits_{i\in\nbb_n}\|\bx_i\|_2(2nc)^{-\frac{1}{2}} \leq \frak{R}_{nc}(\widetilde{H}_{S_p})
  \leq \frac{\Lambda \max_{i\in\nbb_n}\|\bx_i\|_2\min\{c,d\}^{\frac{1}{2}-\frac{1}{p}}}{\sqrt{nc}}, & \mbox{otherwise}.
\end{cases}
\end{equation}
\end{proposition}

The associated data-dependent error bounds, given in Corollary \ref{cor:risk-independent-lp} and Corollary \ref{cor:risk-independent-shatten}, are then immediate.

\begin{corollary}[Data-dependent bound for $\ell_p$-norm regularizer and Lipschitz continuous loss w.r.t. $\|\cdot\|_\infty$]\label{cor:risk-independent-lp}
    Consider the hypothesis space $H_{p,\Lambda}:=H_{\tau,\Lambda}$ in \eqref{hypothesis-space} with $\tau(\bw)=\|\bw\|_{2,p},p\geq1$. Assume that $\Psi_y$ is $L$-Lipschitz continuous w.r.t. $\ell_\infty$-norm for any $y\in\ycal$ and $\hat{B}_\Psi\leq 2e\hat{B}ncL$. Then,
    for any $0<\delta<1$ with probability $1-\delta$, we have
    $$
      A_{p}\leq\frac{27L\Lambda \max_{i\in\nbb_n}\|\phi(\bx_i)\|_2c^{\frac{1}{2}-\frac{1}{\max(2,p)}}}{\sqrt{n}}\Big(1+\log_2^{3\over2}\big(\sqrt{2}n^{3\over2}c\big)\Big).
    $$
\end{corollary}

\begin{corollary}[Data-dependent bound for Schatten-$p$ norm regularizer and Lipschitz continuous loss w.r.t. $\ell_\infty$-norm]\label{cor:risk-independent-shatten}
   Let $\phi$ be the identity map and represent $\bw$ by a matrix $W\in\rbb^{d\times c}$.
   Consider the hypothesis space $H_{S_p,\Lambda}:=H_{\tau,\Lambda}$ in \eqref{hypothesis-space} with $\tau(W)=\|W\|_{S_p},p\geq1$.
   Assume that $\Psi_y$ is $L$-Lipschitz continuous w.r.t. $\ell_\infty$-norm for any $y\in\ycal$ and $\hat{B}_\Psi\leq 2e\hat{B}ncL$. Then, for any $0<\delta<1$ with probability $1-\delta$, we have
    $$
      A_{S_p}\leq			
      \begin{cases}
         \frac{27L\Lambda \max_{i\in\nbb_n}\|\bx_i\|_2}{\sqrt{n}}\Big(1+\log_2^{3\over2}\big(\sqrt{2}n^{3\over2}c\big)\Big), &\mbox{ if } p\leq2,  \\
         \frac{27L\Lambda \max_{i\in\nbb_n}\|\bx_i\|_2\min\{c,d\}^{\frac{1}{2}-\frac{1}{p}}}{\sqrt{n}}\Big(1+\log_2^{3\over2}\big(\sqrt{2}n^{3\over2}c\big)\Big), & \mbox{otherwise}.
      \end{cases}
    $$
\end{corollary}


%

\section{Applications\label{sec:applications}}

In this section, we apply the general results in subsections \ref{sec:data-dependent-bound} and \ref{sec:data-independent-bounds} to study data-dependent
error bounds for some prominent multi-class learning methods. We also compare our data-dependent bounds with the state of the art for different
MC-SVMs. In subsection \ref{sec:discussion}, we make an in-depth discussion to compare error bounds based on GCs with those based on CNs.

\subsection{Classic MC-SVMs}\label{sec:classMCSVM}
We first apply the results from the previous section to several classic MC-SVMs.
For this purpose, we need to show that the associated loss functions satisfy Lipschitz conditions.

To this end, for any $h:\xcal\to\rbb^c$, we denote by
\begin{equation}\label{margin}
  \rho_h(\bx,y):= h_y(\bx)-\max_{y':y'\neq y}h_{y'}(\bx)
\end{equation}
the margin of the model $h$ at $(\bx,y)$. It is clear that the prediction rule $h$ makes an error at $(\bx,y)$ if $\rho_h(\bx,y)< 0$.
In Examples \ref{exp:loss-margin}, \ref{exp:ww}, and \ref{exp:llw} below, we assume that $\ell:\rbb\to\rbb_+$ is a decreasing and $L_\ell$-Lipschitz function.

\begin{example}[Multi-class margin-based loss \citep{crammer2002algorithmic}]\label{exp:loss-margin}
  The loss function defined as
  \begin{equation}\label{loss-function-margin}
    \Psi_y^\ell(\bt):=\max_{y':y'\neq y}\ell(t_y-t_{y'}),\quad\forall\bt\in\rbb^c
  \end{equation}
  is $(2L_\ell)$-Lipschitz continuous w.r.t. $\ell_\infty$-norm and $\ell_2$-norm. Furthermore, we have $\ell(\rho_h(\bx,y))=\Psi_y^\ell(h(\bx))$.
\end{example}

The loss function $\Psi_y^\ell$ defined above in Eq. \eqref{loss-function-margin}
is a margin-based loss function widely used in multi-class classification \citep{crammer2002algorithmic} and structured prediction \citep{mohri2012foundations}.

Next, we study the multinomial logistic loss $\Psi^m_y$ defined below,
which is used in multinomial logistic regression~\citep[][Chapter 4.3.4]{bishop2006pattern}.

\begin{example}[Multinomial logistic loss]\label{exp:soft-max}
  The multinomial logistic loss $\Psi^m_y(\bt)$ defined as
  \begin{equation}\label{soft-max}
    \Psi^m_y(\bt):=\log\big(\sum_{j=1}^c\exp(t_j-t_y)\big),\quad\forall\bt\in\rbb^c
  \end{equation}
  is $2$-Lipschitz continuous w.r.t. the $\ell_\infty$-norm and the $\ell_2$-norm.
\end{example}

The loss $\tilde{\Psi}_y^{\ell}$ defined in Eq. \eqref{ww} below is used in \cite{weston1998multi} to make pairwise comparisons among components of the predictor.

\begin{example}[Loss function used in \cite{weston1998multi}]\label{exp:ww}
  The loss function defined as
  \begin{equation}\label{ww}
    \tilde{\Psi}_y^{\ell}(\bt)=\sum_{j=1}^{c}\ell(t_y-t_j),\quad\forall\bt\in\rbb^c
  \end{equation}
  is Lipschitz continuous w.r.t. a variant of the $\ell_2$-norm involving the Lipschitz constant pair
  $(L_\ell\sqrt{c},L_\ell c)$ and index $y$. Furthermore, it is also $(2L_\ell c)$-Lipschitz continuous w.r.t. the $\ell_\infty$-norm.
\end{example}

Finally, the loss $\bar{\Psi}_y^{\ell}$ defined in Eq. \eqref{llw} and the loss $\hat{\Psi}_y^{\ell}$ defined in Eq. \eqref{jenssen} are used separately
in \cite{lee2004multicategory} based on constrained comparisons.
\begin{example}[Loss function used in \cite{lee2004multicategory}]\label{exp:llw}
  The loss function defined as
  \begin{equation}\label{llw}
    \bar{\Psi}_y^{\ell}(\bt)=\sum_{j=1,j\neq y}^{c}\ell(-t_j),\quad\forall\bt\in\Omega=\{\tilde{\bt}\in\rbb^c:\sum_{j=1}^{c}\tilde{t}_j=0\}
  \end{equation}
  is $(L_\ell\sqrt{c})$-Lipschitz continuous w.r.t. the $\ell_2$-norm and $(L_\ell c)$-Lipschitz continuous w.r.t. the $\ell_\infty$-norm.
\end{example}

\begin{example}[Loss function used in \cite{jenssen2012scatter}]\label{exp:jenssen}
  The loss function defined as
  \begin{equation}\label{jenssen}
    \hat{\Psi}_y^{\ell}(\bt)=\ell(t_y),\quad\forall\bt\in\Omega=\{\tilde{\bt}\in\rbb^c:\sum_{j=1}^{c}\tilde{t}_j=0\}
  \end{equation}
  is Lipschitz continuous w.r.t. a variant of the $\ell_2$-norm involving the Lipschitz constant pair $(0,L_\ell)$ and index $y$, and $L_\ell$-Lipschitz continuous
  w.r.t. the $\ell_\infty$-norm.
\end{example}

The following data-dependent error bounds are immediate by plugging the Lipschitz conditions established in Examples \ref{exp:loss-margin}, \ref{exp:soft-max}, \ref{exp:ww}, \ref{exp:llw} and \ref{exp:jenssen} into Corollaries \ref{cor:risk-dependent-lp}, \ref{cor:risk-dependent-shatten}, \ref{cor:risk-independent-lp} and \ref{cor:risk-independent-shatten}, separately.
In the following, we always assume that the condition $\hat{B}_\Psi\leq 2e\hat{B}ncL$ holds, where $L$ is the Lipschitz constant in Theorem \ref{thm:rademacher-independent-bound}.

\begin{corollary}[Generalization bounds for Crammer and Singer MC-SVM]\label{cor:Cramer}
  Consider the MC-SVM in \cite{crammer2002algorithmic} with the loss function $\Psi_y^{\ell}$ \eqref{loss-function-margin} and the hypothesis space $H_\tau$ with $\tau(\bw)=\|\bw\|_{2,2}$. Let $0<\delta<1$. Then,
  \begin{enumerate}[(a)]
    \item with probability at least $1-\delta$, we have $A_2\leq \frac{4L_\ell\Lambda \sqrt{2\pi c}}{n}\big[\sum_{i=1}^{n}K(\bx_i,\bx_i)\big]^{\frac{1}{2}}$ (by GCs);
    \item with probability at least $1-\delta$, we have $A_2\leq \frac{54L_\ell\Lambda \max_{i\in\nbb_n}\|\phi(\bx_i)\|_2}{\sqrt{n}}\big(1+\log_2^{3\over2}\big(\sqrt{2}n^{3\over2}c\big)\big)$ (by CNs).
  \end{enumerate}
\end{corollary}

Analogous to Corollary \ref{cor:Cramer}, we have the following corollary on error bounds for the multinomial logistic regression in \cite{bishop2006pattern}.
\begin{corollary}[Generalization bounds for multinomial logistic regression]\label{cor:logistic}
  Consider the multinomial logistic regression with the loss function $\Psi_y^{\ell}$ \eqref{soft-max} and the hypothesis space $H_\tau$ with $\tau(\bw)=\|\bw\|_{2,2}$. Let $0<\delta<1$. Then,
  \begin{enumerate}[(a)]
    \item with probability at least $1-\delta$, we have $A_2\leq \frac{4\Lambda \sqrt{2\pi c}}{n}\big[\sum_{i=1}^{n}K(\bx_i,\bx_i)\big]^{\frac{1}{2}}$ (by GCs);
    \item with probability at least $1-\delta$, we have $A_2\leq \frac{54\Lambda \max_{i\in\nbb_n}\|\phi(\bx_i)\|_2}{\sqrt{n}}\big(1+\log_2^{3\over2}\big(\sqrt{2}n^{3\over2}c\big)\big)$ (by CNs).
  \end{enumerate}
\end{corollary}

The following three corollaries give error bounds for MC-SVMs in \cite{weston1998multi,lee2004multicategory,jenssen2012scatter}.
The MC-SVM in Corollary \ref{cor:jenssen} is a minor variant of that in \cite{jenssen2012scatter} with a fixed functional margin.

\begin{corollary}[Generalization bounds for \citeauthor{weston1998multi} MC-SVM]\label{cor:ww}
  Consider the MC-SVM in \citet{weston1998multi} with the loss function $\tilde{\Psi}_y^{\ell}$ \eqref{ww} and the hypothesis space $H_\tau$ with $\tau(\bw)=\|\bw\|_{2,2}$. Let $0<\delta<1$. Then,
  \begin{enumerate}[(a)]
    \item with probability at least $1-\delta$, we have $A_2\leq \frac{4L_\ell \Lambda c\sqrt{2\pi}}{n}\big[\sum_{i=1}^{n}K(\bx_i,\bx_i)\big]^{1\over 2}$ (by GCs);
    \item with probability at least $1-\delta$, we have $A_2\leq \frac{54L_\ell\Lambda c\max_{i\in\nbb_n}\|\phi(\bx_i)\|_2}{\sqrt{n}}\big(1+\log_2^{3\over2}\big(\sqrt{2}n^{3\over2}c\big)\big)$ (by CNs).
  \end{enumerate}
\end{corollary}

\begin{corollary}[Generalization bounds for \citeauthor{lee2004multicategory} MC-SVM]\label{cor:llw}
  Consider the MC-SVM in \citet{lee2004multicategory} with the loss function $\bar{\Psi}_y^{\ell}$ \eqref{llw} and the hypothesis space $H_\tau$ with $\tau(\bw)=\|\bw\|_{2,2}$. Let $0<\delta<1$. Then,
  \begin{enumerate}[(a)]
    \item with probability at least $1-\delta$, we have $A_2\leq \frac{2L_\ell \Lambda c\sqrt{2\pi}}{n}\big[\sum_{i=1}^{n}K(\bx_i,\bx_i)\big]^{1\over 2}$ (by GCs);
    \item with probability at least $1-\delta$, we have $A_2\leq \frac{27L_\ell\Lambda c\max_{i\in\nbb_n}\|\phi(\bx_i)\|_2}{\sqrt{n}}\big(1+\log_2^{3\over2}\big(\sqrt{2}n^{3\over2}c\big)\big)$ (by CNs).
  \end{enumerate}
\end{corollary}

\begin{corollary}[Generalization bounds for \citeauthor{jenssen2012scatter} MC-SVM]\label{cor:jenssen}
  Consider the MC-SVM in \citet{jenssen2012scatter} with the loss function $\hat{\Psi}_y^{\ell}$ \eqref{llw} and the hypothesis space $H_\tau$ with $\tau(\bw)=\|\bw\|_{2,2}$. Let $0<\delta<1$. Then,
  \begin{enumerate}[(a)]
    \item with probability at least $1-\delta$, we have $A_2\leq \frac{2L_\ell \Lambda \sqrt{2\pi}}{n}\big[\sum_{i=1}^{n}K(\bx_i,\bx_i)\big]^{1\over 2}$ (by GCs);
    \item with probability at least $1-\delta$, we have $A_2\leq \frac{27L_\ell\Lambda \max_{i\in\nbb_n}\|\phi(\bx_i)\|_2}{\sqrt{n}}\big(1+\log_2^{3\over2}\big(\sqrt{2}n^{3\over2}c\big)\big)$ (by CNs).
  \end{enumerate}
\end{corollary}

\begin{remark}[Comparison with the state of the art]
It is interesting to compare the above error bounds with the best known results in the literature.
To start with, the data-dependent error bound of Corollary \ref{cor:Cramer} (a) exhibits a square-root dependency on the number of classes,
matching the state of the art from the conference version of this paper \citep{lei2015multi}, which is significantly improved to
a \emph{logarithmic} dependency in Corollary \ref{cor:Cramer} (b).

 The error bound in Corollary \ref{cor:ww} (a) for the MC-SVM by \citet{weston1998multi} scales linearly in $c$.
On the other hand, according to Example \ref{exp:ww}, it is evident that $\tilde{\Psi}_y^{\ell}$ is $(c+\sqrt{c})L_\ell$-Lipschitz
continuous w.r.t. the $\ell_2$-norm, for any $y\in \ycal$.
Therefore, one can apply the structural result \eqref{structural-rademacher} from \citep{maurer2016vector,cortes2016structured}
to derive the bound $O(c^{3\over2}n^{-1}[\sum_{i=1}^{n}K(\bx_i,\bx_i)]^{1\over2})$.
Furthermore, according to Example \ref{exp:jenssen}, $\hat{\Psi}_y^{\ell}$ is $L_\ell$-Lipschitz continuity w.r.t. $\|\cdot\|_2$.
Hence, one can apply the structural result \eqref{structural-rademacher} to derive the bound $O(c^{1\over2}n^{-1}[\sum_{i=1}^{n}K(\bx_i,\bx_i)]^{1\over2})$,
which is worse than the error bound $O(n^{-1}[\sum_{i=1}^{n}K(\bx_i,\bx_i)]^{1\over2})$ based on Lemma \ref{lem:GP-structural-lipschitz} and stated
in Corollary \ref{cor:jenssen} (a), which has no dependency on the number of classes.
This justifies the effectiveness of our new structural result (Lemma \ref{lem:GP-structural-lipschitz}) in capturing the Lipschitz continuity of loss functions w.r.t.
a variant of the $\ell_2$-norm to allow for a relatively large $L_2$, which is precisely the case for some popular MC-SVMs \citep{weston1998multi,jenssen2012scatter,lapin2015top}.

Note that for the MC-SVMs by \citet{weston1998multi,lee2004multicategory,jenssen2012scatter}, the GC-based error bounds are tighter than the
corresponding error bounds based on CNs, up to logarithmic factors.
\end{remark}

\subsection{Top-$k$ MC-SVM\label{sec:top-k-svm}}
Motivated by the ambiguity in class labels caused by the rapid increase in number of classes in modern computer vision benchmarks, \citet{lapin2015top,lapin2016loss} introduce the top-$k$ MC-SVM by using the top-$k$ hinge loss to allow $k$ predictions for each object $\bx$. For any $\bt\in\rbb^c$, let the bracket $[\cdot]$ denote a permutation such that $[j]$ is the index of the $j$-th largest score,
i.e., $t_{[1]}\geq t_{[2]}\geq\cdots\geq t_{[c]}$.

\begin{example}[Top-$k$ hinge loss \citep{lapin2015top}]\label{exp:top-k}
  The top-$k$ hinge loss defined by
  \begin{equation}\label{top-k-loss}
    \Psi_y^k(\bt)=\max\Big\{0,\frac{1}{k}\sum_{j=1}^{k}(1_{y\neq 1}+t_1-t_{y},\ldots,1_{y\neq c}+t_c-t_{y})_{[j]}\Big\},\quad\forall\bt\in\rbb^c
  \end{equation}
  is Lipschitz continuous w.r.t. a variant of the $\ell_2$-norm involving a Lipschitz constant pair $\big(\frac{1}{\sqrt{k}},1\big)$ and index $y$. Furthermore, it is also
  $2$-Lipschitz continuous w.r.t. $\ell_\infty$-norm.
\end{example}

With the Lipschitz conditions established in Example \ref{exp:top-k}, we are now able to give generalization error bounds for the top-$k$ MC-SVM~\citep{lapin2015top}.
\begin{corollary}[Generalization bounds for top-$k$ MC-SVM]\label{cor:top-k-data-dependent}
  Consider the top-$k$ MC-SVM with the loss functions \eqref{top-k-loss} and the hypothesis space $H_\tau$ with $\tau(\bw)=\|\bw\|_{2,2}$. Let $0<\delta<1$. Then,
  \begin{enumerate}[(a)]
    \item with probability at least $1-\delta$, we have $A_2\leq \frac{2\Lambda\sqrt{2\pi}}{n}(c^{1\over 2}k^{-\frac{1}{2}}+1)[\sum_{i=1}^{n}K(\bx_i,\bx_i)]^{1\over2}$ (by GCs);
    \item with probability at least $1-\delta$, we have $A_2\leq\frac{54\Lambda \max_{i\in\nbb_n}\|\phi(\bx_i)\|_2}{\sqrt{n}}\big(1+\log_2^{3\over2}\big(\sqrt{2}n^{3\over2}c\big)\big)$ (by CNs).
  \end{enumerate}
\end{corollary}

\begin{remark}[Comparison with the state of the art]
An appealing property of Corollary \ref{cor:top-k-data-dependent} (a) is the involvement of the factor $k^{-\frac{1}{2}}$.
Note that we even can get error bounds with no dependencies on $c$ if we choose $k>\tilde{C}c$ for a universal constant $\tilde{C}$.

Comparing our result to the state of the art, it follows again from Example \ref{exp:top-k} that $\Psi_y^k$ is $(1+k^{-\frac{1}{2}})$-Lipschitz
continuous w.r.t. the $\ell_2$-norm for all $y\in\ycal$.
Using the structural result \eqref{structural-rademacher} \citep{lei2015multi,cortes2016structured,maurer2016vector},
one can derive an error bound decaying as $O\big(n^{-1}c^{1\over2}\big[\sum_{i=1}^{n}K(\bx_i,\bx_i)\big]^{1\over 2}\big)$,
which is suboptimal to Corollary \ref{cor:top-k-data-dependent} (a) since it does not shed insights on how the parameter $k$ would affect the generalization performance.
Furthermore, the error bound in Corollary \ref{cor:top-k-data-dependent} (b) enjoys a logarithmic dependency on the number of classes.
\end{remark}

\subsection{$\ell_p$-norm MC-SVMs}

In our previous work \citep{lei2015multi}, we introduce the $\ell_p$-norm MC-SVMs as an extension of the Crammer \& Singer MC-SVM
by replacing the associated $\ell_2$-norm regularizer with a general block $\ell_{2,p}$-norm regularizer~\citep{lei2015multi}.
We establish data-dependent error bounds in \citep{lei2015multi}, showing a logarithmic dependency on the number of classes as $p$ decreases to $1$.
The present analysis yields the following bounds, which also hold for the MC-SVM with the multinomial logistic loss
and the block $\ell_{2,p}$-norm regularizer.

\begin{corollary}[Generalization bounds for $\ell_p$-norm MC-SVMs]\label{cor:lp-mcsvm}
  Consider the $\ell_p$-norm MC-SVM with loss function \eqref{loss-function-margin} and the hypothesis space $H_\tau$ with $\tau(\bw)=\|\bw\|_{2,p},p\geq1$. Let $0<\delta<1$. Then,
  \begin{enumerate}[(a)]
    \item with probability at least $1-\delta$, we have:
    $$
      A_p\leq \frac{4L_\ell\Lambda\sqrt{\pi}}{n}\big[\sum_{i=1}^{n}K(\bx_i,\bx_i)\big]^{1\over 2}\inf_{q\geq p}[(q^*)^{1\over 2}c^{1\over {q^*}}] \quad\text{(by GCs)};
    $$
    \item with probability at least $1-\delta$, we have:
    $$
      A_p\leq \frac{54L_\ell\Lambda \max_{i\in\nbb_n}\|\phi(\bx_i)\|_2c^{\frac{1}{2}-\frac{1}{\max(2,p)}}}{\sqrt{n}}\big(1+\log_2^{3\over2}\big(\sqrt{2}n^{3\over2}c\big)\big) \quad\text{(by CNs)}.
    $$
  \end{enumerate}
\end{corollary}

\begin{remark}[Comparison with the state of the art]\label{rem:lp-svm}
  Corollary \ref{cor:lp-mcsvm} (a) is an extension of error bounds in the conference version \citep{lei2015multi} from $1\leq p\leq2$ to the case $p\geq1$.
  We can see how $p$ affects the generalization performance of $\ell_p$-norm MC-SVMs.
  The function $f:\rbb_+\to\rbb_+$ defined by $f(t)=t^{1\over 2}c^{1\over t}$ is monotonically decreasing on the interval $(0,2\log c)$ and increasing on the interval $(2\log c,\infty)$. Therefore, the data-dependent error bounds in Corollary \ref{cor:lp-mcsvm} (a)  transfer to
  $$
    A_p\leq \begin{cases}
              4\Lambda L_\ell\sqrt{\pi p^*}n^{-1}c^{1-\frac{1}{p}}\big[\sum_{i=1}^{n}K(\bx_i,\bx_i)\big]^{\frac{1}{2}}, & \mbox{if } p>\frac{2\log c}{2\log c-1}, \\
              4\Lambda L_\ell(2\pi e\log c)^{1\over 2}n^{-1}\big[\sum_{i=1}^{n}K(\bx_i,\bx_i)\big]^{1\over 2}, & \mbox{otherwise}.
            \end{cases}
  $$
  That is, the dependency on the number of classes would be polynomial with exponent $1/p^*$ if $p>\frac{2\log c}{2\log c-1}$ and logarithmic otherwise.
  On the other hand, the error bounds in Corollary \ref{cor:lp-mcsvm} (b) significantly improve those in Corollary \ref{cor:lp-mcsvm} (a).
  Indeed, the error bounds in Corollary \ref{cor:lp-mcsvm} (b) enjoy a logarithmic dependency on the number of classes if $p\leq 2$ and a polynomial
  dependency with exponent $\frac{1}{2}-\frac{1}{p}$ otherwise (up to logarithmic factors). This phase transition phenomenon at $p=2$ is explained in Remark \ref{rem:phase}.
  It is also clear that error bounds based on CNs outperform those based on GCs by a factor of $\sqrt{c}$ for $p\geq2$ (up to logarithmic factors),
  which, as we will explain in subsection \ref{sec:discussion}, is due to the use of the Lipschitz continuity measured by a norm suitable to the loss function.
\end{remark}

\subsection{Schatten-$p$ Norm MC-SVMs}
\citet{amit2007uncovering} propose to use trace-norm regularization in multi-class classification to uncover shared
structures always existing in the learning regime with many classes.
Here we consider error bounds for the more general Schatten-$p$ norm MC-SVMs.
\begin{corollary}[Generalization bounds for Schatten-$p$ norm MC-SVMs]\label{cor:schatten}
  Let $\phi$ be the identity map and represent $\bw$ by a matrix $W\in\rbb^{d\times c}$.
  Consider Schatten-$p$ norm MC-SVMs with loss functions \eqref{loss-function-margin} and the hypothesis space $H_\tau$ with $\tau(W)=\|W\|_{S_p},p\geq1$. Let $0<\delta<1$. Then,
  \begin{enumerate}[(a)]
    \item with probability at least $1-\delta$, we have:
    $$
      A_{S_p}\leq \begin{cases}
                    \frac{2^{7\over 4}\pi\Lambda L_\ell}{n\sqrt{e}}\inf\limits_{p\leq q\leq 2}(q^*)^{1\over 2}\Big[c^{1\over {q^*}}\big[\sum_{i=1}^{n}\|\bx_i\|_2^2\big]^{1\over 2}+c^{1\over 2}\|\sum_{i=1}^{n}\bx_i\bx_i^\top\|^{1\over 2}_{S_{\frac{q^*}{2}}}\Big], & \mbox{if } p\leq2, \\
                    \frac{2^{9\over 4}\pi\Lambda L_\ell c^{1\over2}\min\{c,d\}^{\frac{1}{2}-\frac{1}{p}}}{n\sqrt{e}}\big[\sum_{i=1}^{n}\|\bx_i\|_2^2\big]^{1\over2}, & \mbox{otherwise}.
                  \end{cases}
    $$
    \item with probability at least $1-\delta$, we have:
    $$
      A_{S_p}\leq\begin{cases}
                   \frac{54L_\ell\Lambda \max_{i\in\nbb_n}\|\bx_i\|_2}{\sqrt{n}}\Big(1+\log_2^{3\over2}\big(\sqrt{2}n^{3\over2}c\big)\Big), & \mbox{if } p\leq2 \\
                   \frac{54L_\ell\Lambda \max_{i\in\nbb_n}\|\bx_i\|_2\min\{c,d\}^{\frac{1}{2}-\frac{1}{p}}}{\sqrt{n}}\Big(1+\log_2^{3\over2}\big(\sqrt{2}n^{3\over2}c\big)\Big), & \mbox{otherwise}.
                 \end{cases}
    $$
  \end{enumerate}
\end{corollary}

\begin{remark}[Analysis of Schatten-$p$ norm MC-SVMs]
Analogous to Remark \ref{rem:lp-svm}, error bounds of Corollary \ref{cor:schatten} (a) transfer to
$$
  \begin{cases}
    O\big(n^{-1}(p^*)^{1\over 2}\big(c^{1\over {p^*}}\big[\sum_{i=1}^{n}\|\bx_i\|_2^2\big]^{1\over2}+c^{1\over2}\|\sum_{i=1}^{n}\bx_i\bx_i^\top\|^{1\over2}_{S_{\frac{p^*}{2}}}\big)\big), & \mbox{if } 2\leq p^*\leq2\log c,\\
    O\big(n^{-1}\sqrt{\log c}\big(\big[\sum_{i=1}^{n}\|\bx_i\|_2^2\big]^{1\over2}+c^{1\over2}\|\sum_{i=1}^{n}\bx_i\bx_i^\top\|_{S_{\log c}}^{1\over2}\big)\big), & \mbox{if } 2<2\log c<p^*, \\
    O\big(n^{-1}c^{1-\frac{1}{p}}\big[\sum_{i=1}^{n}\|\bx_i\|_2^2\big]^{1\over2}\big), & \mbox{if } p>2.
  \end{cases}
$$
As a comparison, error bounds in Corollary \ref{cor:schatten} (b) would decay as $O(n^{-\frac{1}{2}}\log^{3\over2}(n^{3\over2}c))$
if $p\leq2$ and $O\big(n^{-\frac{1}{2}}c^{\frac{1}{2}-\frac{1}{p}}\log^{3\over2}(n^{3\over2}c)\big)$ otherwise, which
significantly outperform those in Corollary \ref{cor:schatten} (a).
\end{remark}

\begin{table}[htbp]
\caption{Comparison of Data-dependent Generalization Error Bounds Derived in This Taper. \normalfont  We use the notation $B_1=\big(\frac{1}{n}\sum_{i=1}^{n}K(\bx_i,\bx_i)\big)^{1\over2}$ and $B_\infty=\max_{i\in\nbb_n}\|\phi(\bx_i)\|_2$. The best bound for each MC-SVM is followed by a bullet. \label{tab:comparison}}
  \centering
  \setlength{\tabcolsep}{0.4em}
  \centering\def\arraystretch{1.2}
    \begin{tabular}{|c|c|c|c|}
    \hline
    MC-SVM & by structural result \eqref{structural-rademacher} &by GCs & by CNs \\ \hline
    Crammer \& Singer & $O\big(B_1n^{-\frac{1}{2}}c^{1\over2}\big)$ & $O\big(B_1n^{-\frac{1}{2}}c^{1\over2}\big)$ & $O\big(B_\infty n^{-\frac{1}{2}}\log^{3\over2}(nc)\big)\ \bullet$ \\ \hline
    Multinomial Logistic & $O\big(B_1n^{-\frac{1}{2}}c^{1\over2}\big)$ & $O\big(B_1n^{-\frac{1}{2}}c^{1\over2}\big)$ & $O\big(B_\infty n^{-\frac{1}{2}}\log^{3\over2}(nc)\big)\ \bullet$ \\ \hline
    \citeauthor{weston1998multi} & $O\big(B_1n^{-\frac{1}{2}}c^{3\over2}\big)$ & $O\big(B_1n^{-\frac{1}{2}}c\big)\ \bullet$ & $O\big(B_\infty n^{-\frac{1}{2}}c\log^{3\over2}(nc)\big)$ \\ \hline
    \citeauthor{lee2004multicategory} & $O\big(B_1n^{-\frac{1}{2}}c\big)\ \bullet$ & $O\big(B_1n^{-\frac{1}{2}}c\big)\ \bullet$ & $O\big(B_\infty n^{-\frac{1}{2}}c\log^{3\over2}(nc)\big)$\\ \hline
    \citeauthor{jenssen2012scatter} &$O\big(B_1n^{-\frac{1}{2}}c^{1\over2}\big)$& $O\big(B_1n^{-\frac{1}{2}}\big)\ \bullet$ & $O\big(B_\infty n^{-\frac{1}{2}}\log^{3\over2}(nc)\big)$\\ \hline
    top-k & $O\big(B_1n^{-\frac{1}{2}}c^{1\over2}\big)$ & $O\big(B_1n^{-\frac{1}{2}}(ck^{-1})^{1\over 2}\big)$ & $O\big(B_\infty n^{-\frac{1}{2}}\log^{3\over2}(nc)\big)\ \bullet$\\ \hline
    $\ell_p$-norm $p\in(1,\infty)$ & $O\big(B_1n^{-\frac{1}{2}}c^{1-\frac{1}{p}}\big)$ & $O\big(B_1n^{-\frac{1}{2}}c^{1-\frac{1}{p}}\big)$  & $O\big(B_\infty n^{-\frac{1}{2}}c^{\frac{1}{2}-\frac{1}{\max(2,p)}}\log^{3\over2}(nc)\big)\ \bullet$ \\ \hline
    Schatten-$p$ $p\in[1,2)$ & $O\big(B_1n^{-\frac{1}{2}}c^{\frac{1}{2}}\big)$  & $O\big(B_1n^{-\frac{1}{2}}c^{\frac{1}{2}}\big)$ & $O\big(B_\infty n^{-\frac{1}{2}}\log^{3\over2}(nc)\big)\ \bullet$ \\ \hline
    Schatten-$p$ $p\in[2,\infty)$ & $O\big(B_1n^{-\frac{1}{2}}c^{1-\frac{1}{p}}\big)$ & $O\big(B_1n^{-\frac{1}{2}}c^{1-\frac{1}{p}}\big)$ & $O\big(B_\infty n^{-\frac{1}{2}}c^{\frac{1}{2}-\frac{1}{p}}\log^{3\over2}(nc)\big)\ \bullet$ \\
    \hline
  \end{tabular}
\end{table}


\subsection{Comparison of the GC   with the CN Approach\label{sec:discussion}}
In this paper, we develop two methods to derive data-dependent error bounds applicable to learning with many classes. We
summarize these two types of error bounds for some specific MC-SVMs in the third and fourth column of Table \ref{tab:comparison},
from which it is clear that each approach may yield better bounds for some MC-SVMs than the other.
For example, for the Crammer \& Singer MC-SVM, the GC-based error bound enjoys a square-root
dependency on the number of classes, while the CN-based bound enjoys a logarithmic dependency.
CN-based error bounds also enjoy significant advantages for $\ell_p$-norm MC-SVMs and Schatten-$p$ norm MC-SVMs.
On the other hand, GC-based analyses also enjoy some advantages. Firstly,
for the MC-SVMs in \citet{weston1998multi,lee2004multicategory}, the GC-based error bounds decay as $O(n^{-\frac{1}{2}}c)$, while the CN-based bounds
decay as $O(n^{-\frac{1}{2}}c\log^{3\over2}(nc))$. Secondly, the GC-based error bounds
involve a summation of $K(\bx_i,\bx_i)$ over training examples, while the CN-based error bounds involve a maximum of $\|\phi(\bx_i)\|_i$ over training examples.
In this sense, GC-based error bounds capture better the properties of the distribution from which the training examples are drawn.

Observing the mismatch between these two types of generalization error bounds, it is interesting to perform an in-depth discussion to explain this phenomenon.
Our GC-based bounds are based on a structural result (Lemma \ref{lem:GP-structural-lipschitz}) of empirical GCs to exploit the Lipschitz continuity
of loss functions w.r.t. a variant of the $\ell_2$-norm, while our CN-based analysis is based on a structural
result of empirical $\ell_\infty$-norm CNs to directly use the Lipschitz continuity of loss functions w.r.t. the $\ell_\infty$-norm. Which
approach is better depends on the Lipschitz continuity of the associated loss functions. Specifically, if $\Psi_y$ is Lipschitz continuous w.r.t. a variant of the $\ell_2$-norm
involving the Lipschitz constant pair $(L_1,L_2)$, and $L$-Lipschitz continuous w.r.t. the $\ell_\infty$-norm, then one can show the following inequality with probability
at least $1-\delta$ for $\delta\in(0,1)$ (Theorem \ref{thm:risk-data-dependent} and Theorem \ref{thm:risk-data-independent}, respectively)
\begin{subnumcases}{A_\tau\leq}
   2\sqrt{\pi}\Big[L_1c\mathfrak{G}_{\widetilde{S}}(\widetilde{H}_\tau)+L_2\mathfrak{G}_{\widetilde{S}'}(\widetilde{H}_\tau)\Big], & (by GCs), \label{comparison-gaussian}
   \\
   27L\sqrt{c}\mathfrak{R}_{nc}(\widetilde{H}_\tau)\Big(1+\log_2^{3\over2}\frac{\hat{B}n\sqrt{c}}{\mathfrak{R}_{nc}(\widetilde{H}_\tau)}\Big), & (by CNs). \label{comparison-covering}
\end{subnumcases}

It is reasonable to assume that $\mathfrak{G}_{\widetilde{S}}(\widetilde{H}_\tau)$ and $\mathfrak{R}_{nc}(\widetilde{H}_\tau)$ decay at a same order. For example,
if $\tau(\bw)=\|\bw\|_{2,p},p\geq2$, then one can show (the first inequality follows from \eqref{risk-data-dependent-3}, \eqref{risk-data-dependent-4} and
\eqref{risk-dependent-lp-1}, and the second inequality follows from Proposition \ref{prop:rademacher-independent-lp})
$$
  \mathfrak{G}_{\widetilde{S}}(\widetilde{H}_\tau)=O\Big(n^{-1}c^{-\frac{1}{p}}\big(\sum_{i=1}^{n}K(\bx_i,\bx_i)\big)^{1\over2}\Big)\quad
  \text{and}\quad\mathfrak{R}_{nc}(\widetilde{H}_\tau)=O\Big(n^{-\frac{1}{2}}c^{-\frac{1}{p}}\max_{i\in\nbb_n}\|\phi(\bx_i)\|_2\Big).
$$
We further assume the dominant term in \eqref{comparison-gaussian} is $L_1c\mathfrak{G}_{\widetilde{S}}(\widetilde{H}_\tau)$ to see clearly the
relative behavior of these two types of error bounds. If $L_1$ and $L$ are of the same order, as exemplified by Example \ref{exp:loss-margin} and Example \ref{exp:soft-max},
then the error bounds based on CNs  outperform those based on  GCs by a factor of $\sqrt{c}$ (up to logarithmic factors).
If $L_1=O(c^{-\frac{1}{2}}L)$, as exemplified by Example \ref{exp:ww}, Example \ref{exp:llw} and Example \ref{exp:jenssen}, then the error bounds based on GCs
 outperform those based on CNs by a factor of $\log^{3\over2}(nc)$.
The underlying reason is that the Lipschitz continuity w.r.t. $\|\cdot\|_2$ is a weaker assumption than that w.r.t. $\|\cdot\|_\infty$ in the magnitude of Lipschitz
constants. Indeed, if $\Psi_y$ is $L_1$-Lipschitz continuous w.r.t. $\|\cdot\|_2$, then one may expect that $\Psi_y$ is $(L_1\sqrt{c})$-Lipschitz continuous w.r.t.
$\|\cdot\|_\infty$ due to the inequality $\|\mathbf{t}\|_2\leq\sqrt{c}\|\mathbf{t}\|_\infty$ for any $\mathbf{t}\in\rbb^c$. This explains why \eqref{comparison-covering}
outperforms \eqref{comparison-gaussian} by a factor of $\sqrt{c}$ if we ignore the Lipschitz constants. To summarize, if $L_1=O(c^{-\frac{1}{2}}L)$, then \eqref{comparison-gaussian}
outperforms \eqref{comparison-covering}. Otherwise, \eqref{comparison-covering} would be better. Therefore, one should choose an appropriate approach according to the associated loss function
to exploit the inherent Lipschitz continuity.

We also include the error bounds based on the structural result \eqref{structural-rademacher} in the second column to demonstrate
the advantages of the structural result based on our variant of the $\ell_2$-norm over \eqref{structural-rademacher}.

\section{ Experiments\label{sec:experiments}}

In this section, we report some experimental results to show the effectiveness of our theory.
We choose the multinomial logistic loss $\Psi_y(\bt)=\Psi_y^m(\bt)$ defined in Example \ref{exp:soft-max}
and the hypothesis space $H_\tau$ with $\tau(\bw)=\|\bw\|_{2,p},p\geq1$
and $\phi(\bx)=\bx$. In subsection \ref{sec:exp-rc}, we aim to show that our error bounds capture well
the effects of the parameter $p$ and the number of classes on the generalization performance. In subsection, \ref{sec:exp-ms}, we aim to show
that our error analysis is able to imply a structural risk working well in the model selection.
We use several benchmark datasets in our experiments: the MNIST collected by \cite{lecun1998gradient},
the NEWS20 collected by \cite{lang1995newsweeder}, the LETTER collected by \cite{hsu2002comparison},
the RCV1 collected by \cite{lewis2004rcv1}, the SECTOR collected by \cite{mccallum1998comparison} and the ALOI collected by \cite{geusebroek2005amsterdam}.
For ALOI, we include the first $67\%$ instances in each class in the training dataset and use the remaining instances as the test dataset.
Table \ref{tab:data_set} gives some information of these datasets.
All these datasets can be downloaded from the LIBSVM website \citep{chang2011libsvm}.
\begin{table}[htbp]
\caption{Description of datasets used in the experiments.
		\label{tab:data_set}}
\small
\setlength{\tabcolsep}{3pt}
\centering
  \begin{tabular}{*{5}{|c}|}\hline
  Dataset & $c$ & $n$ & \# Test Examples & $d$ \\\hline
  MNIST & $10$ & $60,000$ & $10,000$ & $778$\\\hline
  NEWS20 & $20$ & $15,935$ & $3,993$ & $62,060$ \\\hline
  LETTER & $26$  & $10,500 $ & $5,000$ & $16$ \\\hline
  RCV1 & $53$ & $15,564$ & $518,571$ & $47,236$ \\\hline
  SECTOR & $105$ & $6,412$ & $3,207$ & $55,197$ \\\hline
  ALOI & $1,000$ & $72,000$ & $36,000$ & $128$ \\\hline
  \end{tabular}    		
\end{table}

\subsection{ Behavior of empirical Rademacher complexity\label{sec:exp-rc}}


Our generalization analysis depends crucially on estimating the empirical RC $\mathfrak{R}_S(F_{\tau,\Lambda})$,
which, as shown in the proof of Corollary \ref{cor:lp-mcsvm} (b), is bounded by
$O(\Lambda n^{-\frac{1}{2}}c^{\frac{1}{2}-\frac{1}{\max(2,p)}}\max\limits_{i\in\nbb_n}\|\bx_i\|_2)$.
Our purpose here is to investigate whether it  really captures the behavior of $\mathfrak{R}_S(F_{\tau,\Lambda})$ in practice.

Let $\bm{\epsilon}=\{\epsilon_i\}_{i\in\nbb_n}$ be independent Rademacher variables and define
\begin{equation}\label{AERC}
  \widetilde{\mathfrak{R}}_S(\bm{\epsilon},F_{\tau,\Lambda}):=\frac{1}{n}\sup_{\bw\in\rbb^{d\times c}:\|\bw\|_{2,p}\leq\Lambda}\sum_{i=1}^{n}\epsilon_i\Psi_{y_i}^m\big(\langle\bw_1,\bx_i\rangle,\ldots,\langle\bw_c,\bx_i\rangle\big).
\end{equation}
It can be checked that $\widetilde{\mathfrak{R}}_S(\bm{\epsilon},F_{\tau,\Lambda})$ (as a function of $\bm{\epsilon}$) satisfies the increment
condition \eqref{bounded-variation-assumption} in McDiarmid's inequality below and concentrates sharply around its expectation $\mathfrak{R}_S(F_{\tau,\Lambda})$.
We approximate $\mathfrak{R}_S(F_{\tau,\Lambda})$ by an Approximation of Empirical Rademacher
Complexity (AERC) defined by
$\text{AERC}(F_{\tau,\Lambda}):=\frac{1}{20}\sum_{t=1}^{20}\widetilde{\mathfrak{R}}_S(\bm{\epsilon}^{(t)},F_{\tau,\Lambda})$,
where $\bm{\epsilon}^{(t)}=\{\epsilon_i^{(t)}\}_{i\in\nbb_n},t=1,\ldots,20$, are independent sequences of independent Rademacher random variables.
We consider two approaches to identify the effect of $p$ and $c$ on the generalization performance, respectively.

\begin{algorithm2e}[htbp]\label{alg:frank-wolfe}
\SetKwInOut{Input}{input}
  \caption{Frank-Wolfe Algorithm}
    \BlankLine
    Let $k=0$ and $\bw^{(0)}=0\in\rbb^{d\times c}$\\
    \While{Optimality conditions are not satisfied}{
    Compute $\tilde{\bw}=\arg\min_{\bw:\|\bw\|_{2,p}\leq \Lambda}\big\langle\bw,\nabla f(\bw^{(k)})\rangle$ \\
    Calculate the direction $\bv=\tilde{\bw}-\bw^{(k)}$ and step size $\gamma\in[0,1]$\\
    Update $\bw^{(k+1)}=\bw^{(k)}+\gamma\bv$\\
    Set $k=k+1$
    }
\end{algorithm2e}

The calculation of AERC involves the constrained non-convex optimization problem \eqref{AERC}, which we solve by the classic Frank-Wolfe algorithm~\citep{frank1956algorithm,jaggi2013revisiting}.
We describe the Frank-Wolfe algorithm to solve $\min_{\bw\in\triangle_p} f(\bw)$ for a general function $f$ defined
on the feasible set $\triangle_p=\{\bw\in\rbb^{d\times c}:\|\bw\|_{2,p}\leq\Lambda\}$ with $p\geq1$ and $\Lambda>0$ in
Algorithm \ref{alg:frank-wolfe}. This is a projection-free method but involves a constrained linear optimization problem at each iteration,
which, as shown in the following proposition, has a closed-form solution.
In line 4 of Algorithm \ref{alg:frank-wolfe}, we use a backtracking line search to search the step size $\gamma$
satisfying the \emph{Armijo condition} (e.g., page 33 in \citep{nocedal2006numerical}).
Proposition \ref{prop:FW} can be proved by checking $\|\bw^*\|_{2,p}\leq 1$ and $\langle\bw^*,\bv\rangle=-\|\bv\|_{2,p^*}$,
and is deferred to Appendix \ref{sec:proof-fw}.
\begin{proposition}\label{prop:FW}
  Let $\bv=(\bv_1,\ldots,\bv_c)\in\rbb^{d\times c}$ have nonzero column vectors and $p\geq1$. Then the optimization problem
  \begin{equation}\label{linear-optimization}
    \arg\min_{\bw\in\rbb^{d\times c}} \langle \bw, \bv\rangle\quad \text{s.t. }\|\bw\|_{2,p}\leq 1
  \end{equation}
  has a closed-form solution $\bw^*=(\bw_1^*,\ldots,\bw_c^*)$ as follows
  \begin{equation}\label{close-form-linear}
    \bw_j^*=\begin{cases}
            -\bv_{\bar{j}}\|\bv_{\bar{j}}\|_2^{-1}, & \mbox{if } p=1 \mbox{ and } j=\bar{j}, \\
            0, & \mbox{if } p=1 \mbox{ and } j\neq\bar{j}, \\
            -\big(\sum_{\tilde{j}=1}^{c}\|\bv_{\tilde{j}}\|_2^{p^*}\big)^{-\frac{1}{p}}\|\bv_j\|_2^{p^*-2}\bv_j, & \mbox{if } 1<p<\infty, \\
            -\|\bv_j\|_2^{-1}\bv_j, & \mbox{if }p=\infty,
          \end{cases}
  \end{equation}
  where $\bar{j}$ is the smallest index satisfying $\|\bv_{\bar{j}}\|_2=\max_{\tilde{j}\in\nbb_c}\|\bv_{\tilde{j}}\|_2$ and $p^*=p/(p-1)$.
\end{proposition}


In our first approach, we fix the dataset $S$, the parameter $\Lambda=100$ and vary the parameter $p$ over the set
$\{1.33, 1.67, 2, 2.33, 2.67, 3, 4, 5, 8, 10, 25, 50\}$. We plot the AERAs as a function of $p$
for RCV1, SECTOR and ALOI in Figure \ref{fig:aera-p} (a), (b) and (c), respectively.
We also include a plot of the function $f_{\bar{\tau}}(p)=\bar{\tau}c^{\frac{1}{2}-\frac{1}{\max(2,p)}}$
in each panel of Figure \ref{fig:aera-p}, where the corresponding parameter $\bar{\tau}$ is computed by
fitting the AERAs with linear models $\{p\to f_\tau(p):\tau\in\rbb_+\}$.
According to Figure \ref{fig:aera-p}, it is clear that AERAs and $f_{\bar{\tau}}$ match very well, implying that Corollary \ref{cor:lp-mcsvm} (b)
captures well the effect of the parameter $p$ on the error bounds.
\begin{figure}[htbp]
  \centering
  \hspace*{-0.5cm}\subfigure[RCV1.]{\includegraphics[width=0.37\textwidth,trim=7 2 6 3, clip]{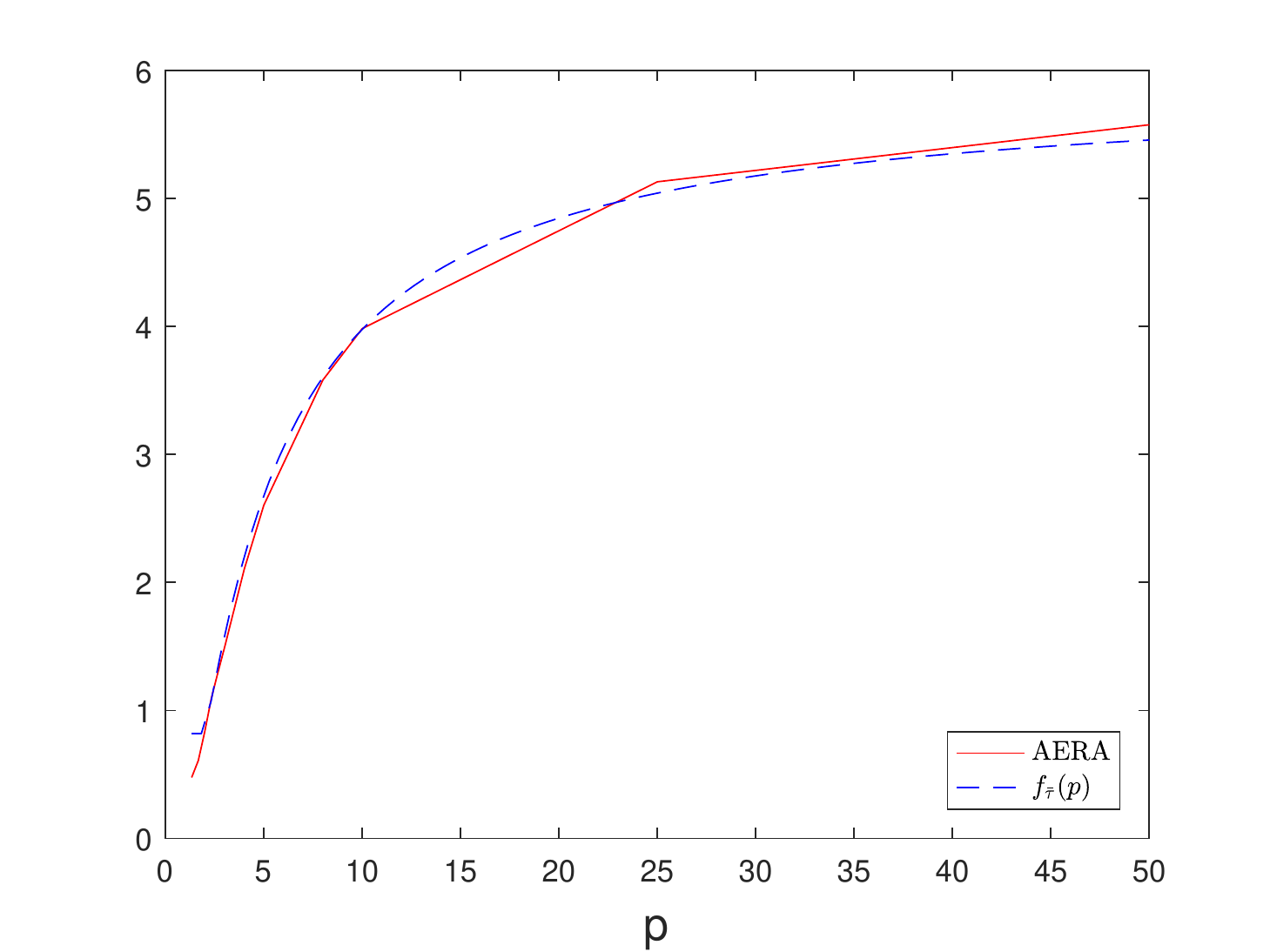}}\hspace*{-0.46cm}
  \subfigure[SECTOR.]{\includegraphics[width=0.37\textwidth,trim=7 0 6 3, clip]{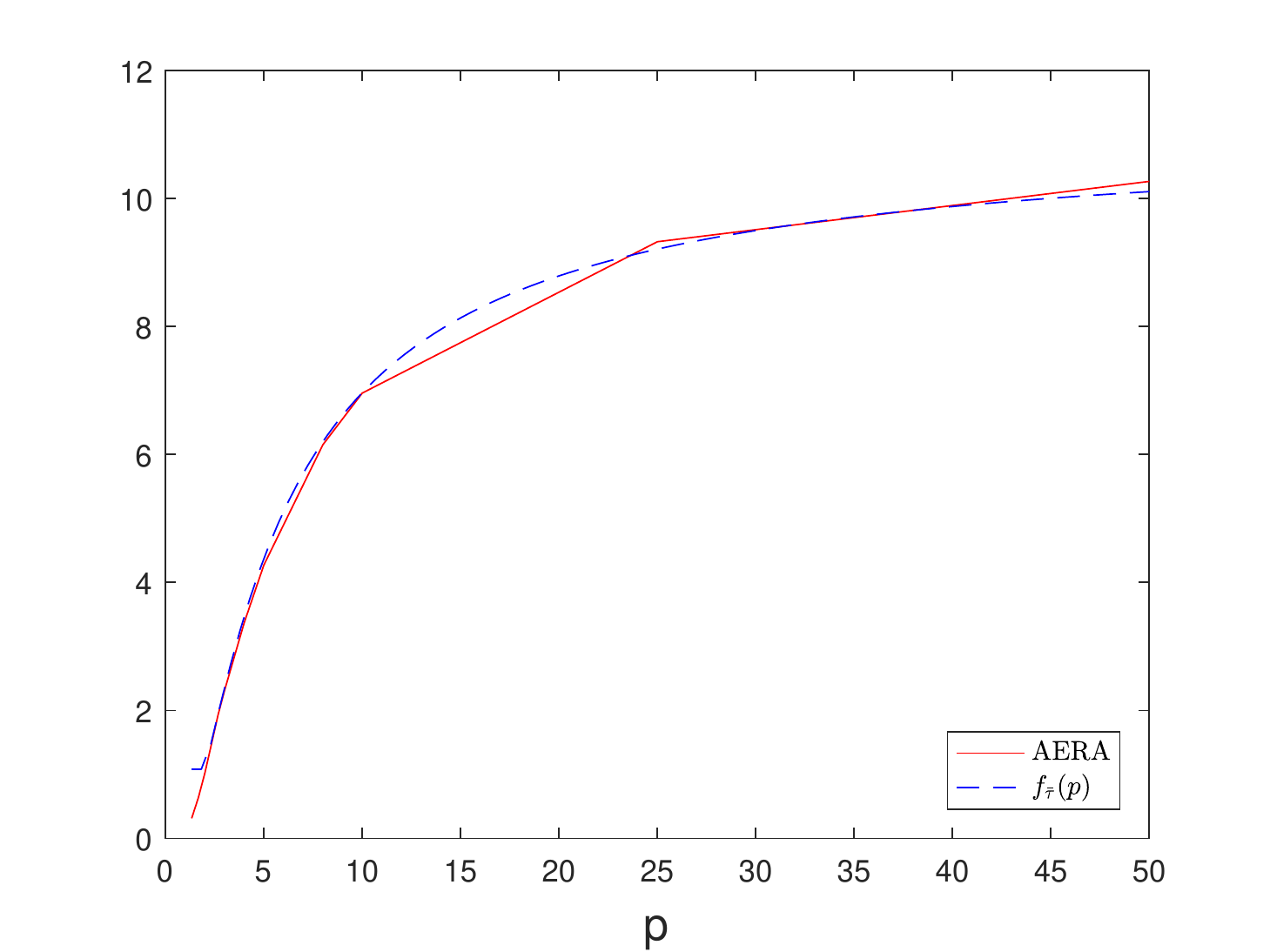}}\hspace*{-0.46cm}  
  \subfigure[ALOI.]{\includegraphics[width=0.37\textwidth,trim=7 0 6 3, clip]{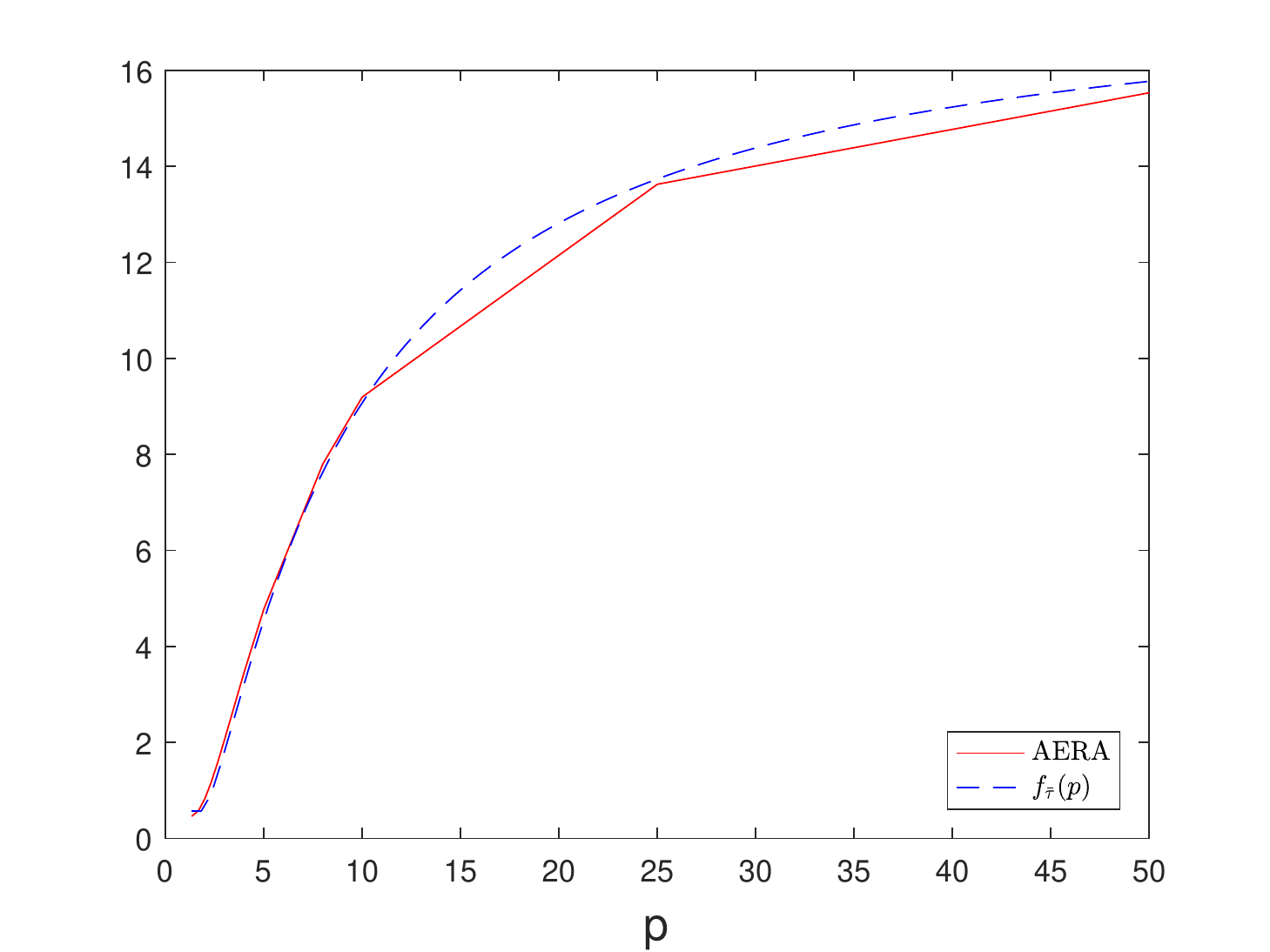}}
  \caption{AERAs as a function of $p$. We vary $p$ over $\{1.33, 1.67, 2, 2.33, 2.67, 3, 4, 5, 8, 10, 25, 50\}$ and calculate
  the AERA for each $p$. The panels (a), (b) and (c) show the results for RCV1, SECTOR and ALOI, respectively. We also include a plot of $f_{\bar{\tau}}(p)$
  in this figure, where $\bar{\tau}$ is calculated by applying the least squares method to fit these AERAs with $f_\tau(p)$.
  \label{fig:aera-p}}
\end{figure}
\begin{figure}[htbp]
  \centering
  \hspace*{-0.5cm}\subfigure[$p=2$.]{\includegraphics[width=0.37\textwidth,trim=5 2 4 3, clip]{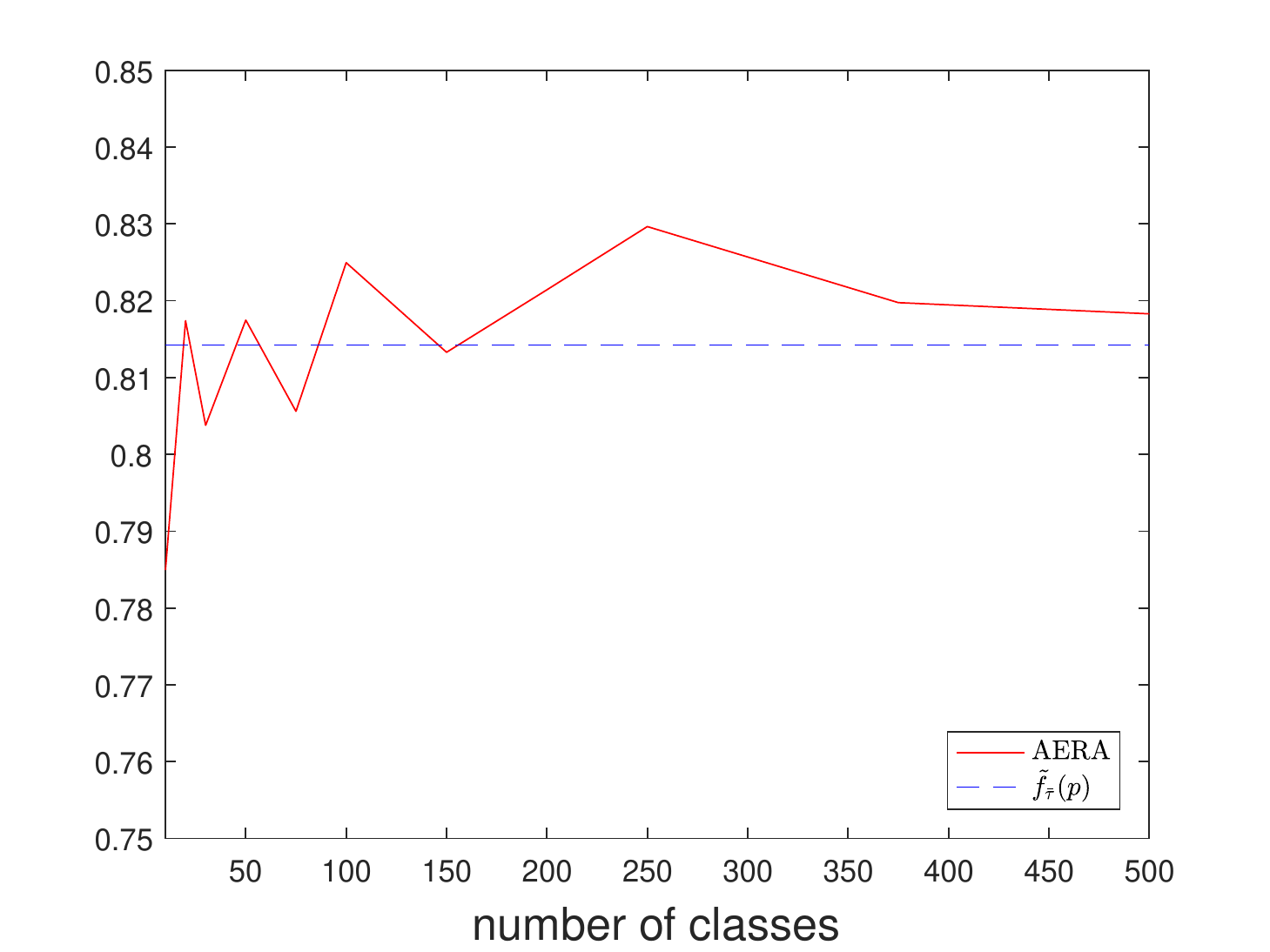}}\hspace*{-0.446cm}
  \subfigure[$p=4$.]{\includegraphics[width=0.37\textwidth,trim=7 2 4 3, clip]{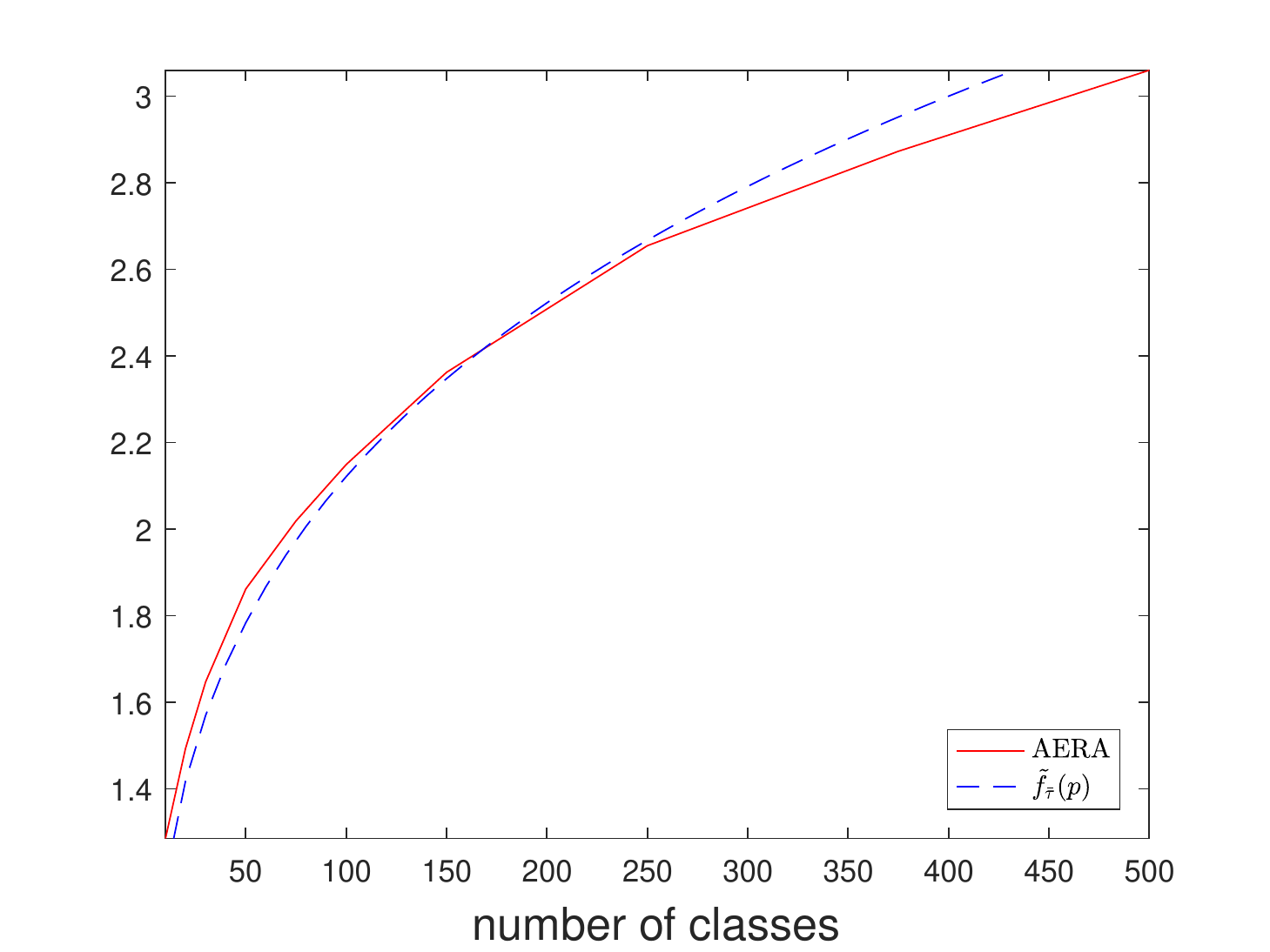}}\hspace*{-0.446cm}
  \subfigure[$p=\infty$.]{\includegraphics[width=0.37\textwidth,trim=7 2 6 3, clip]{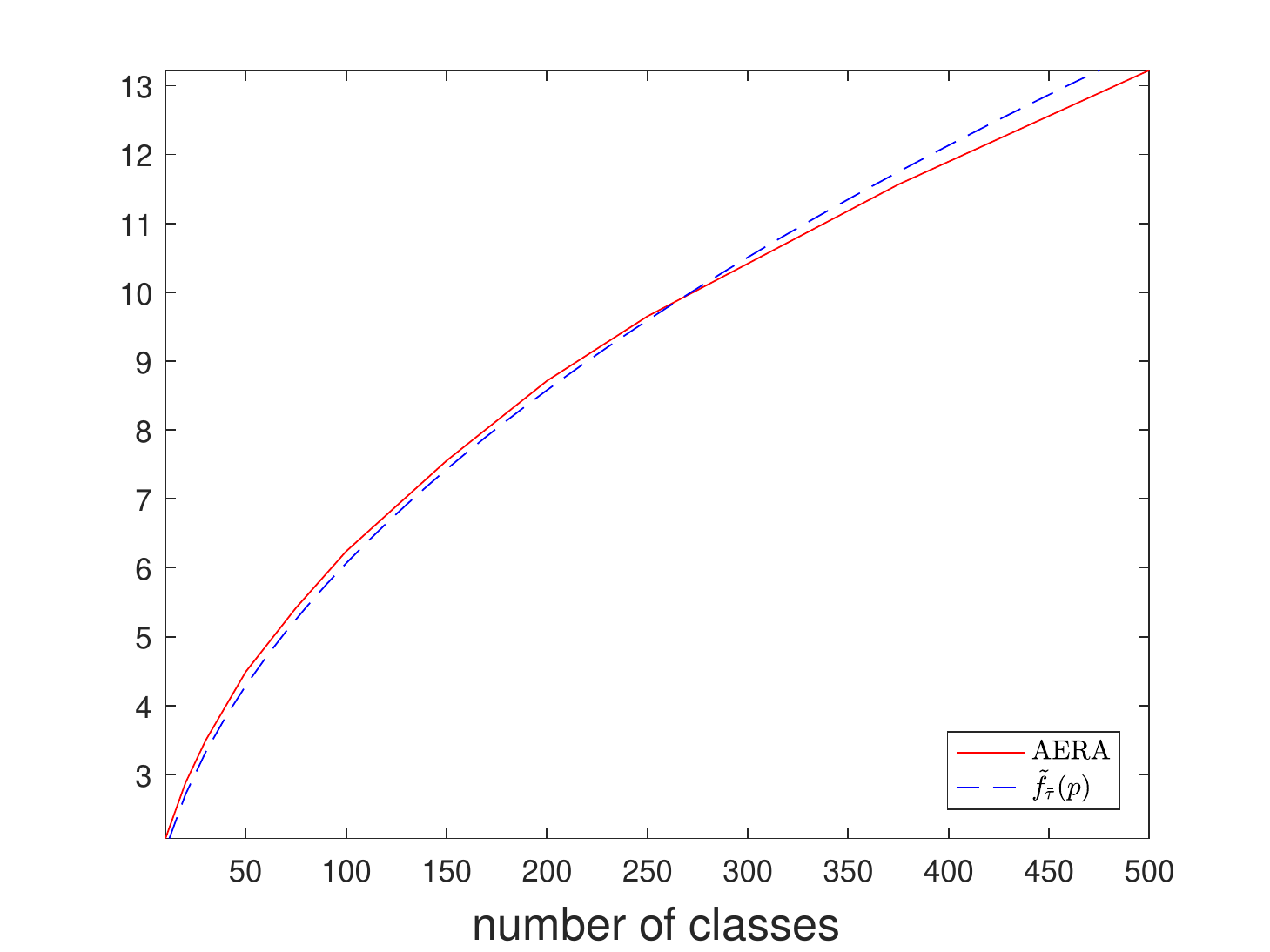}}
  \caption{AERAs as a function of the number of classes. Based on the ALOI, we construct datasets with varying
  number of classes $\tilde{c}$, for each of which we compute the associated AERA. The panels (a), (b) and (c)
  correspond to $p=2,p=4$ and $p=\infty$, respectively. We also include a plot of $\tilde{f}_{\bar{\tau}}(p)$
  in this figure, where $\bar{\tau}$ is calculated by applying the least squares method to fit these AERAs with $\tilde{f}_\tau(p)$.
  \label{fig:aera-c}}
\end{figure}

In our second approach, we fix the input training data $\{\bx_i\}_{i=1}^n$, the parameter $p$ and $\Lambda=100$,
and vary the number of classes $\tilde{c}$ over the set $\{10, 20, 30, 50, 75, 100, 150, 200, 250, 375, 500\}$.
For each considered $\tilde{c}$, we create a data set $S^{(\tilde{c})}=\{(\bx_i,y_i^{(\tilde{c})})\}_{i=1}^n$
with $y_i^{(\tilde{c})}= y_i \mod c+1$. We perform experiments on the ALOI.
We plot the AERAs as a function of $\tilde{c}$ in Figure \ref{fig:aera-c} for $p=2,4,\infty$, respectively.
In each of these panels, we include a plot of the function $\tilde{f}_{\bar{\tau}}(\tilde{c})=\bar{\tau}\tilde{c}^{\frac{1}{2}-\frac{1}{\max(2,p)}}$,
where the corresponding parameter $\bar{\tau}$ is computed by fitting the AERAs with linear models $\{\tilde{c}\to \tilde{f}_\tau(\tilde{c}):\tau\in\rbb_+\}$.
We see from Figure \ref{fig:aera-c} (b), \ref{fig:aera-c} (c) that AERAs and $\tilde{f}_{\bar{\tau}}$ match well,
implying that Corollary \ref{cor:lp-mcsvm} captures well the effect of the number of classes on the error bounds.
For the specific case $p=2$, Figure \ref{fig:aera-c} (a) shows that AERAs fluctuate within the small region $(0.785,0.83)$,
justifying the class-size independency of AERAs in this case.


\subsection{ Behavior of $\ell_p$-norm MC-SVMs and model selection\label{sec:exp-ms}}

Motivated by our GC-based bound given in Corollary \ref{cor:lp-mcsvm} (a), we propose an $\ell_p$-norm MC-SVM
in \citep{lei2015multi} for $p\in(1,2]$. We solve the involved optimization problem by introducing class
weights and alternating the update w.r.t. class weights and the update w.r.t. the model $\bw$. In this
paper, we propose to solve this optimization problem by the Frank-Wolfe algorithm (Algorithm \ref{alg:frank-wolfe}),
which avoids the introduction of additional class weights and extends the algorithm in \citep{lei2015multi}
to the case $p\geq2$.
The closed-form solution established in Proposition \ref{prop:FW} makes the implementation of this algorithm
very simple and efficient.
We traverse $p$ over the set $\{1.33, 1.67, 2, 2.5, 3, 4, 8, \infty\}$ and $\Lambda$
over the set $\{10^{0.5}, 10, 10^{1.5},\ldots,10^{3.5}\}$. For each pair $(p,\Lambda)$, we compute the following $\bw_{p,\Lambda}$
by Algorithm \ref{alg:frank-wolfe}
$$
  \bw_{p,\Lambda}:=\arg\min_{\bw\in\rbb^{d\times c}: \|\bw\|_{2,p}\leq\Lambda}
  \frac{1}{n}\sum_{i=1}^{n}\Psi_{y_i}^m\big(\langle\bw_1,\bx_i\rangle,\ldots,\langle\bw_c,\bx_i\rangle\big)
$$
and compute the accuracy (the percent of instances being labeled correctly)
on the test examples.

We now describe how to apply our error bounds in identifying an appropriate model from the candidate models constructed
above. Since $\bw_{p,\Lambda}\in H_{\tilde{p},\|w_{p,\Lambda}\|_{2,\tilde{p}}}$ for any $\tilde{p}\geq1$,
one can derive from Proposition \ref{cor:lp-mcsvm} the following inequality with probability $1-\delta$
(here we omit the randomness of $\|\bw_{p,\Lambda}\|_{2,\tilde{p}}$ for brevity)
\begin{multline*}
  \ebb_{\bx,y}\Psi_y(h^{\bw_{p,\Lambda}}(\bx))-3B_\Psi\Big[\frac{\log\frac{4}{\delta}}{2n}\Big]^{\frac{1}{2}}\leq
  \frac{1}{n}\sum_{i=1}^n\Psi_{y_i}(h^{\bw_{p,\Lambda}}(\bx_i))+\\
  \frac{54\|\bw_{p,\Lambda}\|_{2,\tilde{p}}\max\limits_{i\in\nbb_n}\|\bx_i\|_2c^{\frac{1}{2}-\frac{1}{\max(2,\tilde{p})}}\big(1+\log_2^{3\over2}\big(\sqrt{2}n^{3\over2}c\big)\big)}{\sqrt{n}}.
\end{multline*}
According to the inequality $\|\bw\|_{2,2}\leq\|\bw\|_{2,\tilde{p}}c^{\frac{1}{2}-\frac{1}{\tilde{p}}}$ for any $\tilde{p}\geq2$,
the term $\|\bw\|_{2,\tilde{p}}c^{\frac{1}{2}-\frac{1}{\max(2,\tilde{p})}}$ attains its minimum at $\tilde{p}=2$.
Hence, we construct the following structural risk (ignoring logarithmic factors here)
\begin{equation}\label{struct-risk}
  \text{Err}_{\text{str},\lambda}(\bw):=\frac{1}{n}\sum_{i=1}^n\Psi_{y_i}(h^{\bw}(\bx_i))+\frac{\lambda\|\bw\|_{2,2}\max_{i\in\nbb_n}\|\bx_i\|_2}{\sqrt{n}}
\end{equation}
and use it to select a model with the minimal structural risk among all candidates $\bw_{p,\Lambda}$.
We use $\lambda=0.5$ in this paper.

In Table \ref{tab:model-selection}, we report the best accuracy achieved by $\ell_2$-norm MC-SVM over all considered $\Lambda$
(the column termed $p=2$),
the best accuracy achieved by $\ell_p$-norm MC-SVMs over all considered $p$ and $\Lambda$ (the column termed Oracle), and the accuracy for
the model with the minimal structural risk (the column termed model selection). The parameter $p$ for both the model with the best accuracy (Oracle)
and the model selected by model selection are also indicated.
For comparison, we also include in the column ``Weston \& Watkins'' the best accuracy achieved by the MC-SVM in Corollary \ref{cor:ww} with $\ell(t)=\log(1+\exp(-t))$.
Experimental results show that the structural risk \eqref{struct-risk} works well in guiding
the selection of a model with comparable prediction accuracy to the best candidate model. Furthermore, $\ell_p$-norm MC-SVMs
can significantly outperform the specific $\ell_2$-norm MC-SVM on some datasets.
For example, the $\ell_\infty$-norm MC-SVM attains $3.5\%$ accuracy gain over the $\ell_2$-norm MC-SVM on ALOI.

Although Corollary \ref{cor:lp-mcsvm} (b) on estimation error bounds suggests that one should use $p\leq2$ when training
$\ell_p$-norm MC-SVMs, Table \ref{tab:model-selection} shows that $\ell_p$-norm MC-SVMs achieve the best accuracy at
different $p$ for different problems. The underlying reason is that the approximation power of $H_p$ increases with $p$
due to the relationship $H_p\subset H_{\tilde{p}}$ if $p\leq\tilde{p}$. A good model should balance
the approximation and estimation errors by choosing $p$ suitable to the associated problem.

Although the error bounds for models in Corollaries \ref{cor:Cramer}-\ref{cor:jenssen} may enjoy different dependencies
on the number of classes, this not necessarily implies that a model would outperform the other in terms of prediction accuracies for all problems.
The underlying reason is that these error bounds are stated for different loss functions and are therefore not comparable.
Let us consider the multinomial logistic regression in Corollary \ref{cor:logistic} and the Weston \& Watkins MC-SVM in Corollary \ref{cor:ww} for example.
Although the dependencies on the number of classes greatly differ in Corollary \ref{cor:logistic} (b) (logarithmic) and Corollary \ref{cor:ww} (a) (linear),
the column ``$p=2$'' and the column ``Weston \& Watkins'' in Table \ref{tab:model-selection} show that their difference in the prediction
accuracy is not that large.

\begin{table}[hbtp]

\caption{Performance of $\ell_p$-norm MC-SVMs on several benchmark datasets.
  \normalfont We train $\ell_p$-norm MC-SVMs by traversing $p$ over $\{1.33, 1.67, 2, 2.5, 3, 4, 8, \infty\}$ and $\Lambda$
  over $\{10^{0.5}, 10,\ldots,10^{3.5}\}$ to get candidate models.
  The column $p=2$ shows the accuracy achieved by the $\ell_2$-norm MC-SVM.
  The column ``Oracle'' corresponds to the model with the best accuracy
  over all candidate models, where we also indicate the associated $p$ and accuracy.
  The column ``Model Selection'' corresponds to the model with the smallest
  structural risk over all candidate models, where we also indicate the
  associated $p$ and accuracy. For comparison, we also include in the column ``Weston \& Watkins'' the best
  accuracy achieved by the MC-SVM in Corollary \ref{cor:ww} with $\ell(t)=\log(1+\exp(-t))$,
  where we traverse $\Lambda$ over $\{10^{0.5}, 10,\ldots,10^{3.5}\}$.
  \label{tab:model-selection}}
  \setlength{\tabcolsep}{6pt}
  \centering\def\arraystretch{1.2}
  \begin{tabular}{|c|c|c|c|c|c|c|}
    \hline
     \multirow{2}{*}{Dataset}& \multirow{2}{*}{$p=2$} & \multicolumn{2}{c|}{Oracle} & \multicolumn{2}{c|}{Model Selection} & \multirow{2}{*}{Weston \& Watkins}\\\cline{3-6}
     & & $p$ & Accuracy & $p$ & Accuracy&\\ \hline
     MNIST & $91.05$ & $2.5$ & $91.44$ & $\infty$ & $91.40$ & $91.00$\\ \hline
     NEWS20 & $84.07$ & $4$ & $84.45$ & $4$ & $84.27$ & $84.10$\\ \hline
     LETTER & $72.22$ & $\infty$ & $74.36$ & $8$ & $74.04$ & $69.28$\\ \hline
     RCV1 & $88.67$ & $1.67$ & $88.74$ & $3$ & $88.08$ & $88.67$\\ \hline
     SECTOR & $93.08$ & $3$ & $93.26$ & $3$ & $92.14$ & $92.83$\\ \hline
     ALOI & $84.12$ & $\infty$ & $87.70$ & $4$ & $87.18$ & 78.56\\ \hline
  \end{tabular}
\end{table}

\section{Proofs}

In this section, we present the proofs of the results presented in the previous sections.

\subsection{Proof of Bounds by Gaussian Complexities\label{sec:proof-dependent}}
In this subsection, we present the proofs for data-dependent bounds in subsection \ref{sec:data-dependent-bound}.
The proof of Lemma \ref{lem:GP-structural-lipschitz} requires to use a comparison result (Lemma \ref{lem:gaussian-comparison}) on Gaussian processes attributed to \citet{slepian1962one}, while the proof of Theorem \ref{thm:risk-data-dependent} is based on a concentration inequality in \citep{mcdiarmid1989method}.
\begin{lemma}\label{lem:gaussian-comparison}
  Let $\{\mathfrak{X}_\theta:\theta\in \Theta\}$ and $\{\mathfrak{Y}_\theta:\theta\in\Theta\}$ be two mean-zero separable Gaussian processes indexed by the same set $\Theta$ and suppose that
  \begin{equation}\label{increment-condition}
    \ebb[(\mathfrak{X}_\theta-\mathfrak{X}_{\bar{\theta}})^2]\leq \ebb[(\mathfrak{Y}_\theta-\mathfrak{Y}_{\bar{\theta}})^2],\quad\forall \theta,\bar{\theta}\in\Theta.
  \end{equation}
  Then $\ebb[\sup_{\theta\in\Theta} \mathfrak{X}_\theta]\leq\ebb[\sup_{\theta\in\Theta} \mathfrak{Y}_\theta]$.
\end{lemma}
\begin{lemma}[McDiarmid inequality~\citep{mcdiarmid1989method}]\label{lem:mcdiarmid}
  Let $Z_1,\ldots,Z_n$ be independent random variables taking values in a set $\zcal$, and assume that $f:\zcal^n\to\rbb$ satisfies
  \begin{equation}\label{bounded-variation-assumption}
    \sup_{\bz_1,\ldots,\bz_n,\bar{\bz}_i\in\zcal}|f(\bz_1,\cdots,\bz_n)-
    f(\bz_1,\cdots,\bz_{i-1},\bar{\bz}_i,\bz_{i+1},\cdots,\bz_n)|\leq c_i
  \end{equation}
  for $1\leq i\leq n$. Then, for any $0<\delta<1$, with probability at least $1-\delta$, we have
  $$f(Z_1,\ldots,Z_n)\leq\ebb f(Z_1,\ldots,Z_n)+\sqrt{\frac{\sum_{i=1}^nc_i^2\log(1/\delta)}{2}}.$$
\end{lemma}
\begin{proof}[Proof of Lemma \ref{lem:GP-structural-lipschitz}]
Define two mean-zero separable Gaussian processes indexed by the finite dimensional Euclidean space $\{(h(\bx_1),\ldots,h(\bx_n)):h\in H\}$ 
\begin{gather*}
  \mathfrak{X}_h:=\sum_{i=1}^ng_if_i(h(\bx_i)),\quad
  \mathfrak{Y}_h:=\sqrt{2}L_1\sum_{i=1}^n\sum_{j=1}^cg_{ij}h_j(\bx_i)+\sqrt{2}L_2\sum_{i=1}^{n}g_ih_{r(i)}(\bx_i).
\end{gather*}
For any $h,h'\in H$, the independence among $g_i,g_{ij}$ and $\ebb g_i^2=1, \ebb g_{ij}^2=1,\forall i\in\nbb_n,j\in\nbb_c$ imply that
\begin{align*}
  \ebb[(\mathfrak{X}_h-\mathfrak{X}_{h'})^2]&=\ebb\Big[\Big(\sum_{i=1}^ng_i\big(f_i(h(\bx_i))-f_i(h'(\bx_i))\big)\Big)^2\Big]= \sum_{i=1}^n\big[f_i(h(\bx_i))-f_i(h'(\bx_i))\big]^2\\
  &\leq\sum_{i=1}^{n}\Big[L_1\big[\sum_{j=1}^{c}|h_j(\bx_i)-h'_j(\bx_i)|^2\big]^{\frac{1}{2}}+L_2|h_{r(i)}(\bx_i)-h'_{r(i)}(\bx_i)|\Big]^2\\
  &\leq 2L_1^2\sum_{i=1}^n\sum_{j=1}^c|h_j(\bx_i)-h'_j(\bx_i)|^2+2L_2^2\sum_{i=1}^{n}|h_{r(i)}(\bx_i)-h'_{r(i)}(\bx_i)|^2\\
  &=\ebb[(\mathfrak{Y}_h-\mathfrak{Y}_{h'})^2],
\end{align*}
where we have used the Lipschitz continuity of $f_i$ w.r.t. a variant of the $\ell_2$-norm in the first inequality, and the elementary inequality $(a+b)^2\leq 2(a^2+b^2)$ in the second inequality.
Therefore, the condition \eqref{increment-condition} holds and  Lemma \ref{lem:gaussian-comparison} can be applied here to give
\begin{align*}
  \ebb_{\bm{g}}\sup_{h\in H}\sum_{i=1}^{n}g_if_i(h(\bx_i))&
  \leq \ebb_{\bm{g}}\sup_{h\in H}\Big[\sqrt{2}L_1\sum_{i=1}^n\sum_{j=1}^cg_{ij}h_j(\bx_i)+\sqrt{2}L_2\sum_{i=1}^{n}g_ih_{r(i)}(\bx_i)\Big] \\
   & \leq \sqrt{2}L_1\ebb_{\bm{g}}\sup_{h\in H}\sum_{i=1}^n\sum_{j=1}^cg_{ij}h_j(\bx_i)+\sqrt{2}L_2\ebb_{\bm{g}}\sup_{h\in H}\sum_{i=1}^{n}g_ih_{r(i)}(\bx_i).
\end{align*}
The proof of Lemma \ref{lem:GP-structural-lipschitz} is complete.
\end{proof}

\begin{proof}[Proof of Theorem \ref{thm:risk-data-dependent}]
  It can be checked that
  $f(\bz_1,\ldots,\bz_n)=\sup_{h^{\bw}\in H_\tau}\Big[\ebb_{\bz}\Psi_y(h^{\bw}(\bx))-\frac{1}{n}\sum_{i=1}^n\Psi_{y_i}(h^{\bw}(\bx_i))\Big]$
  satisfies the increment condition \eqref{bounded-variation-assumption} with $c_i=B_\Psi/n$.
  An application of McDiarmid's inequality (Lemma \ref{lem:mcdiarmid}) then shows the following inequality with probability $1-\delta/2$
  \begin{multline*}
    \sup_{h^{\bw}\in H_\tau}\Big[\ebb_{\bz}\Psi_y(h^{\bw}(\bx))-\frac{1}{n}\sum_{i=1}^n\Psi_{y_i}(h^{\bw}(\bx_i))\Big]\leq \\
    \ebb_{\bz}\sup_{h^{\bw}\in H_\tau}\Big[\ebb_{\bz}\Psi_y(h^{\bw}(\bx))-\frac{1}{n}\sum_{i=1}^n\Psi_{y_i}(h^{\bw}(\bx_i))\Big]+B_\Psi\sqrt{\frac{\log\frac{2}{\delta}}{2n}}.
  \end{multline*}
  It follows from the standard symmetrization technique (see, e.g., proof of Theorem 3.1 in \citep{mohri2012foundations}) that
  $$
    \ebb_{\bz}\sup_{h^{\bw}\in H_\tau}\Big[\ebb_{\bx,y}\Psi_y(h^{\bw}(\bx))-\frac{1}{n}\sum_{i=1}^n\Psi_{y_i}(h^{\bw}(\bx_i))\Big]
    \leq 2\ebb_{\bz}\ebb_{\bm{\epsilon}}\sup_{h^{\bw}\in H_\tau}\Big[\frac{1}{n}\sum_{i=1}^n\epsilon_i\Psi_{y_i}(h^{\bw}(\bx_i))\Big].
  $$
  It can also be checked that the function $f(\bz_1,\ldots,\bz_n)=\ebb_{\bm{\epsilon}}\sup_{h^{\bw}\in H_\tau}\big[\frac{1}{n}\sum_{i=1}^n\epsilon_i\Psi_{y_i}(h^{\bw}(\bx_i))\big]$
  satisfies the increment condition \eqref{bounded-variation-assumption} with $c_i=B_\Psi/n$. Another application of McDiarmid's inequality shows
  the inequality $\ebb_{\bz}\mathfrak{R}_S(F_{\tau,\Lambda})\leq \mathfrak{R}_S(F_{\tau,\Lambda})+B_\Psi\sqrt{\frac{\log\frac{2}{\delta}}{2n}}$ with probability $1-\delta/2$, which
   together with
  the above two inequalities then imply the following inequality with probability at least $1-\delta$
  \begin{equation}\label{risk-data-dependent-1}
    \sup_{h^{\bw}\in H_\tau}\Big[\ebb_{\bz}\Psi_y(h^{\bw}(\bx))-\frac{1}{n}\sum_{i=1}^n\Psi_{y_i}(h^{\bw}(\bx_i))\Big]\leq
    2\mathfrak{R}_S(F_{\tau,\Lambda})+3B_\Psi\sqrt{\frac{\log\frac{2}{\delta}}{2n}}.
  \end{equation}
   Furthermore, according to the following relationship between Gaussian and Rademacher processes for any function class $\tilde{H}$~\citep{bartlett2002rademacher} ($|S|$ is the cardinality of $S$)
  $$
    \mathfrak{R}_S(\tilde{H})\leq\sqrt{\frac{\pi}{2}}\mathfrak{G}_S(\tilde{H})\leq 3\sqrt{\frac{\pi\log |S|}{2}}\mathfrak{R}_S(\tilde{H}),
  $$
  we derive
  \begin{align*}
   \mathfrak{R}_S\big(\big\{\Psi_y(h^{\bw}(\bx)):&h^{\bw}\in H_\tau\big\}\big)
   \leq \sqrt{\frac{\pi}{2}}\mathfrak{G}_S\big(\big\{\Psi_y(h^{\bw}(\bx)):h^{\bw}\in H_\tau\big\}\big)\\
   &=\sqrt{\frac{\pi}{2}}\frac{1}{n}\ebb_{\bm{g}}\sup_{h^{\bw}\in H_\tau}\sum_{i=1}^{n}g_i\Psi_{y_i}(h^{\bw}(\bx_i))\\
   &\leq \frac{L_1\sqrt{\pi}}{n}\ebb_{\bm{g}}\sup_{h^{\bw}\in H_\tau}\sum_{i=1}^n\sum_{j=1}^cg_{ij}h_j^{\bw}(\bx_i)+\frac{L_2\sqrt{\pi}}{n}\ebb_{\bm{g}}\sup_{h^{\bw}\in H_\tau}\sum_{i=1}^ng_ih_{y_i}^{\bw}(\bx_i),
  \end{align*}
  where the last step follows from Lemma \ref{lem:GP-structural-lipschitz} with $f_i=\Psi_{y_i}$ and $r(i)=y_i,\forall i\in\nbb_n$.
  Plugging the above RC bound into \eqref{risk-data-dependent-1} gives the following inequality with probability at least $1-\delta$
  \begin{align}
    A_\tau & \leq  \frac{2L_1\sqrt{\pi}}{n} \ebb_{\bm{g}}\sup_{h^{\bw}\in H_{\tau}}\sum_{i=1}^{n}\sum_{j=1}^{c}g_{ij}\inn{\bw_j,\phi(\bx_i)} +\frac{2L_2\sqrt{\pi}}{n}\ebb_{\bm{g}}\sup_{h^{\bw}\in H_{\tau}}\sum_{i=1}^{n}g_{i}\inn{\bw_{y_i},\phi(\bx_i)}.\label{risk-data-dependent-2}
  \end{align}
   It remains to estimate the two terms on the right-hand side of \eqref{risk-data-dependent-2}.
  By \eqref{S-tilde-identity-inner-product}, the definition of $\widetilde{H}_\tau,\widetilde{S}$ and $\widetilde{S}'$, we know
  \begin{equation}\label{risk-data-dependent-3}
    \ebb_{\bm{g}}\sup_{h^{\bw}\in H_{\tau}}\sum_{i=1}^{n}\sum_{j=1}^{c}g_{ij}\inn{\bw_j,\phi(\bx_i)}
    = \ebb_{\bm{g}}\sup_{\bw:\tau(\bw)\leq\Lambda}\sum_{i=1}^{n}\sum_{j=1}^{c}g_{ij} \langle \bw,\bp_j(\bx_i)\rangle= nc\mathfrak{G}_{\widetilde{S}}(\widetilde{H}_\tau)
  \end{equation}
  and
  $$
    \ebb_{\bm{g}}\sup_{h^{\bw}\in H_{\tau}}\sum_{i=1}^{n}g_{i}\inn{\bw_{y_i},\phi(\bx_i)}
    = \ebb_{\bm{g}}\sup_{\bw:\tau(\bw)\leq\Lambda}\sum_{i=1}^{n}g_{i}\langle \bw, \bp_{y_i}(\bx_i)\rangle=n\mathfrak{G}_{\widetilde{S}'}(\widetilde{H}_\tau).
  $$
  Plugging the above two identities back into \eqref{risk-data-dependent-2} gives \eqref{risk-data-dependent-gaussian}.

  We now show \eqref{risk-data-dependent}.
  According to the definition of dual norm, we derive
  \begin{align}
    \ebb_{\bm{g}}\sup_{h^{\bw}\in H_{\tau}}\sum_{i=1}^{n}\sum_{j=1}^{c}g_{ij}\inn{\bw_j,\phi(\bx_i)}
  & = \ebb_{\bm{g}}\sup_{h^{\bw}\in H_{\tau}}\sum_{j=1}^{c}\big\langle\bw_j,\sum_{i=1}^{n}g_{ij}\phi(\bx_i)\big\rangle
  = \ebb_{\bm{g}}\sup_{h^{\bw}\in H_{\tau}}\inn{\bw,\big(\sum_{i=1}^{n}g_{ij}\phi(\bx_i)\big)_{j=1}^c}\notag\\
  & \leq \ebb_{\bm{g}}\sup_{h^{\bw}\in H_{\tau}}\|\bw\|\big\|\big(\sum_{i=1}^{n}g_{ij}\phi(\bx_i)\big)_{j=1}^c\big\|_*
  = \Lambda \ebb_{\bm{g}}\Big\|\big(\sum_{i=1}^{n}g_{ij}\phi(\bx_i)\big)_{j=1}^c\Big\|_*.\label{risk-data-dependent-4}
  \end{align}
  Analogously, we also have
  \begin{align*}
  \ebb_{\bm{g}}\sup_{h^{\bw}\in H_{\tau}}\sum_{i=1}^{n}g_{i}\inn{\bw_{y_i},\phi(\bx_i)}
  & = \ebb_{\bm{g}}\sup_{h^{\bw}\in H_{\tau}}\sum_{j=1}^{c}\inn{\bw_j,\sum_{i\in I_j}g_i\phi(\bx_i)}\\
  & = \ebb_{\bm{g}}\sup_{h^{\bw}\in H_{\tau}}\inn{\bw,\big(\sum_{i\in I_j}g_i\phi(\bx_i)\big)_{j=1}^c}
  \leq \Lambda \ebb_{\bm{g}}\big\|\big(\sum_{i\in I_j}g_i\phi(\bx_i)\big)_{j=1}^c\big\|_*.
  \end{align*}
  Plugging the above two inequalities back into \eqref{risk-data-dependent-2} gives \eqref{risk-data-dependent}.
\end{proof}

\begin{proof}[Proof of Corollary \ref{cor:risk-dependent-lp}]
Let $q\geq p$ be any real number. It follows from Jensen's inequality and Khintchine-Kahane inequality \eqref{khitchine-kahane} that
\begin{align}
  \ebb_{\bm{g}}\Big\|\Big(\sum_{i=1}^{n}g_{ij}\phi(\bx_i)\Big)_{j=1}^c\Big\|_{2,q^*}  &= \ebb_{\bm{g}}\Big[\sum_{j=1}^{c}\Big\|\sum_{i=1}^{n}g_{ij}\phi(\bx_i)\Big\|_2^{q^*}\Big]^{\frac{1}{q^*}}
   \leq \Big[\sum_{j=1}^{c}\ebb_{\bm{g}}\Big\|\sum_{i=1}^{n}g_{ij}\phi(\bx_i)\Big\|_2^{q^*}\Big]^{\frac{1}{q^*}} \notag\\
   & \leq \Big[\sum_{j=1}^{c}\Big[q^*\sum_{i=1}^{n}\|\phi(\bx_i)\|_2^2\Big]^{\frac{q^*}{2}}\Big]^{\frac{1}{q^*}}
   = c^{\frac{1}{q^*}}\Big[q^*\sum_{i=1}^{n}K(\bx_i,\bx_i)\Big]^{\frac{1}{2}}.\label{risk-dependent-lp-1}
\end{align}

Applying again Jensen's inequality and  Khintchine-Kahane inequality \eqref{khitchine-kahane}, we get
\begin{equation}\label{risk-dependent-lp-1a}
  \ebb_{\bm{g}}\Big\|\Big(\sum_{i\in I_j}g_i\phi(\bx_i)\Big)_{j=1}^c\Big\|_{2,q^*}  \leq \Big[\ebb_{\bm{g}}\sum_{j=1}^{c}\Big\|\sum_{i\in I_j}g_i\phi(\bx_i)\Big\|_2^{q^*}\Big]^{\frac{1}{q^*}}
   \leq \sqrt{q^*}\Big[\sum_{j=1}^{c}\Big[\sum_{i\in I_j}\|\phi(\bx_i)\|_2^2\Big]^{\frac{q^*}{2}}\Big]^{\frac{1}{q^*}}.
\end{equation}
We now control the last term in the above inequality by distinguishing whether $q\geq 2$ or not. If $q\leq 2$, we have $2^{-1}q^*\geq1$ and it follows from the elementary inequality $a^s+b^s\leq (a+b)^s,\forall a,b\geq0,s\geq1$ that
\begin{equation}\label{risk-dependent-lp-2}
  \sum_{j=1}^{c}\Big[\sum_{i\in I_j}K(\bx_i,\bx_i)\Big]^{\frac{q^*}{2}}\leq
  \Big[\sum_{j=1}^{c}\sum_{i\in I_j}K(\bx_i,\bx_i)\Big]^{\frac{q^*}{2}}=\Big[\sum_{i=1}^{n}K(\bx_i,\bx_i)\Big]^{\frac{q^*}{2}}.
\end{equation}
Otherwise we have $2^{-1}q^*\leq1$ and Jensen's inequality implies
\begin{equation}\label{risk-dependent-lp-3}
  \sum_{j=1}^{c}\Big[\sum_{i\in I_j}K(\bx_i,\bx_i)\Big]^{\frac{q^*}{2}}\leq c\Big[\sum_{j=1}^{c}\frac{1}{c}\sum_{i\in I_j}K(\bx_i,\bx_i)\Big]^{\frac{q^*}{2}}=c^{1-\frac{q^*}{2}}\Big[\sum_{i=1}^{n}K(\bx_i,\bx_i)\Big]^{\frac{q^*}{2}}.
\end{equation}
Combining \eqref{risk-dependent-lp-1a}, \eqref{risk-dependent-lp-2} and \eqref{risk-dependent-lp-3} together implies
\begin{equation}\label{risk-dependent-lp-4}
  \ebb_{\bm{g}}\Big\|\Big(\sum_{i\in I_j}g_i\phi(\bx_i)\Big)_{j=1}^c\Big\|_{2,q^*} \leq \max(c^{\frac{1}{q^*}-\frac{1}{2}},1)\Big[q^*\sum_{i=1}^{n}K(\bx_i,\bx_i)\Big]^{\frac{1}{2}}.
\end{equation}
According to the monotonicity of $\|\cdot\|_{2,p}$ w.r.t. $p$, we have $H_{p,\Lambda}\subset H_{q,\Lambda}$ if $p\leq q$. Plugging the complexity bound established in Eqs. \eqref{risk-dependent-lp-1}, \eqref{risk-dependent-lp-4} into the generalization bound given in Theorem \ref{thm:risk-data-dependent}, we get the following inequality with probability at least $1-\delta$
$$
  A_\tau\leq \frac{2\Lambda \sqrt{\pi}}{n}\bigg[L_1c^{\frac{1}{q^*}}\Big[q^*\sum_{i=1}^{n}K(\bx_i,\bx_i)\Big]^{\frac{1}{2}}
  +L_2\max(c^{\frac{1}{q^*}-\frac{1}{2}},1)\Big[q^*\sum_{i=1}^{n}K(\bx_i,\bx_i)\Big]^{\frac{1}{2}}\bigg],\;\forall q\geq p.
$$
The proof is complete.
\end{proof}

\begin{remark}[Tightness of the Rademacher Complexity Bound]\label{rem:risk-dependent-schatten}
  Eq. \eqref{risk-dependent-lp-1} gives an upper bound on \linebreak
  $\ebb_{\bm{g}}\Big\|\Big(\sum_{i=1}^{n}g_{ij}\phi(\bx_i)\Big)_{j=1}^c\Big\|_{2,q^*}$. We now show that this bound is tight up to a constant factor.
  Indeed, according to the elementary inequality for $a_1,\ldots,a_c\geq0$
  $$
    \big(a_1+\cdots+a_c\big)^{\frac{1}{q^*}}\geq c^{\frac{1}{q^*}-1}\big(a_1^{\frac{1}{q^*}}+\cdots+a_c^{\frac{1}{q^*}}\big),
  $$
  we derive
  $$
    \Big\|\Big(\sum_{i=1}^{n}g_{ij}\phi(\bx_i)\Big)_{j=1}^c\Big\|_{2,q^*}=\Big[\sum_{j=1}^{c}\Big\|\sum_{i=1}^{n}g_{ij}\phi(x_i)\Big\|_2^{q^*}\Big]^{\frac{1}{q^*}}\geq c^{\frac{1}{q^*}-1}\sum_{j=1}^{c}\Big\|\sum_{i=1}^{n}g_{ij}\phi(x_i)\Big\|_2.
  $$
  Taking expectations on both sides, we get that
  $$
    \ebb_{\bm{g}}\Big\|\Big(\sum_{i=1}^{n}g_{ij}\phi(\bx_i)\Big)_{j=1}^c\Big\|_{2,q^*} \geq c^{\frac{1}{q^*}-1}\sum_{j=1}^{c}\ebb_{\bm{g}}\Big\|\sum_{i=1}^{n}g_{ij}\phi(x_i)\Big\|_2\geq 2^{-\frac{1}{2}}c^{\frac{1}{q^*}}\Big[\sum_{i=1}^{n}K(\bx_i,\bx_i)\Big]^{\frac{1}{2}},
  $$
  where the second inequality is due to \eqref{khitchine-kahane-norm1}.
  The above lower bound coincides with the upper bound  \eqref{risk-dependent-lp-1} up to a constant factor.
  Specifically, the above upper and lower bounds show that $\ebb_{\bm{g}}\Big\|\Big(\sum_{i=1}^{n}g_{ij}\phi(\bx_i)\Big)_{j=1}^c\Big\|_{2,q^*}$ enjoys
  exactly a square-root dependency on the number of classes if $q=2$.
\end{remark}

\begin{proof}[Proof of Corollary \ref{cor:risk-dependent-shatten}]
We first consider the case $1\leq p\leq2$. Let $q\in\rbb$ satisfy $p\leq q\leq2$.
Denote $\tilde{X}_i^j=(0,\ldots,0,\bx_i,0,\ldots,0)$ with the $j$-th column being $\bx_i$. Then, we have 
\begin{equation}\label{risk-dependent-shatten-1}
  \big(\sum_{i=1}^{n}g_{ij}\bx_i\big)_{j=1}^c=\sum_{i=1}^{n}\sum_{j=1}^{c}g_{ij}\tilde{X}_i^j\quad\text{and}\quad \big(\sum_{i\in I_1}g_i\bx_i,\ldots,\sum_{i\in I_c}g_i\bx_i\big)=\sum_{j=1}^{c}\sum_{i\in I_j}g_i\tilde{X}_i^j.
\end{equation}
Since $q^*\geq2$, we can apply Jensen's inequality and Khintchine-Kahane inequality \eqref{khitchine-kahane-matrix} to derive (recall $\sigma_r(X)$ denotes the $r$-th singular value of $X$)
\begin{multline}\label{risk-dependent-shatten-2}
  \ebb_{\bm{g}}\big\|\sum_{i=1}^{n}\sum_{j=1}^{c}g_{ij}\tilde{X}_i^j\big\|_{S_{q^*}}
  \leq
  \Big[\ebb_{\bm{g}}\sum_{r=1}^{\min\{c,d\}}\sigma_r^{q^*}\big(\sum_{i=1}^{n}\sum_{j=1}^{c}g_{ij}\tilde{X}_i^j\big)\Big]^{\frac{1}{q^*}}\\
   \leq
  2^{-\frac{1}{4}}\sqrt{\frac{\pi q^*}{e}}\max\bigg\{\Big\|\Big[\sum_{i=1}^{n}\sum_{j=1}^{c}(\tilde{X}_i^j)^\top\tilde{X}_i^j\Big]^{\frac{1}{2}}\Big\|_{S_{q^*}},
  \Big\|\Big[\sum_{i=1}^{n}\sum_{j=1}^{c}\tilde{X}_i^j(\tilde{X}_i^j)^\top\Big]^{\frac{1}{2}}\Big\|_{S_{q^*}}\bigg\}.
\end{multline}
For any $\bu=(u_1,\ldots,u_c)\in\rbb^c$, we denote by $\diag(\bu)$ the diagonal matrix in $\rbb^{c\times c}$ with the $j$-th diagonal element being $u_j$.
The following identities can be directly checked
\begin{gather*}
  \sum_{i=1}^{n}\sum_{j=1}^{c}(\tilde{X}_i^j)^\top(\tilde{X}_i^j)=\sum_{i=1}^{n}\sum_{j=1}^{c}\|\bx_i\|_2^2\text{diag}(\be_j)=\sum_{i=1}^{n}\|\bx_i\|_2^2I_{c\times c},\\
  \sum_{i=1}^{n}\sum_{j=1}^{c}(\tilde{X}_i^j)(\tilde{X}_i^j)^\top=\sum_{i=1}^{n}\sum_{j=1}^{c}\bx_i\bx_i^\top=c\sum_{i=1}^{n}\bx_i\bx_i^\top,
\end{gather*}
where $(\be_1,\ldots,\be_c)$ forms the identity matrix $I_{c\times c}\in\rbb^{c\times c}$
Therefore,
\begin{equation}\label{risk-dependent-shatten-3}
  \Big\|\Big[\sum_{i=1}^{n}\sum_{j=1}^{c}(\tilde{X}_i^j)^\top(\tilde{X}_i^j)\Big]^{\frac{1}{2}}\Big\|_{S_{q^*}}=\Big\|\Big(\sum_{i=1}^{n}\|\bx_i\|_2^2\Big)^{\frac{1}{2}}I_{c\times c}\Big\|_{S_{q^*}}=c^{\frac{1}{q^*}}\Big[\sum_{i=1}^{n}\|\bx_i\|_2^2\Big]^{1\over 2},
\end{equation}
and
\begin{align}
  \Big\|\Big[\sum_{i=1}^{n}\sum_{j=1}^{c}(\tilde{X}_i^j)(\tilde{X}_i^j)^\top\Big]^{\frac{1}{2}}\Big\|_{S_{q^*}} & = \sqrt{c}\Big\|\Big(\sum_{i=1}^{n}\bx_i\bx_i^\top\Big)^{\frac{1}{2}}\Big\|_{S_{q^*}}
  = \sqrt{c}\Big[\sum_{r=1}^{\min\{c,d\}}\sigma_r^{q^*}\Big(\big(\sum_{i=1}^{n}\bx_i\bx_i^\top\big)^{1\over 2}\Big)\Big]^{\frac{1}{q^*}}\notag\\
  &=\sqrt{c}\Big[\sum_{r=1}^{\min\{c,d\}}\sigma_r^{\frac{q^*}{2}}\big(\sum_{i=1}^{n}\bx_i\bx_i^\top\big)\Big]^{1\over q^*}
  = \Big[c\big\|\sum_{i=1}^{n}\bx_i\bx_i^\top\big\|_{S_{\frac{q^*}{2}}}\Big]^{1\over2}.\label{risk-dependent-shatten-4}
\end{align}
Plugging  \eqref{risk-dependent-shatten-3} and \eqref{risk-dependent-shatten-4} into \eqref{risk-dependent-shatten-2} gives
\begin{equation}\label{risk-dependent-shatten-5}
  \ebb_{\bm{g}}\Big\|\sum_{i=1}^{n}\sum_{j=1}^{c}g_{ij}\tilde{X}_i^j\Big\|_{S_{q^*}}\leq 2^{-\frac{1}{4}}\sqrt{\frac{\pi q^*}{e}}\max\bigg\{c^{\frac{1}{q^*}}\Big[\sum_{i=1}^{n}\|\bx_i\|_2^2\Big]^{\frac{1}{2}},
  c^{\frac{1}{2}}\Big\|\sum_{i=1}^{n}\bx_i\bx_i^\top\Big\|_{S_{\frac{q^*}{2}}}^{\frac{1}{2}}\bigg\}.
\end{equation}

Applying again Jensen's inequality and Khintchine-Kahane inequality \eqref{khitchine-kahane-matrix} gives
\begin{align}
  \ebb_{\bm{g}}\big\|\sum_{j=1}^{c}&\sum_{i\in I_j}g_{i}\tilde{X}_i^j\big\|_{S_{q^*}}
  \leq
  \Big[\ebb_{\bm{g}}\sum_{r=1}^{\min\{c,d\}}\sigma_r^{q^*}\big(\sum_{j=1}^{c}\sum_{i\in I_j}g_{i}\tilde{X}_i^j\big)\Big]^{\frac{1}{q^*}}\notag\\
  & \leq
  2^{-\frac{1}{4}}\sqrt{\frac{\pi q^*}{e}}\max\bigg\{\Big\|\Big[\sum_{j=1}^{c}\sum_{i\in I_j}(\tilde{X}_i^j)^\top\tilde{X}_i^j\Big]^{\frac{1}{2}}\Big\|_{S_{q^*}},
  \Big\|\Big[\sum_{j=1}^{c}\sum_{i\in I_j}\tilde{X}_i^j(\tilde{X}_i^j)^\top\Big]^{\frac{1}{2}}\Big\|_{S_{q^*}}\bigg\}.\label{risk-dependent-shatten-6}
\end{align}
It can be directly checked that
\begin{gather*}
  \sum_{j=1}^{c}\sum_{i\in I_j}(\tilde{X}_i^j)^\top(\tilde{X}_i^j)=\sum_{j=1}^{c}\sum_{i\in I_j}\|\bx_i\|_2^2\text{diag}(\be_j)=\text{diag}\big(\sum_{i\in I_1}\|\bx_i\|_2^2,\ldots,
  \sum_{i\in I_c}\|\bx_i\|_2^2\big),\\
  \sum_{j=1}^{c}\sum_{i\in I_j}(\tilde{X}_i^j)(\tilde{X}_i^j)^\top=\sum_{j=1}^{c}\sum_{i\in I_j}\bx_i\bx_i^\top=\sum_{i=1}^{n}\bx_i\bx_i^\top,
\end{gather*}
from which and $q^*\geq2$ we derive
\begin{gather*}
  \Big\|\Big[\sum_{j=1}^{c}\sum_{i\in I_j}(\tilde{X}_i^j)^\top\tilde{X}_i^j\Big]^{\frac{1}{2}}\Big\|_{S_{q^*}}= \Big[\sum_{j=1}^{c}\Big(\sum_{i\in I_j}\|\bx_i\|_2^2\Big)^{\frac{q^*}{2}}\Big]^{\frac{1}{q^*}}\leq
  \Big[\sum_{j=1}^{c}\sum_{i\in I_j}\|\bx_i\|_2^2\Big]^{\frac{1}{2}}= \Big[\sum_{i=1}^{n}\|\bx_i\|_2^2\Big]^{\frac{1}{2}},\\
  \Big\|\Big[\sum_{j=1}^{c}\sum_{i\in I_j}\tilde{X}_i^j(\tilde{X}_i^j)^\top\Big]^{\frac{1}{2}}\Big\|_{S_{q^*}}=\Big\|\Big[\sum_{i=1}^{n}\bx_i\bx_i^\top\Big]^{\frac{1}{2}}\Big\|_{S_{q^*}}\leq
  \Big\|\Big[\sum_{i=1}^{n}\bx_i\bx_i^\top\Big]^{\frac{1}{2}}\Big\|_{S_2}=\Big[\sum_{i=1}^{n}\|\bx_i\|_2^2\Big]^{\frac{1}{2}},
\end{gather*}
where we have used deduction similar to \eqref{risk-dependent-shatten-4} in the last identity.
Plugging the above two inequalities back into \eqref{risk-dependent-shatten-6} implies
\begin{equation}\label{risk-dependent-shatten-7}
  \ebb_{\bm{g}}\big\|\sum_{j=1}^{c}\sum_{i\in I_j}g_{i}\tilde{X}_i^j\big\|_{S_{q^*}}\leq 2^{-\frac{1}{4}}\sqrt{\frac{\pi q^*}{e}}\Big[\sum_{i=1}^{n}\|\bx_i\|_2^2\Big]^{\frac{1}{2}}.
\end{equation}
Plugging \eqref{risk-dependent-shatten-5} and \eqref{risk-dependent-shatten-7} into Theorem \ref{thm:risk-data-dependent} and noting that $H_{S_p}\subset H_{S_q}$ we get the following inequality with probability at least $1-\delta$
\begin{equation}\label{risk-dependent-shatten-8}
   A_{S_p}\leq \frac{2^{\frac{3}{4}}\pi\Lambda}{n\sqrt{e}}\inf\limits_{p\leq q\leq 2}(q^*)^{1\over 2}\bigg\{L_1\max\Big\{c^{\frac{1}{q^*}}\Big[\sum_{i=1}^{n}\|\bx_i\|_2^2\Big]^{\frac{1}{2}},
  c^{\frac{1}{2}}\Big\|\sum_{i=1}^{n}\bx_i\bx_i^\top\Big\|_{S_{\frac{q^*}{2}}}^{\frac{1}{2}}\Big\}+L_2\Big[\sum_{i=1}^{n}\|\bx_i\|_2^2\Big]^{\frac{1}{2}}\bigg\}.
\end{equation}
This finishes the proof for the case $p\leq2$.

We now consider the case $p>2$.
For any $W$ with $\|W\|_{S_p}\leq\Lambda$, we have $\|W\|_{S_2}\leq \min\{c,d\}^{\frac{1}{2}-\frac{1}{p}}\Lambda$.
The stated bound \eqref{risk-dependent-shatten} for the case $p>2$ then follows by recalling the established generalization bound \eqref{risk-dependent-shatten-8} for $p=2$.
\end{proof}


\subsection{Proof of Bounds by Covering Numbers\label{sec:proof-independent}}

We use the tool of empirical $\ell_\infty$-norm CNs to prove data-dependent bounds given in subsection \ref{sec:data-independent-bounds}.
The key observation to proceed with the proof is that the empirical $\ell_\infty$-norm CNs of $F_{\tau,\Lambda}$ w.r.t.
the training examples can be controlled by that of $\widetilde{H}_{\tau}$ w.r.t. an enlarged data set of cardinality $nc$, due to
the Lipschitz continuity of loss functions w.r.t. the $\ell_\infty$-norm~\citep{zhang2004statistical,tewari2015generalization}.
The remaining problem is to estimate the empirical CNs of $\widetilde{H}_{\tau}$,
which, by the universal relationship between fat-shattering dimension and CNs (Part (a) of Lemma \ref{lem:shattering}),
can be further transferred to the estimation of fat-shattering dimension. Finally,  the problem of estimating fat-shattering dimension
reduces to the estimation of \emph{worst case} RC (Part (b) of Lemma \ref{lem:shattering}).
We summarize this deduction process in the proof of Theorem \ref{thm:covering-bound-data-independent}.
\begin{definition}[Covering number]\label{def:covering-number}
  Let $F$ be a class of real-valued functions defined over a space $\widetilde{\zcal}$ and $S':=\{\tilde{\bz}_1,\ldots,\tilde{\bz}_n\}\in\tilde{\zcal}^n$
  of cardinality $n$. For any $\epsilon>0$, the empirical $\ell_\infty$-norm CN
  $\ncal_\infty(\epsilon,F,S')$ w.r.t. $S'$ is defined as the minimal number $m$ of a collection of vectors $\bv^1,\ldots,\bv^m\in\rbb^n$
  such that ($\bv_i^j$ is the $i$-th component of the vector $\bv^j$)
  $$
    \sup_{f\in F}\min_{j=1,\ldots,m}\max_{i=1,\ldots,n}|f(\tilde{\bz}_i)-\bv_i^j|\leq\epsilon.
  $$
  In this case, we call $\{\bv^1,\ldots,\bv^m\}$ an $(\epsilon,\ell_\infty)$-cover of $F$ w.r.t. $S'$.
\end{definition}
\begin{definition}[Fat-Shattering Dimension]\label{def:shattering}
  Let $F$ be a class of real-valued functions defined over a space $\widetilde{\zcal}$.
  We define the fat-shattering dimension $\fat_\epsilon(F)$ at scale $\epsilon>0$ as the largest $D\in\nbb$ such that there exist $D$ points $\tilde{\bz}_1,\ldots,\tilde{\bz}_D\in\tilde{\zcal}$ and witnesses $s_1,\ldots,s_D\in\rbb$ satisfying: for any $\delta_1,\ldots,\delta_D\in\{\pm1\}$ there exists $f\in F$ with
  $$
    \delta_i(f(\tilde{\bz}_i)-s_i)\geq\epsilon/2,\qquad\forall i=1,\ldots,D.
  $$
\end{definition}
\begin{lemma}[\cite{srebro2010smoothness,rakhlin2014sequential}]\label{lem:shattering}
  Let $F$ be a class of real-valued functions defined over a space $\widetilde{\zcal}$ and $S':=\{\tilde{\bz}_1,\ldots,\tilde{\bz}_n\}\in\tilde{\zcal}^n$ of cardinality $n$.
\begin{enumerate}[(a)]
  \item If functions in $F$ take values in  $[-B,B]$, then for any $\epsilon>0$ with $\fat_\epsilon(F)<n$ we have
  $$
    \log\ncal_\infty(\epsilon,F,S')\leq \fat_\epsilon(F)\log\frac{2eBn}{\epsilon}.
  $$
  \item For any $\epsilon> 2\frak{R}_n(F)$, we have
  $\fat_\epsilon(F)<\frac{16n}{\epsilon^2}\frak{R}_n^2(F)$.
  \item For any monotone sequence $(\epsilon_k)_{k=0}^\infty$ decreasing to $0$ such that $\epsilon_0\geq \sqrt{n^{-1}\sup_{f\in\fcal}\sum_{i=1}^nf^2(\tilde{\bz}_i)}$,
  the following inequality holds for every non-negative integer $N$:
      \begin{equation}\label{entropy-integral}
        \mathfrak{R}_{S'}(F)\leq 2\sum_{k=1}^N\big(\epsilon_k+\epsilon_{k-1})\sqrt{\frac{\log\ncal_\infty(\epsilon_k,F,S')}{n}}+\epsilon_N.
      \end{equation}
\end{enumerate}
\end{lemma}

\begin{theorem}[Covering number bounds]\label{thm:covering-bound-data-independent}
  Assume that, for any $y\in\ycal$, the function $\Psi_y$ is $L$-Lipschitz continuous w.r.t. the $\ell_\infty$-norm.  Then,
  for any $\epsilon>4L\frak{R}_{nc}(\widetilde{H}_{\tau})$, the CN of $F_{\tau,\Lambda}$ w.r.t. $S=\{(\bx_1,y_1),\ldots,(\bx_n,y_n)\}$ can be bounded by
  $$
    \log\ncal_\infty(\epsilon, F_{\tau,\Lambda},S)\leq\frac{16ncL^2\frak{R}^2_{nc}(\widetilde{H}_{\tau})}{\epsilon^2}\log\frac{2e\hat{B}ncL}{\epsilon}.
  $$
\end{theorem}
\begin{proof}
We proceed with the proof in three steps. Note that $\widetilde{H}_\tau$ is a class of functions defined on a finite set $\widetilde{S}=\big\{\bp_j(\bx_i):i\in\nbb_n,j\in\nbb_c\big\}$.

\textbf{Step 1}. We first estimate the CN of $\widetilde{H}_{\tau}$ w.r.t. $\widetilde{S}$.
    For any $\epsilon>4\frak{R}_{nc}(\widetilde{H}_{\tau})$, Part (b) of Lemma \ref{lem:shattering} implies that
    \begin{equation}\label{covering-bound-data-independent-0}
    \fat_\epsilon(\widetilde{H}_{\tau})< \frac{16nc}{\epsilon^2}\frak{R}^2_{nc}(\widetilde{H}_{\tau})\leq nc.
    \end{equation}
     According to \eqref{S-tilde-identity-inner-product} and the definition of $\hat{B}$, we derive the following inequality for any $\bw$ with $\tau(\bw)\leq\Lambda$ and $i\in\nbb_n,j\in\nbb_c$
    $$
      |\langle \bw,\bp_j(\bx_i) \rangle|=|\langle\bw_j,\phi(\bx_i) \rangle| \leq \|\bw_j\|_2\|\phi(\bx_i)\|_2\leq\sup_{\bw:\tau(\bw)\leq \Lambda}\|\bw\|_{2,\infty}\|\phi(\bx_i)\|_2\leq\hat{B}.
    $$
    Then, the conditions of Part (a) in Lemma \ref{lem:shattering} are satisfied with $F=\widetilde{H}_\tau, B=\hat{B}$ and $S'=\widetilde{S}$, and we can apply it to control the CNs for any $\epsilon>4\frak{R}_{nc}(\widetilde{H}_{\tau})$ (note $\fat_\epsilon(\widetilde{H}_{\tau})< nc$ in \eqref{covering-bound-data-independent-0})
    \begin{equation}\label{covering-bound-data-independent-1}
    \log\ncal_\infty(\epsilon,\widetilde{H}_{\tau},\widetilde{S})\leq\fat_\epsilon(\widetilde{H}_{\tau})\log\frac{2e\hat{B}nc}{\epsilon}
    \leq
    \frac{16nc\frak{R}^2_{nc}(\widetilde{H}_{\tau})}{\epsilon^2}\log\frac{2e\hat{B}nc}{\epsilon},
    \end{equation}
    where the second inequality is due to \eqref{covering-bound-data-independent-0}.

\textbf{Step 2}. We now relate the empirical $\ell_\infty$-norm CNs of $\widetilde{H}_{\tau}$ w.r.t. $\widetilde{S}$ to that of $F_{\tau,\Lambda}$ w.r.t. $S$. Let
    \begin{equation}\label{covering-bound-data-independent-3}
      \Big\{\mathbf{r}^j = \big(r^j_{1,1},r^j_{1,2},\ldots,r^j_{1,c},\ldots,r^j_{n,1},r^j_{n,2},\ldots,r^j_{n,c}\big):j=1,\ldots,N\Big\}\subset\rbb^{nc}
    \end{equation}
    be an $(\epsilon,\ell_\infty)$-cover of
    \begin{multline*}
    \Big\{\big(\underbrace{\langle\bw,\bp_1(\bx_1)\rangle,\ldots,\langle\bw,\bp_c(\bx_1)\rangle}_{\text{related to } \bx_1},
     \underbrace{\langle\bw,\bp_1(\bx_2)\rangle,\ldots,\langle\bw,\bp_c(\bx_2)\rangle}_{\text{related to } \bx_2},\ldots\\
     \underbrace{\langle\bw,\bp_1(\bx_n)\rangle,\ldots,\langle\bw,\bp_c(\bx_n)\rangle}_{\text{related to }\bx_n}\big):\tau(\bw)\leq\Lambda\Big\}\subset\rbb^{nc}
    \end{multline*}
    with $N$ not larger than the right-hand side of  \eqref{covering-bound-data-independent-1}.
    Define $\mathbf{r}^j_i=\big(r^j_{i,1},\ldots,r^j_{i,c}\big)$ for all $i\in\nbb_n,j\in\nbb_N$.
    Now, we show that
    \begin{equation}\label{covering-bound-data-independent-22}
    \Big\{\big(\Psi_{y_1}(\mathbf{r}^j_1),\Psi_{y_2}(\mathbf{r}^j_2)\ldots,\Psi_{y_n}(\mathbf{r}^j_n)\big):j=1,\ldots,N\Big\}\subset\rbb^n
    \end{equation}
    would be an $(L\epsilon,\ell_\infty)$-cover of the set (note $h^{\bw}(\bx)=\big(\inn{\bw_1,\phi(\bx)},\ldots,\inn{\bw_c,\phi(\bx)}\big)$)
    \begin{equation*}
      \Big\{\big(\Psi_{y_1}(h^{\bw}(\bx_1)),\ldots,\Psi_{y_n}(h^{\bw}(\bx_n))\big):\tau(\bw)\leq\Lambda\Big\}\subset\rbb^n.
    \end{equation*}
    Indeed, for any $\bw\in H_K^c$ with $\tau(\bw)\leq\Lambda$, the construction of the cover in Eq. \eqref{covering-bound-data-independent-3} guarantees the existence of $j_{(\bw)}\in\{1,\ldots,N\}$ such that
    \begin{equation}\label{covering-bound-data-independent-33}
      \max_{1\leq i\leq n}\max_{1\leq k\leq c}\big| r^{j_{(\bw)}}_{i,k} -\langle\bw,\bp_k(\bx_i)\rangle\big|\leq\epsilon.
    \end{equation}
    Then, the Lipschitz continuity of $\Psi_y$ w.r.t. the $\ell_\infty$-norm implies that
    \begin{align*}
        \max_{1\leq i\leq n} |\Psi_{y_i}(\mathbf{r}^{j_{(\bw)}}_i)-\Psi_{y_i}(h^{\bw}(\bx_i))|
        &\leq L\max_{1\leq i\leq n}\|\mathbf{r}^{j_{(\bw)}}_i-h^{\bw}(\bx_i)\|_\infty\\
        & = L\max_{1\leq i\leq n}\max_{1\leq k\leq c}\big|r^{j_{(\bw)}}_{i,k}-\langle\bw_k,\phi(\bx_i)\rangle\big| \\
        & = L\max_{1\leq i\leq n}\max_{1\leq k\leq c}\big|r^{j_{(\bw)}}_{i,k}-\langle\bw,\bp_k(\bx_i)\rangle\big|\\
        & \leq L\epsilon,
    \end{align*}
    where we have used \eqref{S-tilde-identity-inner-product} in the third step
    and \eqref{covering-bound-data-independent-33} in the last step.
    That is, the set defined in \eqref{covering-bound-data-independent-22} is also an $(L\epsilon,\ell_\infty)$-cover of $F_{\tau,\Lambda}$ w.r.t. $S=\{(\bx_1,y_1),\ldots,(\bx_n,y_n)\}$. Therefore,
    \begin{equation}\label{covering-bound-data-independent-4}
      \log\ncal_\infty(\epsilon, F_{\tau,\Lambda},S)\leq\log\ncal_\infty(\epsilon/L,\widetilde{H}_{\tau},\widetilde{S}),\quad\forall\epsilon>0.
    \end{equation}

  \textbf{Step 3}. The stated result follows directly if we plug the complexity bound of $\widetilde{H}_{\tau}$ established in \eqref{covering-bound-data-independent-1} into \eqref{covering-bound-data-independent-4}.
\end{proof}

We can now apply the entropy integral \eqref{entropy-integral} to control $\mathfrak{R}_S(F_{\tau,\Lambda})$ in terms of $\mathfrak{R}_{nc}(\widetilde{H}_{\tau})$.
We denote by $\lceil a\rceil$ the least integer not smaller than $a\in\rbb$.
\begin{proof}[Proof of Theorem \ref{thm:rademacher-independent-bound}]
  Let
  $N=\bigg\lceil\log_2\frac{n^{-\frac{1}{2}}\sup_{h\in H_\tau}\big\|\big(\Psi_{y_i}(h(\bx_i))\big)_{i=1}^n\big\|_2}{16L\sqrt{c\log2}\mathfrak{R}_{nc}(\widetilde{H}_{\tau})}\bigg\rceil$, $\epsilon_N=16L\sqrt{c\log2}\mathfrak{R}_{nc}(\widetilde{H}_\tau)$ and $\epsilon_k=2^{N-k}\epsilon_N,k=0,\ldots,N-1$.
  It is clear that $\epsilon_0\geq n^{-\frac{1}{2}}\sup_{h\in H_\tau}\big\|\big(\Psi_{y_i}(h(\bx_i))\big)_{i=1}^n\big\|_2\geq\epsilon_0/2$ and $\epsilon_N\geq 4L\mathfrak{R}_{nc}(\widetilde{H}_\tau)$.
  Plugging the CN bounds established in Theorem \ref{thm:covering-bound-data-independent} into the entropy integral \eqref{entropy-integral}, we derive the following inequality
  \begin{equation}\label{rademacher-independent-bound-1}
    \mathfrak{R}_S(F_{\tau,\Lambda}) \leq 8L\sqrt{c}\frak{R}_{nc}(\widetilde{H}_{\tau})\sum_{k=1}^N\frac{\epsilon_k+\epsilon_{k-1}}{\epsilon_k}\sqrt{\log\frac{2e\hat{B}ncL}{\epsilon_k}}+\epsilon_N.
  \end{equation}
  We know
  \begin{align*}
    \sum_{k=1}^{N}\sqrt{\log\frac{2e\hat{B}ncL}{\epsilon_k}} & = \sum_{k=1}^{N}\sqrt{k\log 2+\log(2e\hat{B}ncL\epsilon_0^{-1})}
     \leq \sqrt{\log 2}\int_1^{N+1}\sqrt{x+\log_2(2e\hat{B}ncL\epsilon_0^{-1})}dx \\
     & = \frac{2\sqrt{\log 2}}{3}\int_1^{N+1}d\big(x+\log_2(2e\hat{B}ncL\epsilon_0^{-1})\big)^{\frac{3}{2}} \leq \frac{2\sqrt{\log 2}}{3}\log_2^{\frac{3}{2}}(4e\hat{B}ncL\epsilon_N^{-1}),
  \end{align*}
  where the last inequality follows from $$4e\hat{B}ncL\geq 2n^{-\frac{1}{2}}\sup_{h\in H_\tau}\big\|\big(\Psi_{y_i}(h(\bx_i))\big)_{i=1}^n\big\|_2\geq\epsilon_0.$$
  Plugging the above inequality back into \eqref{rademacher-independent-bound-1} gives
  \begin{align*}
    \mathfrak{R}_S(F_{\tau,\Lambda}) &\leq 16L\sqrt{c\log 2}\frak{R}_{nc}(\widetilde{H}_{\tau})\log_2^{\frac{3}{2}}\big(4e\hat{B}ncL\epsilon_N^{-1}\big)+\epsilon_N\\
    & = 16L\sqrt{c\log 2}\frak{R}_{nc}(\widetilde{H}_{\tau})\Big(1+\log_2^{\frac{3}{2}}\frac{\sqrt{c}e\hat{B}n}{4\sqrt{\log2}\mathfrak{R}_{nc}(\widetilde{H}_\tau)}\Big).
  \end{align*}
  The proof is complete by noting $e\leq4\sqrt{\log2}$.
\end{proof}


The proof of Theorem \ref{thm:risk-data-independent} is now immediate.
\begin{proof}[Proof of Theorem \ref{thm:risk-data-independent}]
  The proof is complete if we plug the RC bounds established in Theorem \ref{thm:rademacher-independent-bound} back into \eqref{risk-data-dependent-1} and noting $32\sqrt{\log2}\leq27$.
\end{proof}

\begin{proof}[Proof of Corollary \ref{cor:risk-independent-lp}]
  Plugging the complexity bounds of $\widetilde{H}_p$ given in \eqref{rademacher-independent-lp} into Theorem \ref{thm:risk-data-independent} gives the following inequality with probability at least $1-\delta$
  \begin{align*}
    A_p &\leq \frac{27\sqrt{c}L\Lambda \max_{i\in\nbb_n}\|\phi(\bx_i)\|_2}{n^{\frac{1}{2}}c^{\frac{1}{\max(2,p)}}}\Big(1+\log_2^{3\over2}\frac{\sqrt{2}\hat{B}n^{3\over2}c^{\frac{1}{2}+\frac{1}{\max(2,p)}}}{\Lambda\max_{i\in\nbb_n}\|\phi(\bx_i)\|_2}\Big)\\
    &\leq \frac{27L\Lambda \max_{i\in\nbb_n}\|\phi(\bx_i)\|_2c^{\frac{1}{2}-\frac{1}{\max(2,p)}}}{\sqrt{n}}\Big(1+\log_2^{3\over2}(\sqrt{2}n^{3\over2}c)\Big),
  \end{align*}
  where we have used the following inequality in the last step
  $$
    \hat{B}=\max_{i\in\nbb_n}\|\phi(\bx_i)\|_2\sup_{\bw:\|\bw\|_{2,p}\leq \Lambda}\|\bw\|_{2,\infty}\leq \Lambda\max_{i\in\nbb_n}\|\phi(\bx_i)\|_2.
  $$
  The proof of Corollary \ref{cor:risk-independent-lp} is complete.
\end{proof}
\begin{proof}[Proof of Corollary \ref{cor:risk-independent-shatten}]
Consider any $W=(\bw_1,\ldots,\bw_c)\in\rbb^{d\times c}$. If $1<p\leq 2$, then
$$
  \|W\|_{S_p}\geq \|W\|_{S_2}=\|W\|_{2,2}\geq \|W\|_{2,\infty}.
$$
Otherwise, according to the following inequality for any semi-definite positive matrix $A=(a_{j\tilde{j}})_{j,\tilde{j}=1}^c$ (e.g., (1.67) in \citep{tao2012topics})
$$
  \|A\|_{S_{\tilde{p}}}\geq\big[\sum_{j=1}^{c}|a_{jj}|^{\tilde{p}}\big]^{1\over \tilde{p}},\quad\forall \tilde{p}\geq1,
$$
we derive
\begin{align*}
     \|W\|_{S_p} &= \|(W^\top W)^{1\over 2}\|_{S_p}=\Big\|\Big[\big(\bw_j^\top\bw_{\tilde{j}}\big)_{j,\tilde{j}=1}^c\Big]^{1\over2}\Big\|_{S_p}=\Big\|\big(\bw_j^\top\bw_{\tilde{j}}\big)_{j,\tilde{j}=1}^c\Big\|_{S_{p\over 2}}^{1\over 2}\\
     & \geq \Big[\sum_{j=1}^{c}(\bw_j^\top\bw_j)^{p\over 2}\Big]^{1\over p}\geq \max_{j=1,\ldots,c}\|\bw_j\|_2=\|W\|_{2,\infty}.
\end{align*}
Thereby, for the specific choice $\tau(W)=\|W\|_{S_p},p\geq1$, we have
\begin{equation}\label{risk-independent-shatten-1}
  \hat{B}=\max_{i\in\nbb_n}\|\bx_i\|_2\sup_{W:\|W\|_{S_p}\leq\Lambda}\|W\|_{2,\infty}\leq \Lambda \max_{i\in\nbb_n}\|\bx_i\|_2.
\end{equation}

We now consider two cases. If $1<p\leq 2$, plugging the RC bounds of $\widetilde{H}_{S_p}$ given in \eqref{rademacher-independent-schatten} into Theorem \ref{thm:risk-data-independent} gives the following inequality with probability at least $1-\delta$
$$
  A_{S_p}\leq \frac{27L\Lambda \max_{i\in\nbb_n}\|\bx_i\|_2}{\sqrt{n}}\Big(1+\log_2^{3\over2}\frac{\sqrt{2}\hat{B}n^{3\over2}c}{\Lambda\max_{i\in\nbb_n}\|\bx\|_i}\Big)\leq \frac{27
  L\Lambda \max_{i\in\nbb_n}\|\bx_i\|_2}{\sqrt{n}}\Big(1+\log_2^{3\over2}\big(\sqrt{2}n^{3\over2}c\big)\Big),
$$
where the last step follows from \eqref{risk-independent-shatten-1}. If $p>2$, analyzing analogously yields the following inequality with probability at least $1-\delta$
$$
  A_{S_p}
  \leq \frac{27L\Lambda \max_{i\in\nbb_n}\|\bx_i\|_2\min\{c,d\}^{\frac{1}{2}-\frac{1}{p}}}{\sqrt{n}}\Big(1+\log_2^{3\over2}(\sqrt{2}n^{3\over2}c)\Big).
$$
The stated error bounds follow by combining the above two cases together.
\end{proof}

\subsection{Proofs on worst-case Rademacher Complexities\label{sec:proof-independent-rademacher-lp}}

\begin{proof}[Proof of Proposition \ref{prop:rademacher-independent-lp}]
We proceed with the proof by distinguishing two cases according to the value of $p$.

We first consider the case $1\leq p\leq 2$, for which the RC can be lower bounded by
    \begin{align}
       \frak{R}_{nc}(\widetilde{H}_p)
       & = \max_{\bv^i\in\widetilde{S}:i\in\nbb_{nc}}\frac{1}{nc}\ebb_\epsilon\sup_{\|\bw\|_{2,p}\leq \Lambda}\sum_{i=1}^{nc}\epsilon_i\inn{\bw,\bv^i}\notag\\
       & =\max_{\bv^i\in\widetilde{S}:i\in\nbb_{nc}}\frac{1}{nc}\ebb_\epsilon\sup_{\|\bw\|_{2,p}\leq \Lambda}\inn{\bw,\sum_{i=1}^{nc}\epsilon_i\bv^i} \notag\\
       & = \max_{\bv^i\in\widetilde{S}:i\in\nbb_{nc}}\frac{\Lambda}{nc}\ebb_\epsilon\big\|\sum_{i=1}^{nc}\epsilon_i\bv^i\big\|_{2,p^*}\label{rademacher-independent-lp<2}\\
       & \geq \max_{\bv^1\in\widetilde{S}}\frac{\Lambda}{nc}\ebb_\epsilon\big|\sum_{i=1}^{nc}\epsilon_i\big|\|\bv^1\|_{2,p^*}\notag,
    \end{align}
    where the equality \eqref{rademacher-independent-lp<2} follows from the definition of dual norm and the inequality follows by taking $\bv^1=\cdots=\bv^{nc}$.
    Applying the Khitchine-Kahane inequality \eqref{khitchine-kahane-norm1} and using the definition of $\widetilde{S}$ in \eqref{tilde-S}, we then derive ($\|\bv\|_{2,p}=\|\bv\|_{2,\infty}$ for $\bv\in\widetilde{S}$)
    $$
      \frak{R}_{nc}(\widetilde{H}_p)  \geq \frac{\Lambda}{\sqrt{2nc}}\max_{\bv^1\in\widetilde{S}}\|\bv^1\|_{2,p^*}=\frac{\Lambda \max_{i\in\nbb_n}\|\phi(\bx_i)\|_2}{\sqrt{2nc}}.
    $$
    Furthermore, according to the subset relationship $\widetilde{H}_p\subset\widetilde{H}_2,1\leq p\leq2$ due to the monotonicity of $\|\cdot\|_{2,p}$, the term $\frak{R}_{nc}(\widetilde{H}_p)$ can also be upper bounded by ($\bv_j^i$ denotes the $j$-th component of $\bv^i$)
    \begin{align*}
       \frak{R}_{nc}(\widetilde{H}_p) &\leq \frak{R}_{nc}(\widetilde{H}_2)
       = \max_{\bv^i\in\widetilde{S}:i\in\nbb_{nc}}\frac{\Lambda}{nc}\ebb_\epsilon\big\|\sum_{i=1}^{nc}\epsilon_i\bv^i\big\|_{2,2}\\
       &\leq \max_{\bv^i\in\widetilde{S}:i\in\nbb_{nc}}\frac{\Lambda}{nc}\sqrt{\sum_{j=1}^{c}\ebb_\epsilon\|\sum_{i=1}^{nc}\epsilon_i\bv_j^i\|_2^2}
        = \max_{\bv^i\in\widetilde{S}:i\in\nbb_{nc}}\frac{\Lambda}{nc}\sqrt{\sum_{j=1}^{c}\sum_{i=1}^{nc}\|\bv_j^i\|_2^2}\\
       & = \max_{\bv^i\in\widetilde{S}:i\in\nbb_{nc}}\frac{\Lambda}{nc}\sqrt{\sum_{i=1}^{nc}\|\bv^i\|_{2,\infty}^2}
       = \frac{\Lambda \max_{i\in\nbb_n}\|\phi(\bx_i)\|_2}{\sqrt{nc}},
    \end{align*}
    where the first identity is due to \eqref{rademacher-independent-lp<2}, the second inequality is due to Jensen's inequality and the last second identity is due to
    $\sum_{j=1}^{c}\|\bv_j\|_2^2=\|\bv\|_{2,\infty}^2$ for all $\bv\in\widetilde{S}$.

\smallskip
We now turn to the case $p>2$. In this case, we have
    \begin{align}
         \frak{R}_{nc}(\widetilde{H}_p) &= \max_{\bv^i\in\widetilde{S}:i\in\nbb_{nc}}\frac{1}{nc}\ebb_\epsilon\sup_{\|\bw\|_{2,p}\leq \Lambda}\sum_{i=1}^{nc}
         \epsilon_i\sum_{j=1}^{c}\inn{\bw_j,\bv_j^i}\notag\\
         & \geq \max_{\bv^i\in\widetilde{S}:i\in\nbb_{nc}}\frac{1}{nc}\ebb_\epsilon\sup_{\|\bw_j\|_2^p\leq\frac{\Lambda^p}{c}:j\in\nbb_c}\sum_{i=1}^{nc}
         \epsilon_i\sum_{j=1}^{c}\inn{\bw_j,\bv_j^i}\notag\\
         & = \max_{\bv^i\in\widetilde{S}:i\in\nbb_{nc}}\frac{1}{nc}\sum_{j=1}^{c}
         \ebb_\epsilon\sup_{\|\bw_j\|_2^p\leq\frac{\Lambda^p}{c}}\sum_{i=1}^{nc}\epsilon_i\inn{\bw_j,\bv_j^i}\notag\\
         & = \max_{\bv^i\in\widetilde{S}:i\in\nbb_{nc}}\frac{1}{nc}\sum_{j=1}^{c}
         \ebb_\epsilon\sup_{\|\bw_j\|_2^p\leq\frac{\Lambda^p}{c}}\inn{\bw_j,\sum_{i=1}^{nc}\epsilon_i\bv_j^i},\notag
    \end{align}
    where we can exchange the summation over $j$ with the supremum in the second identity since the constraint $\|\bw_j\|_2^p\leq\frac{\Lambda^p}{c},j\in\nbb_c$ are decoupled.
    According to the definition of dual norm and the Khitchine-Kahane inequality \eqref{khitchine-kahane-norm1}, $\frak{R}_{nc}(\widetilde{H}_p)$ can be further controlled by
    \begin{align}
      \frak{R}_{nc}(\widetilde{H}_p) & \geq \max_{\bv^i\in\widetilde{S}:i\in\nbb_{nc}}\frac{1}{nc}\sum_{j=1}^{c}
         \frac{\Lambda}{c^{\frac{1}{p}}}\ebb_\epsilon\|\sum_{i=1}^{nc}\epsilon_i\bv_j^i\|_2\notag\\
         & \geq \max_{\bv^i\in\widetilde{S}:i\in\nbb_{nc}}\frac{1}{nc}\sum_{j=1}^{c}
         \frac{\Lambda}{\sqrt{2}c^{\frac{1}{p}}}\big[\sum_{i=1}^{nc}\|\bv_j^i\|_2^2\big]^{\frac{1}{2}}.\label{rademacher-independent-lp>2}
    \end{align}
    We can find $\bar{\bv}^1,\ldots,\bar{\bv}^{nc}\in\widetilde{S}$ such that for each $j\in\nbb_c$, there are exactly $n$ $\bar{\bv}^k$ with $\|\bar{\bv}_j^k\|_2=\max_{i\in\nbb_n}\|\phi(\bx_i)\|_2$. Then,
    $
      \sum_{i=1}^{nc}\|\bar{\bv}_j^i\|_2^2=n\max_{i\in\nbb_n}\|\phi(\bx_i)\|_2^2,\forall j\in\nbb_c,
    $
    which, coupled with \eqref{rademacher-independent-lp>2}, implies that
    $$
      \frak{R}_{nc}(\widetilde{H}_p) \geq \frac{1}{nc}\sum_{j=1}^{c}
         \frac{\Lambda}{\sqrt{2}c^{\frac{1}{p}}}\big[n\max_{i\in\nbb_n}\|\phi(\bx_i)\|_2^2\big]^{\frac{1}{2}}
         \geq \Lambda \max_{i\in\nbb_n}\|\phi(\bx_i)\|_2(2n)^{-\frac{1}{2}}c^{-\frac{1}{p}}.
    $$
    On the other hand, according to \eqref{rademacher-independent-lp<2} and Jensen's inequality, we derive
    $$
      \frac{nc\frak{R}_{nc}(\widetilde{H}_p)}{\Lambda}=\max_{\bv^i\in\widetilde{S}:i\in\nbb_{nc}}\ebb_\epsilon\big\|\sum_{i=1}^{nc}\epsilon_i\bv^i\big\|_{2,p^*}
      \leq \max_{\bv^i\in\widetilde{S}:i\in\nbb_{nc}}\Big[\ebb_\epsilon\sum_{j=1}^{c}\|\sum_{i=1}^{nc}\epsilon_i\bv_j^i\|_2^{p^*}\Big]^{\frac{1}{p^*}}.
    $$
    By the Khitchine-Kahane inequality \eqref{khitchine-kahane} with $p^*\leq2$ and the following elementary inequality
    $
      \sum_{j=1}^{c}|t_j|^{\tilde{p}}\leq c^{1-\tilde{p}}\big(\sum_{j=1}^{c}|t_j|\big)^{\tilde{p}},\forall 0<\tilde{p}\leq1,
    $
    we get
    \begin{align}
      \frac{nc\frak{R}_{nc}(\widetilde{H}_p)}{\Lambda}
      & \leq \max_{\bv^i\in\widetilde{S}:i\in\nbb_{nc}}\Big[\sum_{j=1}^{c}\big(\sum_{i=1}^{nc}\|\bv_j^i\|_2^2\big)^{\frac{p^*}{2}}\Big]^{\frac{1}{p^*}} \label{data-independent>2-upper}\\
      & \leq \max_{\bv^i\in\widetilde{S}:i\in\nbb_{nc}}\Big[c^{1-\frac{p^*}{2}}\Big(\sum_{j=1}^{c}\sum_{i=1}^{nc}\|\bv_j^i\|_2^2\Big)^{\frac{p^*}{2}}\Big]^{\frac{1}{p^*}} \notag\\
      & \leq \sqrt{nc}c^{\frac{1}{p^*}-\frac{1}{2}}\max_{i\in\nbb_n}\|\phi(\bx_i)\|_2= \sqrt{n}c^{1-\frac{1}{p}}\max_{i\in\nbb_n}\|\phi(\bx_i)\|_2,\notag
    \end{align}
    where we have used the inequality $\sum_{j=1}^{c}\|\bv_j\|_2^2\leq \max_{i\in\nbb_n}\|\phi(\bx_i)\|_2^2$ for all $\bv\in\widetilde{S}$ in the last inequality.

    The above upper and lower bounds in the two cases can be written compactly as \eqref{rademacher-independent-lp}. The proof is complete.
\end{proof}

\subsection{Proofs on Applications}
\begin{proof}[Proof of Example \ref{exp:loss-margin}]
  According to the monotonicity of $\ell$, there holds
  $$
    \ell(\rho_h(\bx,y))=\ell\big(\min_{y':y'\neq y}(h_y(\bx)-h_{y'}(\bx))\big)=\max_{y':y'\neq y}\ell(h_y(\bx)-h_{y'}(\bx))=\Psi_y^\ell(h(\bx)).
  $$
  It remains to show the Lipschitz continuity of $\Psi_y^{\ell}$.
  Indeed, for any $\bt,\bt'\in\rbb^c$, we have
  \begin{align*}
  |\Psi_y^\ell(\bt)-\Psi_y^\ell(\bt')|&=\big|\max_{y':y'\neq y}\ell(t_y-t_{y'})-\max_{y':y'\neq y}\ell(t'_y-t'_{y'})\big|\\
  &\leq\max_{y':y'\neq y}\big|\ell(t_y-t_{y'})-\ell(t'_y-t'_{y'})\big|\\
  &\leq\max_{y':y'\neq y}L_\ell|(t_y-t_{y'})-(t'_y-t'_{y'})|\\
  &\leq2L_\ell\max_{y'\in\ycal}|t_{y'}-t'_{y'}|\\
  &\leq 2L_\ell\|\bt-\bt'\|_2,
  \end{align*}
  where in the first inequality we have used the elementary inequality
  \begin{equation}\label{eq:leqRels}
  |\max\{a_1,\ldots,a_c\}-\max\{b_1,\ldots,b_c\}|\leq\max\{|a_1-b_1|,\ldots,|a_c-b_c|\},\quad\forall\bm{a},\bm{b}\in\rbb^c
  \end{equation}
  and the second inequality is due to the Lipschitz continuity of $\ell$.
\end{proof}
\begin{proof}[Proof of Example \ref{exp:soft-max}]
  Define the function $f^m:\rbb^{c}\to\rbb$ by $f^m(\bt)=\log\big(\sum_{j=1}^{c}\exp(t_j)\big)$. For any $\bt\in\rbb^{c}$, the partial gradient of $f^m$ with respect to $t_k$ is
  $$
    \frac{\partial f^m(\bt)}{\partial t_k}=\frac{\exp(t_k)}{\sum_{j=1}^{c}\exp(t_j)},\quad\forall k=1,\ldots,c,
  $$
  from which we derive that $\|\nabla f^m(\bt)\|_1=1,\forall \bt\in\rbb^{c}$. Here $\nabla$ denotes the gradient operator. For any $\bt,\bt'\in\rbb^{c}$, according to the mean-value theorem we know the existence of $\alpha\in[0,1]$ such that
  \begin{align*}
    |f^m(\bt)-f^m(\bt')| & = \big|\big\langle\nabla f^m(\alpha \bt+(1-\alpha)\bt'),\bt-\bt'\big\rangle\big| \\
     & \leq \|\nabla f^m(\alpha\bt+(1-\alpha)\bt')\|_1\|\bt-\bt'\|_\infty=\|\bt-\bt'\|_\infty.
  \end{align*}
  It then follows that
  \begin{align*}
    |\Psi_y^m(\bt)-\Psi_y^m(\bt')| & = |f^m\big((t_j-t_y)_{j=1}^c\big)-f^m\big((t_j'-t_y')_{j=1}^c\big)| \\
     & \leq \big\|(t_j-t_y)_{j=1}^c-(t_j'-t_y')_{j=1}^c\big\|_\infty \\
     & \leq 2\|\bt-\bt'\|_\infty.
  \end{align*}
  That is, $\Psi_y^m$ is $2$-Lipschitz continuous w.r.t. the $\ell_\infty$-norm.
\end{proof}
\begin{proof}[Proof of Example \ref{exp:ww}]
  For any $\bt,\bt'\in\rbb^c$, we have
  \begin{align*}
    |\tilde{\Psi}_y^{\ell}(\bt)-\tilde{\Psi}_y(\bt')| & = \big|\sum_{j=1}^{c}\ell(t_y-t_j)-\sum_{j=1}^{c}\ell(t'_y-t'_j)\big|
     \leq \sum_{j=1}^{c}\big|\ell(t_y-t_j)-\ell(t_y'-t_j')\big|\\
     & \leq L_\ell c|t_y-t_y'|+L_\ell\sum_{j=1}^{c}|t_j-t_j'|  \leq L_\ell c|t_y-t_y'|+L_\ell\sqrt{c}\|\bt-\bt'\|_2.
  \end{align*}
  The Lipschitz continuity of $\tilde{\Psi}_y^{\ell}(\bt)$ w.r.t. $\ell_\infty$-norm is also clear.
\end{proof}
\begin{proof}[Proof of Example \ref{exp:llw}]
  For any $\bt,\bt'\in\Omega$, we have
  \begin{align*}
    |\bar{\Psi}_y^{\ell}(\bt)-\bar{\Psi}_y^{\ell}(\bt')| & = \Big|\sum_{j=1,j\neq y}^{c}\big[\ell(-t_j)-\ell(-t'_j)\big]\Big| \leq L_\ell\sum_{j=1,j\neq y}^{c}|t_j-t'_j| \\
     & \leq L_\ell\sqrt{c}\|\bt-\bt'\|_2 \leq L_\ell c\|\bt-\bt'\|_\infty.
  \end{align*}
  This establishes the Lipschitz continuity of $\bar{\Psi}_y^\ell$.
\end{proof}

\begin{proof}[Proof of Example \ref{exp:jenssen}]
  For any $\bt,\bt'\in\Omega$, we have
  $$
    \big|\hat{\Psi}_y^{\ell}(\bt)-\hat{\Psi}_y^{\ell}(\bt')\big|=|\ell(t_y)-\ell(t'_y)|\leq L_\ell|t_y-t'_y|\leq L_\ell\|\bt-\bt'\|_\infty.
  $$
  This establishes the Lipschitz continuity of $\hat{\Psi}_y^{\ell}$.
\end{proof}

\begin{proof}[Proof of Example \ref{exp:top-k}]
 It is clear that
  \begin{equation}\label{top-k-identity}
    \sum_{j=1}^{k}t_{[j]}=\max_{1\leq i_1<i_2<\cdots<i_k\leq c}[t_{i_1}+\cdots+t_{i_k}],\quad\forall\bt\in\rbb^c.
  \end{equation}
  For any $\bt,\bt'\in\rbb^c$, we have
  \begin{align}
    &|\Psi_y^k(\bt)\!-\!\Psi_y^k(\bt')|\notag\\
    &\leq \frac{1}{k}\Big|\sum_{j=1}^{k}(1_{y\neq 1}\!+\!t_1\!-\!t_y,\ldots,1_{y\neq c}\!+\!t_c\!-\!t_{y})_{[j]}\!-\!\sum_{j=1}^{k}(1_{y\neq 1}\!+\!t'_1\!-\!t'_y,\ldots,1_{y\neq c}\!+\!t'_c\!-\!t'_y)_{[j]}\Big|\notag\\
    &=\frac{1}{k}\Big|\max_{1\leq i_1<i_2<\cdots<i_k\leq c}\sum_{r=1}^{k}(1_{y\neq i_r}+t_{i_r}-t_y)-\max_{1\leq i_1<i_2<\cdots<i_k\leq c}\sum_{r=1}^{k}(1_{y\neq i_r}+t'_{i_r}-t'_y)\Big|\notag\\
    &\leq \frac{1}{k}\max_{1\leq i_1<i_2<\cdots<i_k\leq c}\Big|\sum_{r=1}^{k}(1_{y\neq i_r}+t_{i_r}-t_y)-\sum_{r=1}^{k}(1_{y\neq i_r}+t'_{i_r}-t'_y)\Big|\notag\\
    &\leq\frac{1}{k}\max_{1\leq i_1<i_2<\cdots<i_k\leq c}\big|\sum_{r=1}^{k}(t_{i_r}-t'_{i_r})\big|+|t_y-t'_y|\notag\\
    &\leq\frac{1}{\sqrt{k}}\max_{1\leq i_1<i_2<\cdots<i_k\leq c}\big[\sum_{r=1}^{k}(t_{i_r}-t'_{i_r})^2\big]^{\frac{1}{2}}+|t_y-t'_y|\label{top-k-1}\\
    &\leq\frac{1}{\sqrt{k}}\big[\sum_{j=1}^{c}(t_j-t'_j)^2\big]^{\frac{1}{2}}+|t_y-t'_y|,\notag
  \end{align}
  where the first and the second inequality are due to \eqref{eq:leqRels} and the first identity is due to \eqref{top-k-identity}. This establishes the Lipschitz continuity w.r.t. a variant of the $\ell_2$-norm. The $2$-Lipschitz continuity of $\Psi_y^k$ w.r.t. $\ell_\infty$-norm is clear from \eqref{top-k-1}. The proof is complete.
\end{proof}

\section{Conclusion}

Motivated by the ever-growing number of label classes occurring in classification problems, we develop two approaches
to derive
data-dependent error bounds that scale favorably in the number of labels.
The two approaches are based on Gaussian and Rademacher complexities, respectively,
of a related linear function class defined
over a finite set induced from the training examples, for which we establish tight upper and lower bounds matching within a constant factor.
Due to the ability to preserve the correlation among class-wise components,
both of these data-dependent bounds admit an improved dependency on the number of classes over the state of the art.

Our first approach is based on a novel structural result on Gaussian complexities of function classes composed by Lipschitz
operators measured by a variant of the $\ell_2$-norm. We show the advantage of our structural result
over the previous one \eqref{structural-rademacher} in \citep{lei2015multi,cortes2016structured,maurer2016vector} by capturing
better the Lipschitz property of loss functions and yielding tighter bounds,
which is the case for some popular MC-SVMs \citep{weston1998multi,jenssen2012scatter,lapin2015top}.


Our second approach is based on a novel structural result controlling the worst-case Rademacher complexity of the loss function
class by $\ell_\infty$-norm covering numbers of an associated linear function class. Our approach addresses the fact that several
loss functions are Lipschitz continuous w.r.t. the $\ell_\infty$ norm with a moderate Lipschitz constant~\citep{zhang2004statistical}.
This allows us to obtain error bounds exhibiting a logarithmic dependency on the number of classes for the MC-SVM by
\citet{crammer2002algorithmic} and multinomial logistic regression,  significantly improving the existing
square-root dependency \citep{lei2015multi,zhang2004statistical}.


We show that each of these two approaches has its own advantages and can outperform the other
for some applications depending on the Lipschitz continuity of the associated loss function. We report experimental results
to show that our theoretical bounds really capturing well the factors influencing models' generalization performance,
and can imply a structural risk working well in model selection.
Furthermore, we propose an efficient algorithm to train $\ell_p$-norm MC-SVMs based on the Frank-Wolfe algorithm.


We sketch here some possible directions for future study.
First, research in classification with many classes increasingly focuses on \emph{multi-label} classification with
each output $\mathbf{y}_i$ taking values in $\{0,1\}^c$~\citep{bhatia2015sparse,yu2014large,jain2016extreme}.
It would be interesting to transfer the results obtained in the present analysis to the multi-label case.
To this aim, it is helpful to check the Lipschitz continuity of loss functions in multi-label learning, which,
as in the present work, are typically of the form $\Psi_{\mathbf{y}}(h(\bx))$~\citep{dembczynski2012label,yu2014large},
(e.g., Hamming loss, subset zero-one loss, and ranking loss~\citep{dembczynski2012label}).
Secondly, we study examples with the functional $\tau$ depending on the components of $\bw$ in the RKHS. It would be
interesting to consider examples with $\tau$ defined in other forms, such as those in \citep{shi2011concentration,guo2017thresholded}.
Thirdly, our error bounds are derived for convex surrogates of the $0$-$1$ loss. It would be interesting to relate these
error bounds to excess generalization errors measured by the $0$-$1$ loss~\citep{zhang2004statistical,zhang2004statisticalb,bartlett2006convexity,tewari2007consistency}.

\section*{Acknowledgment}
We thank Rohit Babbar, Alexander Binder, Moustapha Cisse, Vitaly Kuznetsov, Stephan Mandt, Mehryar Mohri and Robert Vandermeulen for interesting discussions.

YL and DZ acknowledge support from the NSFC/RGC Joint Research Scheme [RGC Project No.
N\_C CityU120/14 and NSFC Project No. 11461161006]. MK acknowledges
support from the German Research Foundation (DFG) award KL 2698/2-1 and from the
Federal Ministry of Science and Education (BMBF) awards 031L0023A and 031B0187B.

\appendix
\numberwithin{equation}{section}
\numberwithin{theorem}{section}
\numberwithin{figure}{section}
\numberwithin{table}{section}
\renewcommand{\thesection}{{\Alph{section}}}
\renewcommand{\thesubsection}{\Alph{section}.\arabic{subsection}}
\renewcommand{\thesubsubsection}{\Roman{section}.\arabic{subsection}.\arabic{subsubsection}}
\setcounter{secnumdepth}{-1}
\setcounter{secnumdepth}{3}

\section{Khintchine-Kahane Inequality}


The following Khintchine-Kahane inequality~\citep{de2012decoupling,lust1991non} provides a powerful tool to control the $p$-th norm of the summation of Rademacher (Gaussian) series.
\begin{lemma}\label{lem:khitchine-kahane}
\begin{enumerate}[(a)]
  \item Let $\bv_1,\ldots,\bv_n\in\hcal$, where $\hcal$ is a Hilbert space with $\|\cdot\|$ being the associated norm. Let $\epsilon_1,\ldots,\epsilon_n$ be a sequence of independent Rademacher variables. Then, for any $p\geq1$ there holds
  \begin{gather}
    \min(\sqrt{p-1},1)\big[\sum_{i=1}^{n}\|\bv_i\|^2\big]^{\frac{1}{2}}\leq\big[\ebb_\epsilon\|\sum_{i=1}^{n}\epsilon_i\bv_i\|^p\big]^{\frac{1}{p}}
    \leq \max(\sqrt{p-1},1)\big[\sum_{i=1}^{n}\|\bv_i\|^2\big]^{\frac{1}{2}},\label{khitchine-kahane}\\
  \ebb_\epsilon\|\sum_{i=1}^{n}\epsilon_i\bv_i\|\geq 2^{-\frac{1}{2}}\big[\sum_{i=1}^{n}\|\bv_i\|^2\big]^{\frac{1}{2}}.\label{khitchine-kahane-norm1}
  \end{gather}
  The above inequalities also hold when the Rademacher variables are replaced by $N(0,1)$ random variables.
  \item Let $X_1,\ldots,X_n$ be a set of matrices of the same dimension and let $g_1,\ldots,g_n$ be a sequence of independent $N(0,1)$ random variables. For all $q\geq2$,
  \begin{equation}\label{khitchine-kahane-matrix}
    \Big(\ebb_{\bm{g}}\big\|\sum_{i=1}^{n}g_iX_i\big\|_{S_q}^q\Big)^{\frac{1}{q}}\leq 2^{-\frac{1}{4}}\sqrt{\frac{q\pi}{e}}\max\Big\{
    \big\|\big(\sum_{i=1}^{n}X_i^\top X_i\big)^{\frac{1}{2}}\big\|_{S_q},\big\|\big(\sum_{i=1}^{n}X_i X_i^\top\big)^{\frac{1}{2}}\big\|_{S_q}\Big\}.
  \end{equation}
\end{enumerate}
\end{lemma}
\begin{proof}
  For Part (b), the original Khintchine-Kahane inequality for matrices is stated for Rademacher random variables, i.e, the Gaussian variables $g_i$ are replaced by Rademacher variables $\epsilon_i$.
	We now show that it also holds for Gaussian variables. Let $\psi_i^{(k)}=\frac{1}{\sqrt{k}}\sum_{j=1}^{k}\epsilon_{ik+j}$ with $\epsilon_{ik+j}$ being a sequence of independent Rademacher variables, then we have
  \begin{align*}
     \big(\ebb_\epsilon\|\sum_{i=1}^{n}\psi_i^{(k)}X_i\|_{S_q}^q\big)^{\frac{1}{q}} &  =
     \big(\ebb_\epsilon \|\sum_{i=1}^{n}\sum_{j=1}^{k}\epsilon_{ik+j}\frac{1}{\sqrt{k}}X_i\|^q_{S_q}\big)^{\frac{1}{q}} \\
     &\leq
     \sqrt{\frac{q\pi}{2^{\frac{1}{2}}e}}\max\Big\{\big\|\big(\sum_{i=1}^{n}\sum_{j=1}^k\frac{1}{k}X_i^\top X_i\big)^{\frac{1}{2}}\big\|_{S_q},
     \big\|\big(\sum_{i=1}^{n}\sum_{j=1}^k\frac{1}{k}X_i X_i^\top\big)^{\frac{1}{2}}\big\|_{S_q}\Big\}\\
     & \leq \sqrt{\frac{q\pi}{2^{\frac{1}{2}}e}}\max\Big\{\big\|\big(\sum_{i=1}^{n}X_i^\top X_i\big)^{\frac{1}{2}}\big\|_{S_q},\big\|\big(\sum_{i=1}^{n}X_i X_i^\top\big)^{\frac{1}{2}}\big\|_{S_q}\Big\},
  \end{align*}
  where the first inequality is due to the Khintchine-Kahane inequality for matrices involving Rademacher random variables \citep{lust1991non}. The proof is complete if we take $k$ to $\infty$ and use central limit theorem.  
\end{proof}

\section{Proof of Proposition \ref{prop:rademacher-independent-schatten}\label{sec:proof-independent-rademacher-schatten}}
We present the proof of Proposition \ref{prop:rademacher-independent-schatten} in the appendix due to its similarity to the proof of Proposition \ref{prop:rademacher-independent-lp}.

  We first consider the case $1\leq p\leq 2$. Since the dual norm of $\|\cdot\|_{S_p}$ is $\|\cdot\|_{S_{p^*}}$, we have the following lower bound on RC in this case
    \begin{align}
       \frak{R}_{nc}(\widetilde{H}_{S_p})
       & = \max_{V^i\in\widetilde{S}:i\in\nbb_{nc}}\frac{1}{nc}\ebb_\epsilon\sup_{\|W\|_{S_p}\leq\Lambda}\sum_{i=1}^{nc}\epsilon_i\inn{W,V^i} \notag\\
       & = \max_{V^i\in\widetilde{S}:i\in\nbb_{nc}}\frac{1}{nc}\ebb_\epsilon\sup_{\|W\|_{S_p}\leq\Lambda}\inn{W,\sum_{i=1}^{nc}\epsilon_iV^i} \notag\\
       & = \max_{V^i\in\widetilde{S}:i\in\nbb_{nc}}\frac{\Lambda}{nc}\ebb_\epsilon\big\|\sum_{i=1}^{nc}\epsilon_iV^i\big\|_{S_{p^*}}. \label{rademacher-independent-shatter<2}
    \end{align}
    Taking $V^1=\cdots=V^{nc}$ and applying the Khitchine-Kahane inequality \eqref{khitchine-kahane-norm1} further imply
    $$
      \frak{R}_{nc}(\widetilde{H}_{S_p}) \geq \max_{V^1\in\widetilde{S}}\frac{\Lambda}{nc}\ebb_\epsilon\big|\sum_{i=1}^{nc}\epsilon_i\big|\|V^1\|_{S_{p^*}}
        \geq \frac{\Lambda}{\sqrt{2nc}}\max_{V^1\in\widetilde{S}}\|V^1\|_{S_{p^*}}
        = \frac{\Lambda \max_{i\in\nbb_n}\|\bx_i\|_2}{\sqrt{2nc}},
    $$
    where the last identity follows from the following identity for any $V\in\widetilde{S}$
    \begin{equation}\label{identity-rank-one}
        \|V\|_{S_{p^*}}=\|V\|_{S_2}=\|V\|_{2,2}=\|V\|_{2,\infty}.
    \end{equation}

    We now turn to the upper bound. It follows from the relationship $\widetilde{H}_{S_p}\subset \widetilde{H}_{S_2},\forall 1\leq p\leq 2$ and \eqref{rademacher-independent-shatter<2} that ($\tr(A)$ denotes the trace of $A$)
    \begin{align}
       \frak{R}_{nc}(\widetilde{H}_{S_p}) & \leq \frak{R}_{nc}(\widetilde{H}_{S_2})
         = \max_{V^i\in\widetilde{S}:i\in\nbb_{nc}}\frac{\Lambda}{nc}\ebb_\epsilon\big\|\sum_{i=1}^{nc}\epsilon_iV^i\big\|_{S_2}\notag\\
       & = \max_{V^i\in\widetilde{S}:i\in\nbb_{nc}}\frac{\Lambda}{nc}\ebb_\epsilon\sqrt{\tr\big(\sum_{i,\tilde{i}=1}^{nc}\epsilon_i\epsilon_{\tilde{i}}V^i(V^{\tilde{i}})^\top\big)}\notag\\
       &  \leq \max_{V^i\in\widetilde{S}:i\in\nbb_{nc}}\frac{\Lambda}{nc}\sqrt{\sum_{i=1}^{nc}\tr(V^i(V^i)^\top)}\notag\\
       & =\max_{V^i\in\widetilde{S}:i\in\nbb_{nc}}\frac{\Lambda}{nc}\sqrt{\sum_{i=1}^{nc}\|V^i\|_{2,\infty}^2}
        \leq \frac{\Lambda \max_{i\in\nbb_n}\|\bx_i\|_2}{\sqrt{nc}},\label{rademacher-independent-shatter<2b}
    \end{align}
    where the second identity follows from the identity between Frobenius norm and $\|\cdot\|_{S_2}$,  the second inequality follows from the Jensen's inequality and the last identity is due to \eqref{identity-rank-one}.

    \smallskip
    We now consider the case $p>2$. According to the relationship $\widetilde{H}_{S_2}\subseteq \widetilde{H}_{S_p}$
    for all $p>2$ and the discussion for the case $p=2$, we know
    $$
      \frak{R}_{nc}(\widetilde{H}_{S_p})\geq \frak{R}_{nc}(\widetilde{H}_{S_2})\geq\frac{\Lambda\max_{i\in\nbb_n}\|\bx_i\|_2}{\sqrt{2nc}}.
    $$
    Furthermore, for any $W$ with $\|W\|_{S_p}\leq\Lambda$ we have
    $
      \|W\|_{S_2}\leq \min\{c,d\}^{\frac{1}{2}-\frac{1}{p}}\Lambda,
    $
    which, combined with \eqref{rademacher-independent-shatter<2b}, implies that
    $$
      \frak{R}_{nc}(\widetilde{H}_{S_p}) \leq  \max_{V^i\in\widetilde{S}:i\in\nbb_{nc}}
      \frac{1}{nc}\ebb_\epsilon\sup_{\|W\|_{S_2}\leq \Lambda \min\{c,d\}^{\frac{1}{2}-\frac{1}{p}}}\sum_{i=1}^{nc}\epsilon_i\inn{W,V^i}
       \leq \frac{\Lambda \max_{i\in\nbb_n}\|\bx_i\|_2\min\{c,d\}^{\frac{1}{2}-\frac{1}{p}}}{\sqrt{nc}}.
    $$
  The proof is complete.

\section{Proof of Proposition \ref{prop:FW}\label{sec:proof-fw}}
  It suffices to check $\|\bw^*\|_{2,p}\leq 1$ and $\langle\bw^*,\bv\rangle=-\|\bv\|_{2,p^*}$. We consider three cases.

  If $p=1$, it is clear that $\|\bw^*\|_{2,1}\leq1$ and $\langle\bw^*,\bv\rangle=-\|\bv\|_{2,\infty}$.

  If $p=\infty$, it is clear that $\|\bw^*\|_{2,\infty}\leq1$ and $\langle\bw^*,\bv\rangle=-\sum_{j=1}^{c}\|\bv_j\|_2=-\|\bv\|_{2,1}$.

  If $1<p<\infty$, it is clear that
  $$
    \|\bw^*\|_{2,p}=\big(\sum_{\tilde{j}=1}^{c}\|\bv_{\tilde{j}}\|_2^{(p^*-1)p}\big)^{\frac{1}{p}}/\big(\sum_{\tilde{j}=1}^{c}\|\bv_{\tilde{j}}\|_2^{p^*}\big)^{\frac{1}{p}}=1
  $$
  and
  $$
    \langle\bw^*,\bv\rangle=-\big(\sum_{\tilde{j}=1}^{c}\|\bv_{\tilde{j}}\|_2^{p^*}\big)^{-\frac{1}{p}}\sum_{\tilde{j}=1}^{c}\|\bv_{\tilde{j}}\|_2^{p^*}=-\|\bv\|_{2,p^*}.
  $$
  The proof is complete.

\small
\bibliographystyle{abbrvnat}
\setlength{\bibsep}{0.08cm}
\vskip 0.2in

\end{document}